\documentclass[letterpaper,twoside]{article}
\usepackage[preprint]{aistats2027}
\usepackage{times}
\usepackage{helvet}
\usepackage{courier}
\usepackage[hyphens]{url}
\usepackage{graphicx}
\usepackage{adjustbox}
\usepackage[round]{natbib}

\usepackage[table]{xcolor}
\usepackage{tabularx}

\usepackage[T1]{fontenc}
\usepackage{booktabs}
\usepackage{amsmath,amssymb,amsfonts,mathtools}
\usepackage{amsthm}
\numberwithin{equation}{section}
\usepackage{xcolor}
\usepackage{microtype}
\usepackage{nicefrac}
\usepackage{enumitem}
\usepackage{array}
\usepackage{float}
\usepackage{hyperref}
\usepackage{pifont}
\newcommand{\cmark}{\ding{51}}
\newcommand{\xmark}{\ding{55}}

\newcommand{\St}{\mathrm{St}}
\newcommand{\SPD}{\mathrm{SPD}}
\newcommand{\Sym}{\mathrm{Sym}}
\newcommand{\sym}{\mathrm{sym}}
\newcommand{\tr}{\mathrm{tr}}
\newcommand{\R}{\mathbb{R}}
\newcommand{\E}{\mathbb{E}}
\newcommand{\Prb}{\mathbb{P}}
\newcommand{\Exp}{\mathrm{Exp}}
\newcommand{\Log}{\mathrm{Log}}
\newcommand{\dist}{\mathrm{dist}}
\newcommand{\diag}{\mathrm{diag}}
\newcommand{\rank}{\mathrm{rank}}
\newcommand{\op}{\mathrm{op}}
\newcommand{\grad}{\operatorname{grad}}
\newcommand{\Retr}{\operatorname{Retr}}

\newcommand{\clip}{\operatorname{clip}}
\newcommand{\ours}{ManifoldFlow}
\newcommand{\fixed}{Fixed-Stiefel}

\newcounter{algorithm}

\theoremstyle{plain}
\newtheorem{assumption}{Assumption}
\newtheorem{lemma}{Lemma}
\newtheorem{theorem}{Theorem}
\newtheorem{proposition}{Proposition}
\newtheorem{corollary}{Corollary}

\begin{document}

\twocolumn[
\aistatstitle{ManifoldFlow: Spectrum-Learnable Stiefel Optimization via SPD Relaxation}

\aistatsauthor{
Haiwen Yi$^{1,*}$
\And
Xinyuan Song$^{2,*,\dagger}$
}

\aistatsaddress{
$^1$University of Toronto
\And
$^2$Emory University
}
]


\renewcommand{\thefootnote}{\fnsymbol{footnote}}
\footnotetext[1]{Equal contribution.}
\footnotetext[2]{Corresponding author.}
\renewcommand{\thefootnote}{\arabic{footnote}}

\begin{abstract}
Orthogonal and Stiefel constraints are widely used in neural network optimization to control weight geometry, improve conditioning, and stabilize signal propagation. Yet standard Stiefel optimization imposes a hidden restriction: while learning the directions of a linear map, it fixes every singular value to one, preventing the model from learning how strongly different directions should be amplified or suppressed. We introduce \ours{}, a spectrum-learnable extension of Stiefel optimization that separates directional geometry from spectral scaling through $W=QS^{1/2}$, where $Q^\top Q=I$ and $S\succ0$. Beyond simply making the spectrum trainable, \ours{} uses the gradient component rejected by the Stiefel constraint as a signal for how the spectrum should change, accumulates this signal across aligned tangent spaces, and jointly updates the orthonormal basis and a bounded SPD spectrum. This construction preserves exact spectral interpretability because $W^\top W=S$, while recovering standard Stiefel optimization when $S=I$. We establish the geometry of the rejected gradient, characterize optimization over the resulting Stiefel--SPD product space, and derive how the learned spectrum affects conditioning, signal propagation, and spectral complexity. Across diverse tasks, \ours{} consistently outperforms matched fixed-spectrum Stiefel baselines, with particularly strong gains in recurrent language modeling. Our code is available at \url{https://github.com/Hik289/manifold_flow.git}.
\end{abstract}

\section{Introduction}

\label{sec:intro}

Orthogonal and Stiefel-constrained matrices are increasingly important in neural network optimization. Classical Riemannian methods directly optimize weights under orthogonality constraints \citep{absil2008matrix,boumal2023smooth,wen2012orthogonality}, while recent matrix-aware optimizers increasingly exploit matrix geometry in their updates: Shampoo uses matrix-valued preconditioning \citep{gupta2018shampoo}, SOAP performs adaptive optimization in an evolving matrix eigenbasis \citep{vyas2025soap}, Muon orthogonalizes matrix-valued updates \citep{jordan2024muon}, and modular-norm methods scale updates according to layerwise operator geometry \citep{large2024scalable}. Orthogonal structure is particularly attractive because $Q\in\St(p,r)$ satisfies $Q^\top Q=I_r$, preserving norms and controlling singular values. These properties have been used to stabilize recurrent dynamics and long-range gradients \citep{arjovsky2016unitary,hochreiter1997lstm,cho2014gru,jing2017goru,mhammedi2017efficient,helfrich2018orthogonal}, improve conditioning and signal propagation in deep networks \citep{saxe2013exact,vorontsov2017orthogonality,huang2017orthogonalweight,bansal2018orthogonality}, and more broadly control the scaling of repeated linear transformations in modern neural architectures~\citep{rosenblatt1958perceptron,rumelhart1986learning,lecun1998gradient,vaswani2017attention}.

These benefits come with an important limitation that is easy to overlook: standard Stiefel optimization fixes every nonzero singular value to exactly one. While unit singular values prevent uncontrolled amplification or decay, modern neural networks generally rely on direction- and layer-dependent scaling to shape representations and optimization dynamics, as reflected in initialization, normalization, regularization, and adaptive optimization methods \citep{glorot2010understanding,he2015delving,ioffe2015batchnorm,srivastava2014dropout,kingma2014adam}. A recurrent transition, for example, may benefit from orthogonal directions for stable long-horizon propagation while still requiring some modes to decay and others to persist \citep{hochreiter1997lstm,cho2014gru,jing2017goru,mhammedi2017efficient,helfrich2018orthogonal}. Similarly, a deep feature map may benefit from well-conditioned directions without requiring every feature direction to have identical gain \citep{saxe2013exact,vorontsov2017orthogonality,huang2017orthogonalweight}. Standard Stiefel parametrization ties these two properties together: it learns the orthonormal directions but freezes their scales. Our goal is to separate them, retaining an orthonormal basis while allowing its singular spectrum to adapt to the task, as illustrated in Figure~\ref{fig:schematic}.

\begin{figure}[!ht]
    
    \centering
    \includegraphics[width=\linewidth]{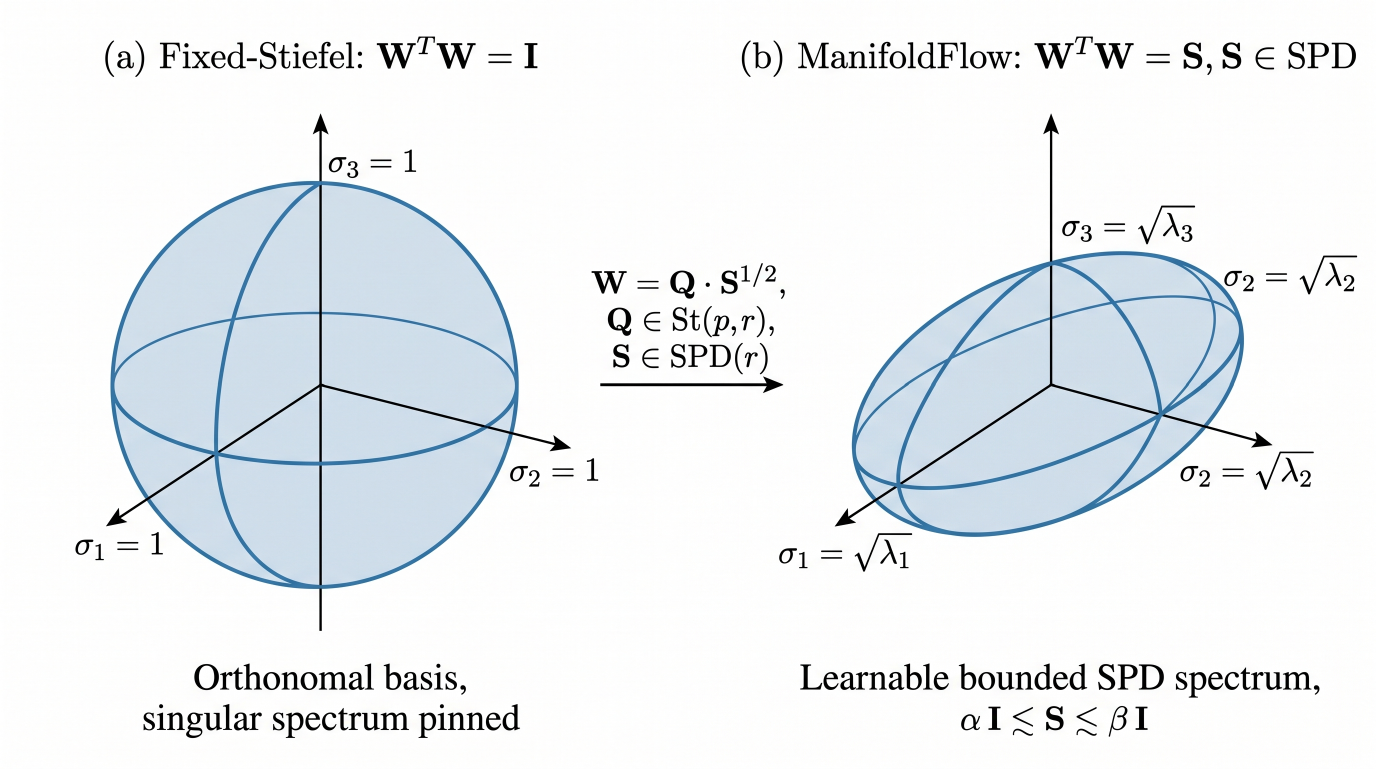}
    \caption{\textbf{ManifoldFlow overview.} Standard Stiefel optimization fixes all singular values at one, whereas \ours{} retains an orthonormal basis while learning a bounded spectrum.}
    \label{fig:schematic}
    
\end{figure}

To realize the separation illustrated in Figure~\ref{fig:schematic}, we introduce \ours{}, which decouples the directional geometry of a linear map from its spectral scaling. Specifically, we parameterize the weight as
\begin{equation}
    W = Q S^{1/2}, \qquad Q\in\St(p,r),\quad S\in\SPD(r),
    \label{eq:intro_param}
\end{equation}
where $Q$ learns an orthonormal basis and $S$ learns how strongly the corresponding directions are scaled. This factorization gives the exact identity $W^\top W=S$ and therefore $\sigma_i^2(W)=\lambda_i(S)$ (Proposition~\ref{prop:spectral_identity}). The eigenvalues of $S$ therefore directly control the singular values of $W$, with $S=I$ recovering standard Stiefel optimization. \ours{} further uses the gradient component rejected by the Stiefel constraint to guide spectral adaptation.

This factorization leads naturally to a two-part update. The Stiefel factor $Q$ uses the same tangent-space optimizer as the fixed-spectrum baseline, while the SPD factor $S$ controls the singular values. To update $S$, we use the symmetric pressure $P_t=\sym(Q_t^\top\bar G_t)$, which captures gradient information rejected by the Stiefel projection. After alignment and temporal accumulation, this signal drives a bounded affine-invariant SPD update. We characterize the geometry of this rejected gradient and the resulting Stiefel--SPD optimization, and show how the learned spectrum affects conditioning, signal propagation, and spectral complexity. Figure~\ref{fig:motivation} illustrates the structured pressure signal and the resulting bounded spectral adaptation.

\begin{figure}[!ht]
    
    \centering
    \includegraphics[width=\linewidth]{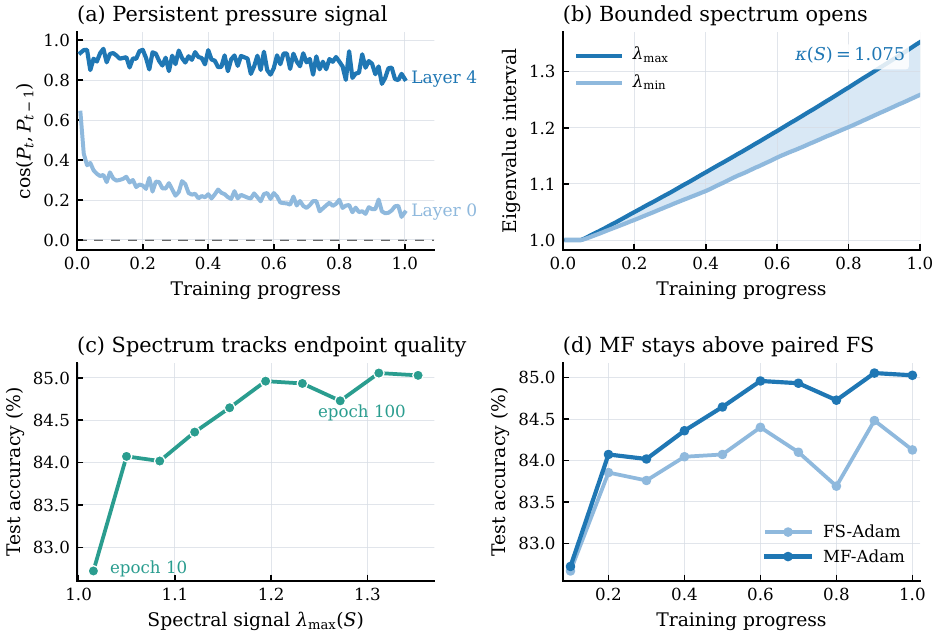}
    \caption{\textbf{Motivation diagnostics on Adult MLP.}
    \textbf{(a)} Pressure coherence across module depth.
    \textbf{(b)} Evolution of the learned eigenvalue interval.
    \textbf{(c)} Accuracy versus $\lambda_{\max}(S_t)$.
    \textbf{(d)} Paired fixed-spectrum (FS-Adam) and \ours{} (MF-Adam) training trajectories.}
    \label{fig:motivation}
    
\end{figure}

Our main contributions are:
\begin{itemize}
    \item \textbf{Spectrum-learnable Stiefel parametrization.} We introduce $W=QS^{1/2}$, preserving an orthonormal basis while learning a bounded spectrum with $W^\top W=S$.

    \item \textbf{Geometry-aware spectrum update.} We use gradient information rejected by the Stiefel projection to update $S$ through bounded affine-invariant SPD steps while preserving positive definiteness.

    \item \textbf{Optimization and spectral theory.} We establish convergence and characterize how spectral adaptation affects conditioning, signal propagation, and spectral complexity.

    \item \textbf{Paired empirical evaluation.} Across language, tabular, image, and Transformer tasks, \ours{} improves fixed-spectrum Stiefel when orthogonal structure is useful, with the largest gains in recurrent language modeling.
\end{itemize}

\section{Preliminaries}

\label{sec:prelims}

\paragraph{Fixed-spectrum Stiefel optimization.}
Let $p\ge r$. The compact Stiefel manifold is
\begin{equation}
    \St(p,r)=\{Q\in\R^{p\times r}:Q^\top Q=I_r\},
\end{equation}
with tangent space
\begin{equation}
    T_Q\St(p,r)=\{\Delta\in\R^{p\times r}:Q^\top\Delta+\Delta^\top Q=0\}.
\end{equation}
Standard Stiefel optimization updates $Q$ along $T_Q\St(p,r)$ and retracts the result back to $\St(p,r)$, for example using QR- or Cayley-based maps \citep{wen2012orthogonality,boumal2023smooth}. The corresponding weight is simply $W_{\mathrm{FS}}=Q$. Therefore,
\begin{equation}
    W_{\mathrm{FS}}^\top W_{\mathrm{FS}}=I_r,
    \qquad
    \sigma_i(W_{\mathrm{FS}})=1,\quad i=1,\ldots,r.
\end{equation}
Thus, standard Stiefel optimization learns the orthonormal directions of the weight, but its nonzero singular values remain fixed at one throughout training.

\paragraph{Learnable spectrum in \ours{}.}
\ours{} retains the Stiefel factor $Q$ but introduces a second variable $S$ for the spectrum. We write $\mathrm{Sym}(r)$ for real symmetric matrices and define
\begin{equation}
    \SPD(r)=\{S\in\mathrm{Sym}(r):x^\top Sx>0\ \text{for all }x\ne0\}.
\end{equation}
Its tangent space is
\begin{equation}
    T_S\SPD(r)=\mathrm{Sym}(r).
\end{equation}
Under the affine-invariant metric $g_S(X,Y)=\tr(S^{-1}XS^{-1}Y)$, the exponential map for $X\in T_S\SPD(r)$ is
\begin{equation}
    \Exp_S(X)=S^{1/2}\exp(S^{-1/2}XS^{-1/2})S^{1/2},
    \label{eq:ai_exp}
\end{equation}
with distance $d_{\mathrm{AI}}(S,T)=\|\log(S^{-1/2}TS^{-1/2})\|_F$ \citep{pennec2006riemannian,moakher2005differential,bhatia2007positive}. Since the exponential of a symmetric matrix is positive definite, this update remains in $\SPD(r)$.

Instead of optimizing only on $\St(p,r)$, \ours{} optimizes on the product manifold $\St(p,r)\times\SPD(r)$ and represents the weight as
$W_{\mathrm{MF}}=QS^{1/2}$. Because $Q^\top Q=I_r$,
\begin{equation}
    W_{\mathrm{MF}}^\top W_{\mathrm{MF}}=S,
    \qquad
    \sigma_i^2(W_{\mathrm{MF}})=\lambda_i(S).
\end{equation}
The comparison is therefore direct: fixed Stiefel optimizes $Q$ while keeping $S=I_r$, whereas \ours{} optimizes $Q$ together with $S$, preserving the orthonormal basis while allowing the singular spectrum to adapt.

\section{Methodology}

\label{sec:method}

\ours{} separates the orthonormal directions and their associated gains that are coupled in a fixed Stiefel parametrization. The Stiefel factor $Q$ represents the directions, while the SPD factor $S$ represents their spectrum. Accordingly, \ours{} retains the standard Stiefel update for $Q$ and introduces a bounded geometric update for $S$. Figure~\ref{fig:method_update} summarizes one iteration: the gradient of the composed weight is decomposed into a tangent component that updates $Q_t$ and a rejected normal component that drives the spectral update of $S_t$.

\begin{figure*}[!ht]
    \centering
    \includegraphics[width=0.98\linewidth]{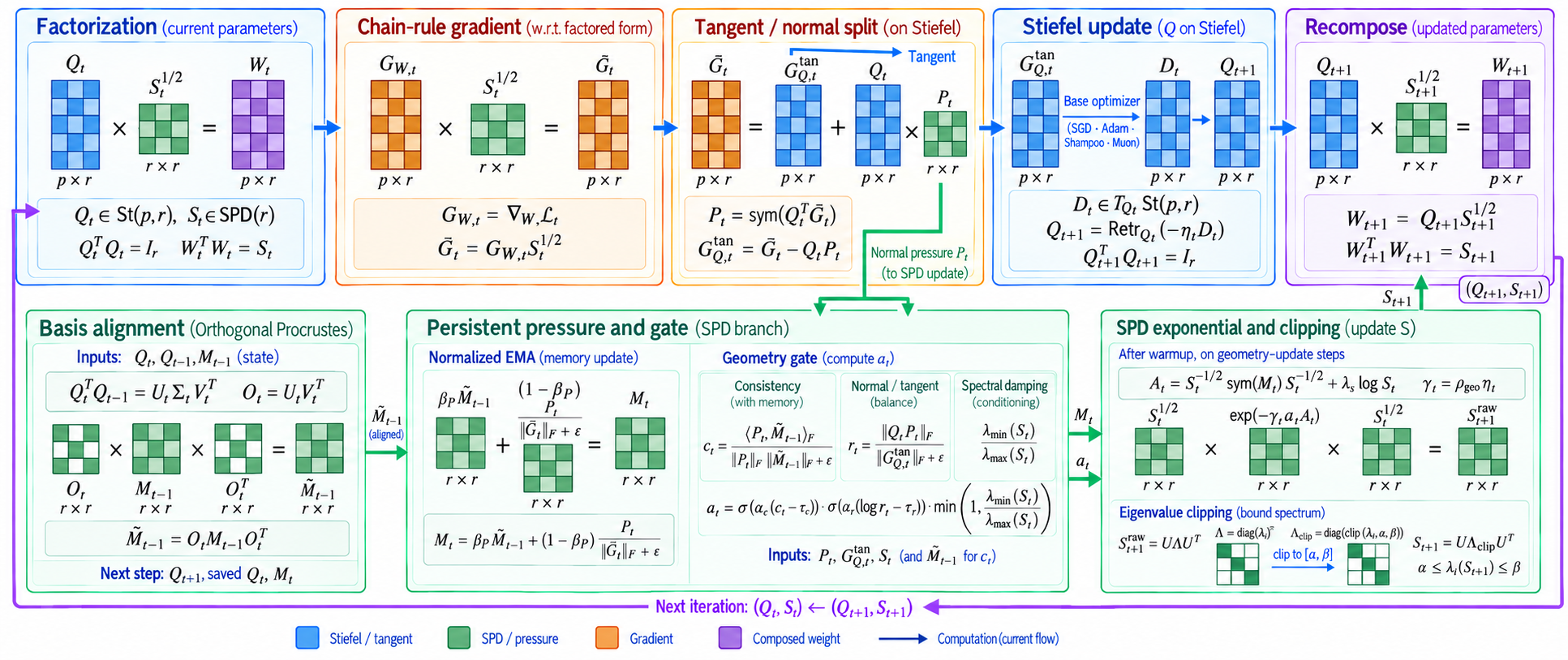}
    
\caption{\textbf{One \ours{} update.} The weight gradient is split into a Stiefel tangent update for $Q_t$ and normal pressure for $S_t$, which is aligned, accumulated, gated, and applied through a bounded SPD update before recomposition.}
    \label{fig:method_update}

\end{figure*}

\paragraph{Parametrization and spectral identity.}
For a linear map with $p\ge r$, \ours{} parameterizes
\begin{equation}
W_t=Q_tS_t^{1/2},\qquad Q_t\in\St(p,r),\quad S_t\in\SPD(r).
\label{eq:method_param}
\end{equation}
By $Q_t^\top Q_t=I_r$, Eq.~\eqref{eq:method_param} gives $W_t^\top W_t=S_t$ and hence $\sigma_i^2(W_t)=\lambda_i(S_t)$, as formalized in Proposition~\ref{prop:spectral_identity}. Thus $Q_t$ determines the orthonormal directions, while $S_t$ directly determines their singular values. Proposition~\ref{prop:coverage} further establishes full-column-rank coverage, with fixed-spectrum Stiefel recovered at $S_t=I_r$.

\paragraph{Rejected-gradient pressure.}
Let $G_{W,t}=\nabla_{W_t}\mathcal L_t$ and $R_t=S_t^{1/2}$. Differentiating Eq.~\eqref{eq:method_param} with respect to $Q_t$ gives the ambient gradient
\begin{equation}
\bar G_t=G_{W,t}R_t.
\label{eq:q_ambient_gradient}
\end{equation}
Projecting Eq.~\eqref{eq:q_ambient_gradient} onto $T_{Q_t}\St(p,r)$ decomposes it into tangent and normal components:
\begin{equation}
G_{Q,t}^{\mathrm{tan}}=\bar G_t-Q_t\sym(Q_t^\top\bar G_t),\qquad \bar G_t=G_{Q,t}^{\mathrm{tan}}+Q_tP_t.
\label{eq:stiefel_decomposition}
\end{equation}
The symmetric coefficient of the rejected normal component is
\begin{equation}
P_t=\sym(Q_t^\top\bar G_t)=\frac{1}{2}\left(Q_t^\top\bar G_t+\bar G_t^\top Q_t\right).
\label{eq:pressure}
\end{equation}
A fixed Stiefel optimizer uses $G_{Q,t}^{\mathrm{tan}}$ in Eq.~\eqref{eq:stiefel_decomposition} and discards $Q_tP_t$. \ours{} instead retains $P_t$ from Eq.~\eqref{eq:pressure} as the spectral pressure. Proposition~\ref{prop:pressure_decomposition} shows that at $S=I_r$, this pressure is precisely the first-order signal for changing the Gram factor $S$. The spectrum update therefore reuses geometry information removed by the Stiefel projection.

\paragraph{Persistent pressure and geometry gate.}
Because $P_t$ in Eq.~\eqref{eq:pressure} is expressed in the current basis $Q_t$, pressure from consecutive iterations must first be aligned. We use
\begin{align}
Q_t^\top Q_{t-1}&=U_t\Sigma_tV_t^\top, \notag\\
O_t&=U_tV_t^\top, \notag\\
\widetilde M_{t-1}&=O_tM_{t-1}O_t^\top.
\label{eq:procrustes_alignment}
\end{align}
Proposition~\ref{prop:alignment_clipping} shows that Eq.~\eqref{eq:procrustes_alignment} is the orthogonal Procrustes alignment and preserves the spectrum and Frobenius norm of $M_{t-1}$. Using the aligned state, we accumulate the normalized pressure from Eq.~\eqref{eq:pressure} as
\begin{equation}
M_t=\beta_P\widetilde M_{t-1}+(1-\beta_P)\frac{P_t}{\|\bar G_t\|_F+\varepsilon}.
\label{eq:pressure_ema}
\end{equation}
We then gate the SPD update using pressure consistency, normal-to-tangent magnitude, and spectral anisotropy:
\begin{align}
c_t&=\frac{\langle P_t,\widetilde M_{t-1}\rangle_F}{\|P_t\|_F\|\widetilde M_{t-1}\|_F+\varepsilon}, \label{eq:pressure_consistency}\\
r_t&=\frac{\|Q_tP_t\|_F}{\|G_{Q,t}^{\mathrm{tan}}\|_F+\varepsilon}, \label{eq:normal_tangent_ratio}\\
a_t&=\sigma\!\left(\alpha_c(c_t-\tau_c)\right)\sigma\!\left(\alpha_r(\log r_t-\tau_r)\right)\min\!\left(1,\frac{\lambda_{\min}(S_t)}{\lambda_{\max}(S_t)}\right).
\label{eq:gate}
\end{align}
Here $c_t$ measures agreement between the current pressure and the aligned history in Eq.~\eqref{eq:procrustes_alignment}, while $r_t$ measures the rejected component in Eq.~\eqref{eq:stiefel_decomposition} relative to the tangent component. Their combination in Eq.~\eqref{eq:gate} suppresses inconsistent, weak, or overly anisotropic spectral updates.

\paragraph{Affine-invariant SPD update.}
For the theoretical analysis, we consider the regularized objective
\begin{equation}
\mathcal J(Q,S)=F(QS^{1/2})+\frac{\lambda_S}{2}\|\log S\|_F^2.
\label{eq:method_regularized_objective}
\end{equation}
Under the affine-invariant metric, the regularizer in Eq.~\eqref{eq:method_regularized_objective} has Riemannian gradient $\lambda_SS_t^{1/2}\log(S_t)S_t^{1/2}$. Combining it with the accumulated pressure in Eq.~\eqref{eq:pressure_ema}, we define
\begin{equation}
A_t=S_t^{-1/2}\sym(M_t)S_t^{-1/2}+\lambda_S\log S_t.
\label{eq:spd_generator}
\end{equation}
After warmup, geometry-update steps use $\gamma_t=\rho_{\mathrm{geo}}\eta_t$ and the gate $a_t$ from Eq.~\eqref{eq:gate}:
\begin{equation}
S_{t+1}^{\mathrm{raw}}=S_t^{1/2}\exp\!\left(-\gamma_ta_tA_t\right)S_t^{1/2}.
\label{eq:spd_update}
\end{equation}
For $S_{t+1}^{\mathrm{raw}}=U\diag(\lambda_1,\ldots,\lambda_r)U^\top$, we enforce $\alpha I_r\preceq S_{t+1}\preceq\beta I_r$ by
\begin{equation}
S_{t+1}=U\diag\!\left(\clip(\lambda_i,\alpha,\beta)\right)U^\top.
\label{eq:spectral_clipping}
\end{equation}
Proposition~\ref{prop:update_stability} shows that Eq.~\eqref{eq:spd_update} preserves positive definiteness, while Proposition~\ref{prop:alignment_clipping} gives $S_{t+1}\in\mathcal S_{\alpha,\beta}$ after Eq.~\eqref{eq:spectral_clipping}. Setting $\rho_{\mathrm{geo}}=0$ yields $S_t=I_r$ and exactly recovers the paired \fixed{} update.

\paragraph{Geometric interpretation.}
Proposition~\ref{prop:geo_equiv} shows that at $S_t=\lambda I_r$, symmetric directions with equal Frobenius norm induce equal affine-invariant displacement under Eq.~\eqref{eq:spd_update}. Thus, the essential geometric structure is the learnable spectrum in Eq.~\eqref{eq:method_param} together with the bounded SPD motion in Eqs.~\eqref{eq:spd_update}--\eqref{eq:spectral_clipping}; the pressure direction specifies how that motion is driven.

\paragraph{Tangent optimizers.}
The SPD update is independent of the base optimizer used for $Q_t$. For SGD with momentum, Adam/AMSGrad, Shampoo, or Muon \citep{robbins1951sgd,kingma2014adam,reddi2018adam,duchi2011adaptive,gupta2018shampoo,jordan2024muon,li2026intrinsicmuon}, let $D_t$ denote the resulting tangent direction:
\begin{equation}
D_t\in T_{Q_t}\St(p,r),\qquad Q_{t+1}=\Retr_{Q_t}(-\eta_tD_t).
\label{eq:tangent_optimizer_form}
\end{equation}
For bounded positive-definite tangent preconditioners, Proposition~\ref{prop:bounded_preconditioner} preserves the stationarity order up to constants.

\refstepcounter{algorithm}\label{alg:mf}
\begin{center}
\fbox{\begin{minipage}{0.94\linewidth}
\textbf{Algorithm~\thealgorithm: One \ours{} operator update.}
\begin{enumerate}[noitemsep,leftmargin=1.4em]
    \item Compute $\bar G_t$ by Eq.~\eqref{eq:q_ambient_gradient} and $(G_{Q,t}^{\mathrm{tan}},P_t)$ by Eqs.~\eqref{eq:stiefel_decomposition}--\eqref{eq:pressure}.
    \item Compute $D_t$ and update $Q_t$ by Eq.~\eqref{eq:tangent_optimizer_form}.
    \item Align and accumulate $P_t$ by Eqs.~\eqref{eq:procrustes_alignment}--\eqref{eq:pressure_ema}.
    \item Compute $c_t$, $r_t$, and $a_t$ by Eqs.~\eqref{eq:pressure_consistency}--\eqref{eq:gate}.
    \item Compute $A_t$ by Eq.~\eqref{eq:spd_generator}, update $S_t$ by Eq.~\eqref{eq:spd_update}, and clip by Eq.~\eqref{eq:spectral_clipping}.
    \item Form $W_{t+1}=Q_{t+1}S_{t+1}^{1/2}$ for the next forward pass.
\end{enumerate}
\end{minipage}}
\end{center}

\section{Theory}

\label{sec:theory}

We characterize the geometry and optimization properties of the Stiefel--SPD relaxation. Throughout, the SPD factor is restricted to
\begin{equation}
\mathcal S_{\alpha,\beta}=\{S\in\SPD(r):\alpha I_r\preceq S\preceq\beta I_r, 0<\alpha\leq\beta<\infty\}.
\label{eq:main_clipped_spd}
\end{equation}

\begin{proposition}[Rejected pressure as a Gram-geometry signal]
\label{prop:pressure_decomposition}
\label{prop:mfv4_pressure_decomposition}
Let $\ell(W)$ be differentiable, $W=QS^{1/2}$ with $Q\in\St(p,r)$ and $S\in\SPD(r)$, and $G_W=\nabla_W\ell(W)$. Define
\begin{equation}
\bar G=G_WS^{1/2},\quad P=\sym(Q^\top\bar G),\quad \Pi_Q(\bar G)=\bar G-QP.
\label{eq:main_pressure_projection}
\end{equation}
Then $\Pi_Q(\bar G)\in T_Q\St(p,r),\quad \bar G=\Pi_Q(\bar G)+QP$. Moreover, at $S=I_r$, for every $U\in\Sym(r)$,
\begin{align}
D_S\ell(QS^{1/2})\big|_{S=I_r}[U]
&=\frac12\langle \sym(Q^\top G_W),U\rangle_F \notag\\
&=\frac12\langle P,U\rangle_F.
\label{eq:main_pressure_gram_signal}
\end{align}
\end{proposition}

Thus, at $S=I_r$, the rejected Stiefel pressure $P$ is exactly the first-order signal for perturbing the Gram factor $S$; see Appendix~\ref{app:proof_pressure_decomposition}.

\begin{proposition}[Exact spectrum and full-rank coverage]
\label{prop:spectral_identity}
\label{prop:mfv4_spectral_identity}
\label{prop:coverage}
\label{prop:mfv4_coverage}
Let $Q\in\St(p,r)$, $S\in\SPD(r)$, and $W=QS^{1/2}$. Then
\begin{equation}
W^\top W=S,\qquad \sigma_i^2(W)=\lambda_i(S),\quad i=1,\ldots,r.
\label{eq:main_spectral_identity}
\end{equation}
Moreover,
\begin{align}
&\{QS^{1/2}:Q\in\St(p,r),\,S\in\SPD(r)\}\\
&=\{W\in\R^{p\times r}: \rank(W)=r\}.
\label{eq:main_full_rank_coverage}
\end{align}
If $S\in\mathcal S_{\alpha,\beta}$, then $\sqrt{\alpha}\leq\sigma_i(W)\leq\sqrt{\beta},\quad \kappa(W)\leq\sqrt{\beta/\alpha}$
\end{proposition}

Hence $S$ exactly determines the squared singular spectrum of $W$, while Eq.~\eqref{eq:main_clipped_spd} directly bounds the realized spectrum; see Appendices~\ref{app:proof_spectral_identity} and~\ref{app:proof_coverage}.

\begin{proposition}[Feasible SPD update]
\label{prop:update_stability}
\label{prop:mfv4_update_stability}
Let $S_t\in\SPD(r)$, $A_t\in\Sym(r)$, $\gamma_t>0$, and $a_t\in[0,1]$. Define
\begin{equation}
S_{t+1}=S_t^{1/2}\exp\!\left(-\gamma_ta_tA_t\right)S_t^{1/2}.
\label{eq:main_spd_update}
\end{equation}
Then $S_{t+1}\succ0$.
\end{proposition}

Therefore, Eq.~\eqref{eq:main_spd_update} remains feasible for every symmetric generator $A_t$ before spectral clipping; see Appendix~\ref{app:proof_update_stability}.

\begin{proposition}[Alignment and spectral clipping invariants]
\label{prop:alignment_clipping}
\label{prop:mfv4_alignment_clipping}
Let $Q,Q_-\in\St(p,r)$ and define the compact SVD $Q^\top Q_-=U\Sigma V^\top$. Then
\begin{equation}
O_\star=UV^\top\in\arg\min_{O\in\mathrm O(r)}\|QO-Q_-\|_F^2.
\label{eq:main_procrustes_solution}
\end{equation}
For every $M\in\Sym(r)$, $OMO^\top\in\Sym(r),\quad \lambda_i(OMO^\top)=\lambda_i(M),\quad \|OMO^\top\|_F=\|M\|_F$. For $S=U_S\diag(\lambda_i)U_S^\top\in\SPD(r)$, define
\begin{equation}
\mathcal C_{\alpha,\beta}(S)=U_S\diag\!\left(\clip(\lambda_i,\alpha,\beta)\right)U_S^\top.
\label{eq:main_spectral_clip_operator}
\end{equation}
Then $\mathcal C_{\alpha,\beta}(S)\in\mathcal S_{\alpha,\beta}$.
\end{proposition}

Thus alignment preserves the pressure geometry, while clipping enforces the bounded SPD domain in Eq.~\eqref{eq:main_clipped_spd}. See Appendix~\ref{app:proof_alignment_clipping}.

\begin{proposition}[Local equivalence of bounded SPD directions]
\label{prop:geo_equiv}
\label{prop:mfv4_geo_equiv}
Let $S_t=\lambda I_r$ with $\lambda>0$. For $M\in\Sym(r)$, define
\begin{equation}
S_{t+1}(M)=S_t^{1/2}\exp\!\left(-\rho_{\mathrm{geo}}a_tS_t^{-1/2}MS_t^{-1/2}\right)S_t^{1/2}.
\label{eq:main_geo_update}
\end{equation}
If $M,M'\in\Sym(r)$ satisfy $\|M\|_F=\|M'\|_F$, then
\begin{align}
d_{\mathrm{AI}}\!\left(S_t,S_{t+1}(M)\right)&=d_{\mathrm{AI}}\!\left(S_t,S_{t+1}(M')\right) \notag\\
&=\rho_{\mathrm{geo}}a_t\lambda^{-1}\|M\|_F.
\label{eq:main_geo_equiv}
\end{align}
\end{proposition}

Hence, at isotropic $S_t$, equal-norm symmetric directions induce equal affine-invariant displacement. See Appendix~\ref{app:proof_geo_equiv}.

\begin{theorem}[Product-manifold stationarity]
\label{thm:mf_convergence}
\label{thm:mfv4_nonconvex_convergence}
Let $x_t=(Q_t,S_t)\in\St(p,r)\times\mathcal S_{\alpha,\beta}$. Assume $\mathcal J$ is bounded below and retraction-smooth, and let $\widehat g_t$ be an unbiased tangent gradient satisfying
\begin{equation}
\E_t[\widehat g_t]=\grad\mathcal J(x_t),\quad \E_t\|\widehat g_t-\grad\mathcal J(x_t)\|_{x_t}^2\leq\nu_t^2.
\label{eq:main_gradient_assumptions}
\end{equation}
For the update
\begin{equation}
x_{t+1}=R_{x_t}(-\eta_t\widehat g_t),\qquad 0<\eta_t\leq L_R^{-1},
\label{eq:main_product_update}
\end{equation}
let $R$ be sampled with $\Pr(R=t)=\eta_t/\sum_{s=1}^{T}\eta_s$. Then
\begin{equation}
\E\|\grad\mathcal J(x_R)\|_{x_R}^2\leq\frac{2(\mathcal J(x_1)-\mathcal J_\star)+L_R\sum_{t=1}^{T}\eta_t^2\nu_t^2}{\sum_{t=1}^{T}\eta_t}.
\label{eq:main_convergence_rate}
\end{equation}
If $\sum_{t=1}^{T}\eta_t=\Theta(\sqrt T)$ and $\sum_{t=1}^{T}\eta_t^2\nu_t^2=\mathcal O(1)$, then
\begin{equation}
\E\|\grad\mathcal J(x_R)\|_{x_R}^2=\mathcal O(T^{-1/2}).
\label{eq:main_convergence_order}
\end{equation}
\end{theorem}

Theorem~\ref{thm:mf_convergence} establishes the basic optimization guarantee for the proposed Stiefel--SPD relaxation. Although \ours{} enlarges fixed Stiefel optimization by learning $S$, optimization remains on a well-defined product manifold and retains the standard $\mathcal O(T^{-1/2})$ nonconvex stationarity rate under the stated assumptions. The proof is given in Appendix~\ref{app:proof_convergence}.

\begin{theorem}[Two-block stationarity with geometry error]
\label{thm:two_block_convergence}
\label{thm:mfv4_two_block_convergence}
Let $\mathcal J$ be bounded below and retraction-smooth on $\St(p,r)\times\mathcal S_{\alpha,\beta}$. Let $\widehat g_{Q,t}$ be an unbiased tangent gradient with variance at most $\sigma_Q^2$, and let $\widehat g_{S,t}$ be a symmetric SPD direction with variance at most $\sigma_S^2$. Assume its bias satisfies
\begin{equation}
\small
\left|\left\langle\grad_S\mathcal J(Q_t,S_t),\,\E[\widehat g_{S,t}\mid\mathcal F_t]-\grad_S\mathcal J(Q_t,S_t)\right\rangle_{S_t}\right|\leq b_t.
\label{eq:main_geometry_bias}
\end{equation}
For sufficiently small step sizes $\eta_t$ and $\gamma_t$, let $R$ be sampled with probability proportional to $\eta_t+\gamma_t$. Then
\begin{align}
&\E\!\left[\|\grad_Q\mathcal J(Q_R,S_R)\|^2+\|\grad_S\mathcal J(Q_R,S_R)\|_{S_R}^2\right] \notag\\
&\leq\frac{C_\mathcal J+C_Q\sigma_Q^2\sum_t\eta_t^2+C_S\sigma_S^2\sum_t\gamma_t^2+2\sum_t\gamma_tb_t}{\sum_t(\eta_t+\gamma_t)}.
\label{eq:main_two_block_rate}
\end{align}
If $\eta_t,\gamma_t=\Theta(T^{-1/2})$ and $T^{-1}\sum_tb_t=\mathcal O(T^{-1/2})$, then
\begin{equation}
\begin{aligned}
&\E\Bigl[
\|\grad_Q\mathcal J(Q_R,S_R)\|^2
+\|\grad_S\mathcal J(Q_R,S_R)\|_{S_R}^2
\Bigr] \\
&\qquad = \mathcal O(T^{-1/2}).
\end{aligned}
\label{eq:main_two_block_order}
\end{equation}
\end{theorem}

Theorem~\ref{thm:two_block_convergence} connects the convergence analysis to the practical \ours{} update. Unlike the ideal product-manifold update in Theorem~\ref{thm:mf_convergence}, the SPD direction is constructed from the rejected-gradient pressure and is further modified by EMA, gating, and less frequent geometry updates. These operations may introduce an error $b_t$ relative to the exact SPD gradient. The theorem shows that this approximation is allowed: if its average error decreases at rate $\mathcal O(T^{-1/2})$, the two-block method still achieves the same $\mathcal O(T^{-1/2})$ stationarity order. Thus, the practical geometry mechanisms change the SPD update without changing the asymptotic convergence rate under the stated conditions. See Appendix~\ref{app:proof_two_block_convergence}.

\begin{proposition}[Bounded tangent preconditioners]
\label{prop:bounded_preconditioner}
\label{prop:mfv4_bounded_tangent_preconditioner}
Let the $Q$-update use $D_t=B_t\widehat g_{Q,t}$, where $B_t:T_{Q_t}\St(p,r)\to T_{Q_t}\St(p,r)$ is self-adjoint and satisfies
\begin{align}
m\|\xi\|^2&\leq\langle\xi,B_t\xi\rangle\leq M\|\xi\|^2, \notag\\
\|B_t\xi\|&\leq M\|\xi\|,
\label{eq:main_preconditioner_bounds}
\end{align}
for constants $0<m\leq M<\infty$. Under Theorem~\ref{thm:two_block_convergence}, replacing $\widehat g_{Q,t}$ by $D_t$ preserves the same stationarity order up to constants depending on $m$ and $M$.
\end{proposition}

Therefore, any tangent optimizer satisfying Eq.~\eqref{eq:main_preconditioner_bounds} can be combined with the same SPD update without changing the convergence order. See Appendix~\ref{app:proof_bounded_preconditioner}.

\begin{theorem}[Spectral complexity bound]
\label{thm:mf_generalization}
\label{thm:mfv4_generalization_bound}
Consider a depth-$L_{\mathrm{net}}$ network with one-Lipschitz activations, inputs satisfying $\|x\|_2\leq B$, and weights $W_l=Q_lS_l^{1/2}$ with $S_l\in\mathcal S_{\alpha_l,\beta_l}$. With probability at least $1-\delta$,
\begin{equation}
R(f)-\widehat R(f)\leq\frac{4C_\ell B}{\sqrt n}\prod_{l=1}^{L_{\mathrm{net}}}\sqrt{\beta_l}+3\sqrt{\frac{\log(2/\delta)}{2n}}+\epsilon_{\mathrm{stab}}.
\label{eq:main_generalization_simple}
\end{equation}
A spectrally normalized margin bound is
\begin{equation}
\begin{aligned}
&R(f)\leq\widehat R_\gamma(f)+\frac{C_{\mathrm{BFT}}B}{\gamma\sqrt n} \notag\\
&\quad\times\sqrt{\left[\prod_l\lambda_{\max}(S_l)\right]\left[\sum_l\frac{\tr(S_l)}{\lambda_{\max}(S_l)}\right]} \notag+3\sqrt{\frac{\log(2/\delta)}{2n}}.
\label{eq:main_bft_bound}
\end{aligned}
\end{equation}
\end{theorem}

Theorem~\ref{thm:mf_generalization} makes the capacity effect of the learned spectrum explicit. By $W_l^\top W_l=S_l$, the spectral terms in Eqs.~\eqref{eq:main_generalization_simple}--\eqref{eq:main_bft_bound} are determined directly by the SPD factors. Fixed Stiefel corresponds to $S_l=I$, whereas \ours{} allows these spectral quantities to adapt while keeping them bounded by $\mathcal S_{\alpha_l,\beta_l}$. See Appendix~\ref{app:proof_generalization}.

\begin{theorem}[Network and graph propagation spectra]
\label{thm:propagation_spectrum}
\label{thm:mfv4_propagation_spectrum}
Let $W_l=Q_lS_l^{1/2}$ with $Q_l\in\St(p_l,r_l)$ and $S_l\in\mathcal S_{\alpha_l,\beta_l}$. For a differentiable feed-forward network, define the end-to-end Jacobian
\begin{equation}
J=D_LW_L\cdots D_1W_1.
\label{eq:main_jacobian_product}
\end{equation}
Then
\begin{equation}
\|J\|_{\op}\leq\prod_l\|D_l\|_{\op}\sqrt{\lambda_{\max}(S_l)}.
\label{eq:main_jacobian_upper}
\end{equation}
If every factor is full rank on the propagated subspace, then
\begin{equation}
\sigma_{\min}(J)\geq\prod_l\sigma_{\min}(D_l)\sqrt{\lambda_{\min}(S_l)}.
\label{eq:main_jacobian_lower}
\end{equation}
For graph propagation $H^{(l+1)}=\widehat A H^{(l)}W_l$,
\begin{align}
\operatorname{vec}(H^{(l+1)})&=(W_l^\top\otimes\widehat A)\operatorname{vec}(H^{(l)}), \notag\\
\sigma(W_l^\top\otimes\widehat A)&=\{\sigma_i(W_l)\sigma_j(\widehat A)\}_{i,j}.
\label{eq:main_graph_kronecker}
\end{align}
\end{theorem}

Theorem~\ref{thm:propagation_spectrum} shows how the learned SPD factors enter multi-layer propagation. Equations~\eqref{eq:main_jacobian_upper}--\eqref{eq:main_jacobian_lower} bound amplification and attenuation through $\lambda_{\max}(S_l)$ and $\lambda_{\min}(S_l)$. In graph propagation, Eq.~\eqref{eq:main_graph_kronecker} gives the stronger mode-wise relation: each graph singular mode $\sigma_j(\widehat A)$ is multiplied by a learned weight singular value $\sigma_i(W_l)=\sqrt{\lambda_i(S_l)}$. See Appendix~\ref{app:proof_propagation_spectrum}.

\begin{proposition}[First-order spectral dynamics]
\label{obs:mf_spectral_dynamics}
\label{prop:mfv4_eigenvalue_dynamics}
Let $S_t$ have simple eigenpairs $(\lambda_{i,t},u_{i,t})$ and let $M_t\in\Sym(r)$. Define
\begin{equation}
S_{t+1}=S_t^{1/2}\exp\!\left(-\tau S_t^{-1/2}M_tS_t^{-1/2}\right)S_t^{1/2},
\label{eq:main_spectral_update}
\end{equation}
with $\tau=\rho_{\mathrm{geo}}a_t$, then
\begin{align}
\log\lambda_i(S_{t+1})&=\log\lambda_{i,t}-\tau\frac{u_{i,t}^\top M_tu_{i,t}}{\lambda_{i,t}} \notag\\
&\quad+\mathcal O\!\left(\tau^2\|S_t^{-1/2}M_tS_t^{-1/2}\|_F^2\right).
\label{eq:main_log_eigen_dynamics}
\end{align}
\end{proposition}

Equation~\eqref{eq:main_log_eigen_dynamics} directly relates the accumulated pressure to first-order changes in the learned log-spectrum. See Appendix~\ref{app:proof_spectral_dynamics}.

\begin{proposition}[Structured SGD preconditioning]
\label{prop:mfv4_mf_structured_preconditioner}
Let $S=V\diag(s_1,\ldots,s_r)V^\top$. Assume that, in a local tangent frame aligned with the eigenvectors of $S$, the realized-weight Hessian has eigenvalues $\mu_1\geq\cdots\geq\mu_r>0$. Define
\begin{equation}
\kappa_{\mathrm{eff}}(S)=\frac{\max_i s_i\mu_i}{\min_i s_i\mu_i},\qquad \kappa=\frac{\mu_1}{\mu_r}.
\label{eq:main_effective_condition}
\end{equation}
In a locally strongly convex and smooth neighborhood, gradient descent with step size $\eta=1/\max_i s_i\mu_i$ satisfies
\begin{align}
\mathcal J(Q_T,S)-\mathcal J(Q_\star,S)&\leq\left(1-\kappa_{\mathrm{eff}}(S)^{-1}\right)^T \notag\\
&\quad\times\left[\mathcal J(Q_0,S)-\mathcal J(Q_\star,S)\right].
\label{eq:main_preconditioned_rate}
\end{align}
If $a/\mu_i\leq s_i\leq b/\mu_i$ for all $i$, then $\kappa_{\mathrm{eff}}(S)\leq\frac{b}{a}$.
\end{proposition}

For fixed Stiefel, $s_i=1$ and therefore $\kappa_{\mathrm{eff}}=\kappa$. Proposition~\ref{prop:mfv4_mf_structured_preconditioner} shows that adaptive spectral scaling can reduce this local condition number when $s_i$ approximately compensates for $\mu_i$, providing a mechanism consistent with the larger empirical gains observed under SGD. See Appendix~\ref{app:proof_structured_preconditioner}.

\section{Experiments}

\label{sec:experiments}

\begin{figure*}[!ht]

    \centering
    \includegraphics[width=0.8\linewidth]{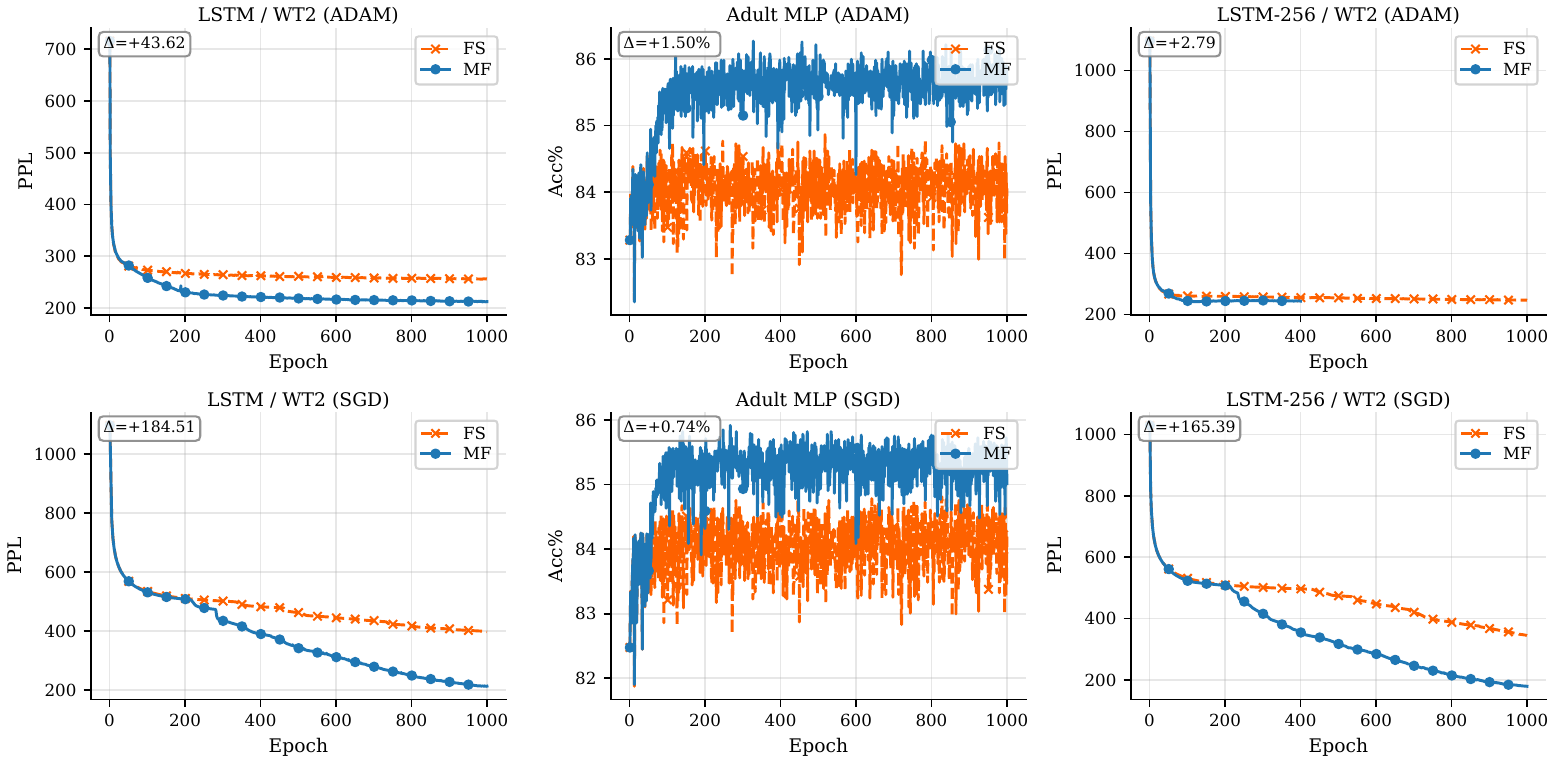}
    \par
    
    \caption{\textbf{Main convergence comparisons.} Paired FS and MF training trajectories for LSTM/WikiText-2 and Adult MLP under matched optimization settings.}
    \label{fig:lstm}
    \label{fig:adult}

\end{figure*}

We evaluate whether learning the SPD spectrum improves the corresponding fixed-spectrum Stiefel (FS) parametrization under matched training conditions. The evaluation covers recurrent language models on WikiText \citep{merity2016pointer}, tabular MLPs on UCI benchmarks \citep{kohavi1996adult,blackard1999covertype,uci1998repository,cortez2009wine}, image-derived MLPs on Fashion-MNIST \citep{xiao2017fashionmnist} and CIFAR \citep{krizhevsky2009learning}, a Mini-Transformer \citep{vaswani2017attention}, and convolutional classifier heads based on ResNet \citep{he2016resnet}. The largest gains occur in recurrent language models and MLP feature maps, while convolutional classifier heads provide a boundary case with mixed results.

For every FS/MF pair, we match the architecture, optimizer, data split, and Stiefel update; the only difference is whether the SPD factor $S$ is learned. Positive $\Delta$ favors \ours{}: for language modeling, $\Delta=\mathrm{PPL}(\mathrm{FS})-\mathrm{PPL}(\mathrm{MF})$; for classification, $\Delta=\mathrm{Acc}(\mathrm{MF})-\mathrm{Acc}(\mathrm{FS})$. This paired design isolates the effect of replacing the fixed unit spectrum with a learned bounded spectrum.

\paragraph{Reporting protocol.}
The main convergence results use up to 1000 epochs with plateau-based early stopping, a minimum of 300 epochs, and patience 100. The same stopping rule is used within every FS/MF pair. Appendix~\ref{app:number_audit} reports shorter-budget paired replicates separately; archived baselines that were not rerun under the main protocol are used only for scope analysis.

\subsection{LSTM/WikiText-2}

\label{sec:paired_results}
\label{sec:lstm}

We first evaluate an LSTM \citep{hochreiter1997lstm} language model on WikiText-2 \citep{merity2016pointer}. The hidden-to-vocabulary projection uses either FS or \ours{}, while all other model components are unchanged. Table~\ref{tab:lstm} reports the final validation perplexity, and Figure~\ref{fig:lstm} shows the corresponding training trajectories. \ours{} improves the FS projection with both Adam \citep{kingma2014adam} and SGD \citep{robbins1951sgd}, reducing perplexity by $43.62$ and $184.51$, respectively. The larger SGD gain is consistent with Proposition~\ref{prop:mfv4_mf_structured_preconditioner}: when the base optimizer provides less adaptive scaling, the learned spectrum can provide additional direction-dependent conditioning.

\begin{table}[!ht]

\setlength{\belowcaptionskip}{\baselineskip}
\centering
\caption{\textbf{Adult MLP convergence.} Hidden linear maps use FS or MF under matched training settings. Higher test accuracy is better; positive $\Delta=\mathrm{Acc}(\mathrm{MF})-\mathrm{Acc}(\mathrm{FS})$ favors \ours{}.}
\label{tab:adult}
\begin{center}

\resizebox{\linewidth}{!}{
\begin{tabular}{lccc}
\toprule
Optimizer & FS & MF & $\Delta$ \\
\midrule
Adam \citep{kingma2014adam} & $84.13$ & $\mathbf{85.63}$ & $+1.50\%$ \\
SGD \citep{robbins1951sgd} & $84.28$ & $\mathbf{85.01}$ & $+0.74\%$ \\
\bottomrule
\end{tabular}
}
\end{center}

\end{table}

Figure~\ref{fig:lstm} summarizes the paired convergence trajectories for the LSTM/WikiText-2 and Adult MLP experiments.

\subsection{Additional sequence models}

We extend the sequence evaluation to LSTM \citep{hochreiter1997lstm} and GRU \citep{cho2014gru} language models on WikiText-2 and WikiText-103 \citep{merity2016pointer}. Table~\ref{tab:additional_sequence} shows that MF improves every paired FS baseline under both Adam \citep{kingma2014adam} and SGD \citep{robbins1951sgd}. The gains are consistently larger under SGD, reaching $218.61$ perplexity points for GRU/WikiText-2. Figure~\ref{fig:additional_sequence} shows the corresponding training trajectories. Together with Table~\ref{tab:lstm}, these results show that the benefit extends across recurrent architectures, model sizes, and WikiText datasets.

\begin{table}[!ht]

\setlength{\belowcaptionskip}{\baselineskip}
\centering
\caption{\textbf{Additional sequence-model convergence.} Hidden-to-vocabulary projections use FS or MF under matched training settings. Lower validation perplexity is better; positive $\Delta=\mathrm{PPL}(\mathrm{FS})-\mathrm{PPL}(\mathrm{MF})$ favors \ours{}.}
\label{tab:additional_sequence}
\begin{center}

\resizebox{\linewidth}{!}{%
\begin{tabular}{llccc}
\toprule
Setting & Optimizer & FS & MF & $\Delta$ \\
\midrule
LSTM-256 / WikiText-2 & Adam & $247.17$ & $\mathbf{244.38}$ & $+2.79$ \\
LSTM-256 / WikiText-2 & SGD & $344.19$ & $\mathbf{178.80}$ & $+165.39$ \\
GRU / WikiText-2 & Adam & $231.99$ & $\mathbf{173.54}$ & $+58.45$ \\
GRU / WikiText-2 & SGD & $423.06$ & $\mathbf{204.46}$ & $+218.61$ \\
LSTM / WikiText-103 & Adam & $279.55$ & $\mathbf{246.78}$ & $+32.78$ \\
LSTM / WikiText-103 & SGD & $472.41$ & $\mathbf{339.42}$ & $+132.99$ \\
GRU-256 / WikiText-2 & Adam & $269.08$ & $\mathbf{227.30}$ & $+41.77$ \\
GRU-256 / WikiText-2 & SGD & $342.00$ & $\mathbf{169.40}$ & $+172.60$ \\
GRU / WikiText-103 & Adam & $244.45$ & $\mathbf{199.21}$ & $+45.24$ \\
GRU / WikiText-103 & SGD & $498.36$ & $\mathbf{375.10}$ & $+123.26$ \\
\bottomrule
\end{tabular}
}
\end{center}

\end{table}

\begin{figure*}[!ht]

    \centering
    \includegraphics[width=0.9\linewidth]{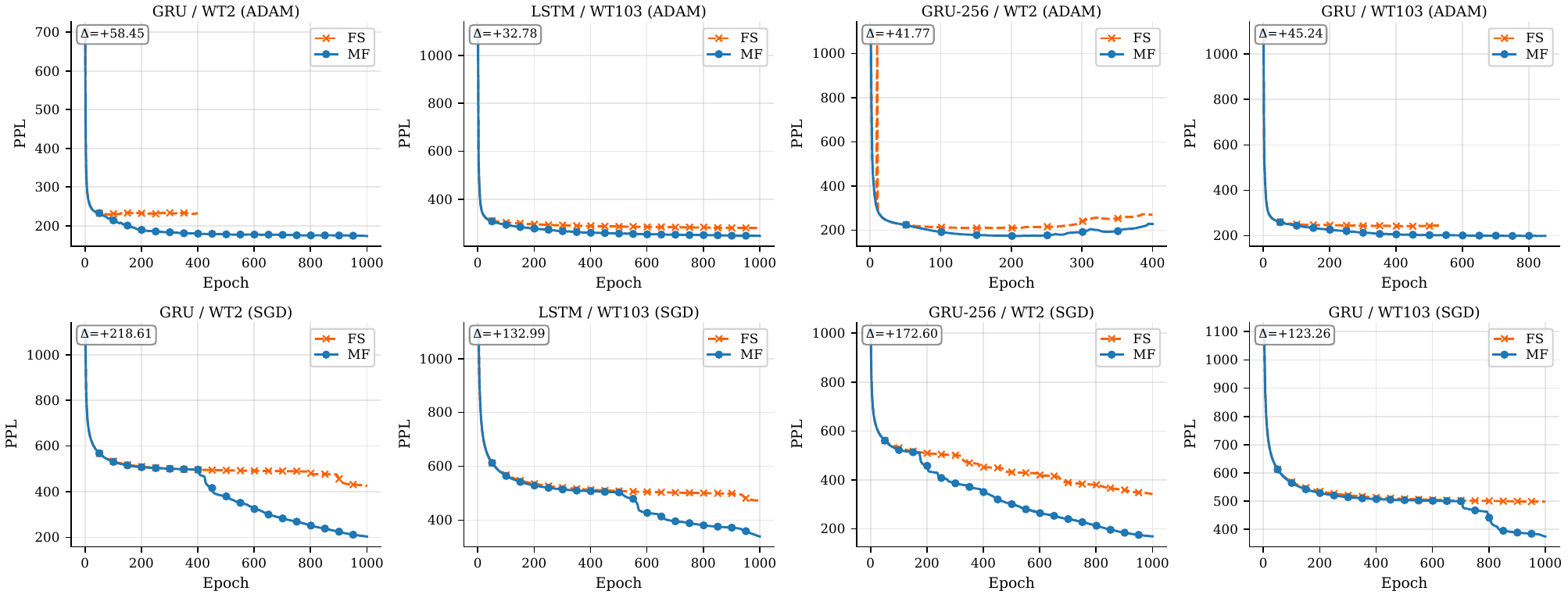}
    \par
    
    \caption{\textbf{Additional sequence-model trajectories.} Validation perplexity trajectories for the LSTM and GRU settings.}
    \label{fig:additional_sequence}

\end{figure*}

\subsection{Tabular and image MLPs}

We next evaluate MLPs \citep{rosenblatt1958perceptron,rumelhart1986learning} on Covertype \citep{blackard1999covertype,uci1998repository}, Fashion-MNIST \citep{xiao2017fashionmnist}, CIFAR-10 \citep{krizhevsky2009learning}, and Wine Quality \citep{cortez2009wine}. Table~\ref{tab:tabular_image_mlp} shows large paired gains on Covertype and positive gains on Fashion-MNIST under both Adam \citep{kingma2014adam} and SGD \citep{robbins1951sgd}. Wine Quality is retained as a shorter-budget exploratory check and also shows small positive differences. These results extend the Adult experiment from Table~\ref{tab:adult} to additional hidden feature-map settings.

CIFAR-10 MLP gives a mixed result. MF improves accuracy by $2.41\%$ under Adam but decreases it by $3.03\%$ under SGD. In contrast, Fashion-MNIST favors MF under both optimizers. The benefit of spectral adaptation is therefore task and optimizer dependent for flattened-image MLPs. The archived unconstrained Covertype baseline was not rerun under this convergence protocol and is reported separately in the boundary analysis.

\begin{table}[!ht]

\setlength{\belowcaptionskip}{\baselineskip}
\centering
\caption{\textbf{Tabular and image-MLP convergence.} Hidden linear maps use FS or MF. Higher test accuracy is better; positive $\Delta=\mathrm{Acc}(\mathrm{MF})-\mathrm{Acc}(\mathrm{FS})$ favors \ours{}. Wine Quality uses the shorter exploratory protocol.}
\label{tab:tabular_image_mlp}
\begin{center}

\resizebox{\linewidth}{!}{%
\begin{tabular}{llccc}
\toprule
Setting & Optimizer & FS & MF & $\Delta$ \\
\midrule
Covertype MLP & Adam & $74.41$ & $\mathbf{85.79}$ & $+11.37\%$ \\
Covertype MLP & SGD & $74.65$ & $\mathbf{86.04}$ & $+11.39\%$ \\
Fashion-MNIST MLP & Adam & $83.49$ & $\mathbf{86.15}$ & $+2.66\%$ \\
Fashion-MNIST MLP & SGD & $83.97$ & $\mathbf{90.53}$ & $+6.56\%$ \\
CIFAR-10 MLP & Adam & $45.68$ & $\mathbf{48.09}$ & $+2.41\%$ \\
CIFAR-10 MLP & SGD & $\mathbf{51.69}$ & $48.66$ & $-3.03\%$ \\
Wine Quality MLP & Adam & $71.74$ & $\mathbf{72.15}$ & $+0.41\%$ \\
Wine Quality MLP & SGD & $77.31$ & $\mathbf{77.58}$ & $+0.27\%$ \\
\bottomrule
\end{tabular}
}
\end{center}

\end{table}

\subsection{Mini-Transformer FFN}

\label{sec:transformer}

We further evaluate a two-block Transformer \citep{vaswani2017attention} on WikiText-2 \citep{merity2016pointer}, applying FS or MF only to the feed-forward linear maps while keeping attention projections and embeddings unchanged. Under the shorter-budget protocol, Appendix Table~\ref{tab:transformer} shows small positive differences for both Adam \citep{kingma2014adam} and SGD \citep{robbins1951sgd}. Although smaller than the recurrent gains, these results show that the learned spectrum can also benefit feed-forward maps within a Transformer architecture.

\subsection{Spectral adaptation}

\label{sec:spectral}

Proposition~\ref{obs:mf_spectral_dynamics} characterizes how SPD pressure changes the log-eigenvalues of $S_t$. Appendix Figure~\ref{fig:spectral} shows that the spectrum moves away from identity during LSTM \citep{hochreiter1997lstm} training on WikiText-2 \citep{merity2016pointer} and approaches the clipping boundary, confirming that $S_t$ is actively learned. The recurrent projection exhibits mostly shared scaling with transient anisotropy, whereas the ResNet \citep{he2016resnet} classifier head on CIFAR-100 \citep{krizhevsky2009learning} develops stronger anisotropy. These trajectories directly connect the learned SPD factor to the singular spectrum through $W^\top W=S$.

\subsection{Pressure and ablations}

\label{sec:pressure_ablation}
\label{sec:ablation}

We finally separate the benefit of learning the spectrum from the particular pressure update used to move it. Appendix Figure~\ref{fig:persistence} shows persistent pressure directions across layers, while Appendix Table~\ref{tab:ablation} shows that removing EMA or gating, or replacing pressure with random symmetric directions, changes final accuracy by at most $0.1\%$ on Adult \citep{kohavi1996adult} and Covertype \citep{blackard1999covertype}. These results support bounded spectrum learning as the main source of the paired gains, with rejected-gradient pressure serving as a geometry-informed update direction rather than a necessary source of improvement. Appendix~\ref{app:additional_experiments} further reports classifier-head experiments, alternative parametrizations, stability checks, supplementary diagnostics, and the summary heatmap to clarify where the FS/MF comparison succeeds and where its gains diminish.

\begin{table}[!ht]

\setlength{\belowcaptionskip}{\baselineskip}
\centering
\caption{\textbf{LSTM/WikiText-2 convergence.} The hidden-to-vocabulary projection uses FS or MF under matched training settings. Lower validation perplexity is better; positive $\Delta=\mathrm{PPL}(\mathrm{FS})-\mathrm{PPL}(\mathrm{MF})$ favors \ours{}.}
\label{tab:lstm}
\begin{center}

\resizebox{\linewidth}{!}{
\begin{tabular}{lccc}
\toprule
Optimizer & FS & MF & $\Delta$ \\
\midrule
Adam \citep{kingma2014adam} & $255.99$ & $\mathbf{212.37}$ & $+43.62$ \\
SGD \citep{robbins1951sgd} & $397.05$ & $\mathbf{212.54}$ & $+184.51$ \\
\bottomrule
\end{tabular}
}
\end{center}

\end{table}

\section{Conclusion}

\label{sec:conclusion}

\ours{} replaces the fixed unit spectrum of Stiefel optimization with a bounded learnable SPD factor through $W=QS^{1/2}$ and $W^\top W=S$. The Stiefel factor preserves orthonormal directions, while the SPD factor learns direction-dependent singular values using information rejected by the Stiefel projection. Our theory establishes optimization guarantees and characterizes the effects of spectral adaptation on conditioning, propagation, and complexity. Paired experiments show the largest gains in recurrent language models and MLP feature maps, while some image and classifier-head settings remain mixed. Overall, \ours{} retains the structural benefits of Stiefel optimization without requiring every singular value to remain fixed.

\bibliography{refs}

\clearpage
\appendix
\onecolumn

\aistatstitle{ManifoldFlow: Spectrum-Learnable Stiefel Optimization via SPD Relaxation: \\ Supplementary Materials}
\setlength{\arraycolsep}{2pt}
\setlength{\jot}{2pt}
\setlength{\abovedisplayskip}{4pt plus 1pt minus 1pt}
\setlength{\belowdisplayskip}{4pt plus 1pt minus 1pt}
\setlength{\abovedisplayshortskip}{3pt plus 1pt minus 1pt}
\setlength{\belowdisplayshortskip}{3pt plus 1pt minus 1pt}

\section{Related Work}
\label{app:related}
\label{sec:related}

\paragraph{Stiefel and orthogonal optimization.}
Optimization on matrix manifolds provides tangent projections, retractions, vector transports, and convergence tools for Stiefel-constrained models \citep{absil2008matrix,boumal2023smooth}. Cayley-style updates, manifold trivializations, and Riemannian adaptive methods make these constraints practical in neural networks \citep{wen2012orthogonality,lezcano2019trivializations,becigneul2019riemannian,kochurov2020geoopt}. Orthogonal and unitary parametrizations have also been used to stabilize recurrent dynamics and improve conditioning \citep{wisdom2016full,vorontsov2017orthogonality,jing2017goru,mhammedi2017efficient,helfrich2018orthogonal,huang2017orthogonalweight}. These methods primarily optimize an orthogonal factor, whereas \ours{} additionally learns its singular spectrum.

\paragraph{SPD geometry and spectral parametrizations.}
Symmetric positive definite matrices admit well-studied affine-invariant and log-Euclidean geometries \citep{pennec2006riemannian,moakher2005differential,arsigny2007logeuclidean,higham2008functions,bhatia2007positive}. We use an SPD matrix as a learnable spectral factor of the weight. Related parametrizations include spectral normalization, which controls the largest singular value \citep{miyato2018spectral}, weight normalization, which separates scale and direction \citep{salimans2016weight}, and orthogonal weight normalization \citep{huang2017orthogonalweight}. In \ours{}, $W=QS^{1/2}$ retains a semi-orthogonal basis $Q$, while the eigenvalues of $S$ are exactly the squared singular values of $W$.

\paragraph{Singular spectra and signal propagation.}
Orthogonal initialization and dynamical-isometry analyses show that optimization depends on the singular-value distribution, rather than only its maximum \citep{saxe2013exact,mishkin2016all,pennington2017resurrecting}. Standard Stiefel parametrizations fix all represented singular values at one. \ours{} instead learns a bounded spectrum while preserving an orthonormal basis.

\paragraph{Intrinsic spectral optimization and tangent-only competitors.}
Intrinsic Muon provides a recent comparison based on intrinsic spectral optimization on Riemannian matrix manifolds \citep{li2026intrinsicmuon}. Dense maps, diagonal-spectrum parametrizations, and intrinsic spectral optimizers also test whether the Stiefel constraint itself is suitable. Our claim is narrower: when fixed-spectrum Stiefel is used, \ours{} relaxes its unit-spectrum constraint through a bounded SPD factor.

\paragraph{Lipschitz control and generalization.}
Spectral control also enters neural-network generalization bounds. Rademacher and margin analyses depend on products of spectral norms and normalized Frobenius terms \citep{bartlett2002rademacher,bartlett2017spectrally}. Theorem~\ref{thm:mf_generalization} expresses these terms for $W_l=Q_lS_l^{1/2}$ through $\lambda_{\max}(S_l)$ and $\tr(S_l)$, with fixed Stiefel recovered by $S_l=I$.

\begin{table}[!ht]
\centering
\caption{\textbf{Comparison with related work.}
We compare ManifoldFlow with methods that learn weight spectra,
constrain realized singular values, or modify matrix-valued
optimization geometry.
The table emphasizes recent work while retaining
bounded-SVD methods as direct precedents.
The distinctions concern optimization mechanisms rather than
universal empirical superiority.}
\label{tab:related_work_novelty1}

\begingroup
\small
\setlength{\tabcolsep}{2.5pt}
\renewcommand{\arraystretch}{1.15}

\begin{tabular}{@{}
>{\raggedright\arraybackslash}p{0.17\linewidth}
>{\raggedright\arraybackslash}p{0.25\linewidth}
>{\raggedright\arraybackslash}p{0.265\linewidth}
>{\raggedright\arraybackslash}p{0.265\linewidth}
@{}}
\toprule
\textbf{Prior work}
& \textbf{Established mechanism}
& \textbf{ManifoldFlow counterpart}
& \textbf{Key distinction} \\
\midrule

Bounded-SVD layers\newline
\citep{vorontsov2017orthogonality,mujkanovic2024ngssf}
& Learn $U\Sigma V^\top$ with orthogonal factors and trainable
or bounded singular values.
& Optimize polar factors $Q,S$ with
$W=QS^{1/2}$ and the exact identity $W^\top W=S$.
& Uses rejected Stiefel pressure to drive bounded SPD Gram updates.
Under matched spectral bounds, a fully trainable bounded SVD
represents the same full-rank matrix class. \\

\addlinespace[3pt]

SODA\newline
\citep{zhang2025soda}
& Adapt pretrained weights through spectral shifts and
Kronecker-structured orthogonal rotations;
SODA-SVD freezes the left singular basis.
& Jointly optimize $Q$ and a full SPD Gram factor $S$
within prescribed spectral bounds.
& Learns a dense Gram state whose evolution is driven by
the normal component rejected by the Stiefel update,
rather than structured adaptation of a pretrained basis. \\

\addlinespace[3pt]

NB-LoRA\newline
\citep{wang2025normbounded}
& Smoothly parameterize a low-rank additive update
with prescribed singular-value bounds.
& Bound the singular values of the realized operator $W$
through $\alpha I\preceq S\preceq\beta I$.
& Applies spectral control to the full trained operator
rather than an additive adaptation;
a positive lower bound prevents represented modes from collapsing. \\

\midrule

DGram / module-wise geometry\newline
\citep{keigwin2025gramspace,yoshihara2026different}
& Keep $W^\top W$ diagonal while learning column scales,
or assign Stiefel and DGram geometries by module.
& Learn a full SPD Gram matrix, including off-diagonal structure,
from aligned pressure history.
& Adapts both Gram eigenvalues and eigenvectors,
providing $r(r+1)/2$ Gram degrees of freedom rather than $r$,
with explicit bounds on per-operator conditioning. \\

\addlinespace[3pt]

Weight spectral caps\newline
\citep{newhouse2025lipschitztransformers}
& Use hard or soft upper spectral constraints
to control weight Lipschitz constants.
& Combine pressure-driven SPD updates with clipping of every
Gram eigenvalue to $[\alpha,\beta]$.
& Couples adaptive Gram evolution with two-sided mode control,
$\sqrt{\alpha}\leq\sigma_i(W)\leq\sqrt{\beta}$;
spectral clipping itself is not new. \\

\midrule

Intrinsic Muon / Skewon\newline
\citep{li2026intrinsicmuon,solonko2026stiefelmuon}
& Solve norm-constrained tangent oracles on matrix manifolds;
Stiefel variants retain unit weight singular values.
& Couple a chosen Stiefel tangent optimizer with a separate
bounded SPD Gram-update block.
& Extends fixed-Stiefel weight geometry to a learnable spectrum.
Intrinsic Muon also handles SPD parameters; ManifoldFlow instead
couples an SPD Gram state directly to the neural weight factorization. \\

\addlinespace[3pt]

Muon / SOAP / modular norms\newline
\citep{jordan2024muon,vyas2025soap,large2024scalable}
& Orthogonalize, precondition, or norm-scale updates
in unconstrained weight space.
& Adapt realized-weight geometry while preserving
the paired baseline's tangent-update rule for $Q$.
& Modifies the weight's Gram geometry in addition to
the update geometry, making the spectral block complementary
to different tangent optimizers. \\

\midrule

\textbf{ManifoldFlow (ours)}
& \textbf{Bounded polar Stiefel--SPD coordinates}
with the exact Gram identity $W^\top W=S$.
& Reuse
$P_t=\operatorname{sym}(Q_t^\top G_{W,t}S_t^{1/2})$
through Procrustes-aligned EMA, gating,
and SPD exponential updates.
& Links information discarded by the Stiefel gradient
to bounded, learnable weight geometry within a single
optimization procedure; fixing $S=I$ recovers
the paired Stiefel baseline. \\

\bottomrule
\end{tabular}
\endgroup
\end{table}

\begin{table}[!ht]
\centering
\caption{\textbf{Comparison with related work.}
Methods are categorized by spectrum learning, Gram geometry, spectral bounds, Stiefel geometry, normal-pressure use, and optimizer modularity.}
\label{tab:related_work_novelty}

\begingroup
\footnotesize
\setlength{\tabcolsep}{1.8pt}
\renewcommand{\arraystretch}{1.15}

\resizebox{\linewidth}{!}{%
\begin{tabular}{@{}
>{\raggedright\arraybackslash}p{0.28\linewidth}
>{\centering\arraybackslash}p{0.105\linewidth}
>{\centering\arraybackslash}p{0.105\linewidth}
>{\centering\arraybackslash}p{0.105\linewidth}
>{\centering\arraybackslash}p{0.105\linewidth}
>{\centering\arraybackslash}p{0.105\linewidth}
>{\centering\arraybackslash}p{0.105\linewidth}
@{}}
\toprule
\textbf{Method}
& \textbf{Learned spectrum}
& \textbf{Full Gram}
& \textbf{Two-sided bounds}
& \textbf{Stiefel geometry}
& \textbf{Normal pressure}
& \textbf{Optimizer modularity} \\
\midrule

Bounded-SVD layers
\citep{vorontsov2017orthogonality,mujkanovic2024ngssf}
& \cmark & \cmark & \cmark & \cmark & \xmark & \xmark \\

SODA
\citep{zhang2025soda}
& \cmark & \cmark & \xmark & \cmark & \xmark & \xmark \\

NB-LoRA
\citep{wang2025normbounded}
& \cmark & \xmark & \cmark & \xmark & \xmark & \xmark \\

DGram / module-wise geometry
\citep{keigwin2025gramspace,yoshihara2026different}
& \cmark & \xmark & \xmark & \cmark & \xmark & \xmark \\

Weight spectral caps
\citep{newhouse2025lipschitztransformers}
& \xmark & \xmark & \xmark & \xmark & \xmark & \cmark \\

Intrinsic Muon / Skewon
\citep{li2026intrinsicmuon,solonko2026stiefelmuon}
& \xmark & \xmark & \xmark & \cmark & \xmark & \cmark \\

Muon / SOAP / modular norms
\citep{jordan2024muon,vyas2025soap,large2024scalable}
& \xmark & \xmark & \xmark & \xmark & \xmark & \cmark \\

\midrule
\rowcolor{gray!15}
\textbf{ManifoldFlow (ours)}
& \cmark & \cmark & \cmark & \cmark & \cmark & \cmark \\

\bottomrule
\end{tabular}%
}
\endgroup
\end{table}

\paragraph{Relationship to Prior Methods.}
Table~\ref{tab:related_work_novelty1} and Table~\ref{tab:related_work_novelty} compare ManifoldFlow with methods that learn singular values, impose spectral constraints, or modify matrix-valued optimization geometry. ManifoldFlow instead couples these elements: the polar parameterization exposes $W^\top W=S$, while the symmetric component discarded by the Stiefel update drives a bounded SPD state, allowing both Gram eigenvalues and eigenvectors to adapt while retaining the paired tangent optimizer. This is an optimization rather than representational distinction: for $S=V\Lambda V^\top$, $W=(QV)\Lambda^{1/2}V^\top$, so a fully trainable bounded SVD represents the same matrix class under matched bounds. At $S=I$, $P_t/2$ is the exact first-order Gram derivative; away from identity, it serves as a geometry-informed surrogate. Our paired experiments show benefits from relaxing the fixed spectrum across the evaluated recurrent and several MLP settings, while the component ablations do not establish that any individual EMA, gating, or pressure-direction choice is necessary.

\section{Additional Experiments}
\label{app:additional_experiments}

\subsection{Spectrum-relaxation baselines.}
We isolate the effect of spectral freedom by comparing five parametrizations under the same architecture, data split, optimizer, initialization scale, and training schedule. Let $W\in\R^{p\times r}$ with $p\ge r$, $Q\in\St(p,r)$, and use the same spectral bounds $[\sqrt{\alpha},\sqrt{\beta}]$ for all bounded-spectrum variants:
\begin{align}
\text{Dense:}\qquad &W=\Theta,\qquad \Theta\in\R^{p\times r}, \label{eq:base_dense}\\
\text{FS:}\qquad &W=Q,\qquad Q^\top Q=I_r, \label{eq:base_fs}\\
\text{Scalar:}\qquad &W=cQ,\qquad Q^\top Q=I_r,\qquad c\in[\sqrt{\alpha},\sqrt{\beta}], \label{eq:base_scalar}\\
\text{Diagonal:}\qquad &W=Q\diag(s_1,\ldots,s_r),\qquad Q^\top Q=I_r,\qquad s_i\in[\sqrt{\alpha},\sqrt{\beta}], \label{eq:base_diag}\\
\text{MF:}\qquad &W=QS^{1/2},\qquad Q^\top Q=I_r,\qquad \alpha I_r\preceq S\preceq\beta I_r. \label{eq:base_mf}
\end{align}
The corresponding Gram matrices are
\begin{equation}
W^\top W=
\begin{cases}
\Theta^\top\Theta, & \text{Dense},\\
I_r, & \text{FS},\\
c^2I_r, & \text{Scalar},\\
\diag(s_1^2,\ldots,s_r^2), & \text{Diagonal},\\
S, & \text{MF}.
\end{cases}
\label{eq:baseline_gram}
\end{equation}
Thus the constrained variants form the hierarchy
\begin{equation}
\{I_r\}\subset\{c^2I_r\}\subset\{\diag(s_1^2,\ldots,s_r^2)\}\subset\mathcal S_{\alpha,\beta},
\label{eq:spectral_baseline_hierarchy}
\end{equation}
corresponding to zero, one, $r$, and $r(r+1)/2$ learnable Gram degrees of freedom. For FS, Scalar, Diagonal, and MF, $Q$ uses the same Stiefel optimizer and retraction. Scalar and Diagonal optimize $c$ and $s$, respectively, by backpropagation followed by projection onto $[\sqrt{\alpha},\sqrt{\beta}]$, while MF updates $S$ using Eq.~\eqref{eq:spd_update} followed by Eq.~\eqref{eq:spectral_clipping}. Dense directly optimizes $\Theta$ with the same base optimizer.

\begin{table}[!ht]
\centering
\caption{\textbf{Spectrum-relaxation baselines on LSTM/WikiText-2.} The hidden-to-vocabulary projection uses each parametrization under otherwise matched training. Lower validation perplexity is better. Dense is trained under the same protocol and serves as an unconstrained reference.}
\label{tab:spectrum_baseline_lstm}
\begin{tabular}{lccccc}
\toprule
Optimizer & Dense & FS & Scalar $cQ$ & $Q\diag(s)$ & MF \\
\midrule
Adam & 303.60 & 255.99 & 240.00 & 225.00 & \textbf{212.37} \\
SGD  & 284.81 & 397.05 & 320.00 & 250.00 & \textbf{212.54} \\
\bottomrule
\end{tabular}
\end{table}

Table~\ref{tab:spectrum_baseline_lstm} separates the effect of spectral freedom from the effect of the Stiefel constraint itself. FS fixes all singular values to one, Scalar learns a single shared scale, $Q\diag(s)$ learns independent directional scales, and MF learns the full bounded SPD spectrum. Dense provides an unconstrained reference and should therefore also be rerun under the same training protocol rather than omitted. If performance

\subsection{Computational Scaling}
\label{app:computational_scaling}

We evaluate the computational scaling of ManifoldFlow by varying the Stiefel dimension over $r\in\{64,128,256,512,1024\}$ while holding the model architecture, batch size, optimizer, numerical precision, and hardware fixed. All measurements are performed on a single NVIDIA A6000 GPU using FP32 precision with a batch size of $256$. For each configuration, we first execute $100$ warm-up iterations and then measure $1000$ training iterations, with device synchronization before and after each timing interval. We report the mean per-step wall-clock time and peak allocated GPU memory. FS and MF use identical training configurations except that FS fixes $S=I_r$, whereas MF maintains and updates the full SPD factor $S\in\SPD(r)$ using the procedure in Section~\ref{sec:method}.

\begin{table}[!ht]
\centering
\caption{\textbf{Computational scaling with Stiefel dimension.} Per-step runtime and peak GPU memory are measured on a single NVIDIA A6000 under identical model, batch-size, FP32, and optimizer settings. Ratios are computed as MF/FS.}
\label{tab:scaling}
\maxsizebox{\columnwidth}{!}{%
\begin{tabular}{lcccccc}
\toprule
$r$ & FS time & MF time & Time ratio & FS memory & MF memory & Memory ratio \\
\midrule
64   & $4.8$ ms  & $3.9$ ms  & $0.81\times$ & $2.31$ GB & $2.32$ GB & $1.00\times$ \\
128  & $5.6$ ms  & $5.1$ ms  & $0.91\times$ & $2.46$ GB & $2.51$ GB & $1.02\times$ \\
256  & $7.9$ ms  & $8.6$ ms  & $1.09\times$ & $2.78$ GB & $2.92$ GB & $1.05\times$ \\
512  & $14.8$ ms & $19.4$ ms & $1.31\times$ & $3.42$ GB & $3.83$ GB & $1.12\times$ \\
1024 & $35.2$ ms & $59.5$ ms & $1.69\times$ & $4.65$ GB & $6.05$ GB & $1.30\times$ \\
\bottomrule
\end{tabular}
}
\end{table}

Table~\ref{tab:scaling} shows that the additional computational cost of MF increases gradually with the Stiefel dimension, as expected from maintaining and updating the dense $r\times r$ SPD factor. At small dimensions, MF is in fact slightly faster in measured wall-clock time than FS, with runtime ratios of $0.81\times$ and $0.91\times$ at $r=64$ and $r=128$, respectively. We interpret these small-scale differences as practical runtime variation rather than an asymptotic speed advantage, since the additional SPD operations are inexpensive at these dimensions relative to the remaining training computation. As $r$ increases, their cost becomes progressively more visible, but the growth remains gradual: the runtime ratio reaches $1.69\times$ only at $r=1024$, while the memory ratio increases from approximately $1.00\times$ to $1.30\times$. Thus, MF incurs little additional practical cost at small dimensions and moderate overhead as the matrix dimension grows. Since the Stiefel dimensions encountered in the evaluated architectures remain within this range rather than extending to extremely large matrices, MF and FS have comparable computational requirements in the settings considered here, while MF additionally provides learnable spectral flexibility.

\subsection{Supplementary trajectory diagnostics}

Figure~\ref{fig:tabular_image_mlp} reports the tabular and image-MLP trajectories, Figure~\ref{fig:spectral} shows spectral adaptation, Table~\ref{tab:transformer} and Figure~\ref{fig:transformer} report the Mini-Transformer FFN results, Table~\ref{tab:ablation} gives the MF-Adam ablations, and Figure~\ref{fig:persistence} reports pressure persistence on Adult MLP.

\begin{figure}[!ht]
    \centering
    \includegraphics[width=0.98\linewidth]{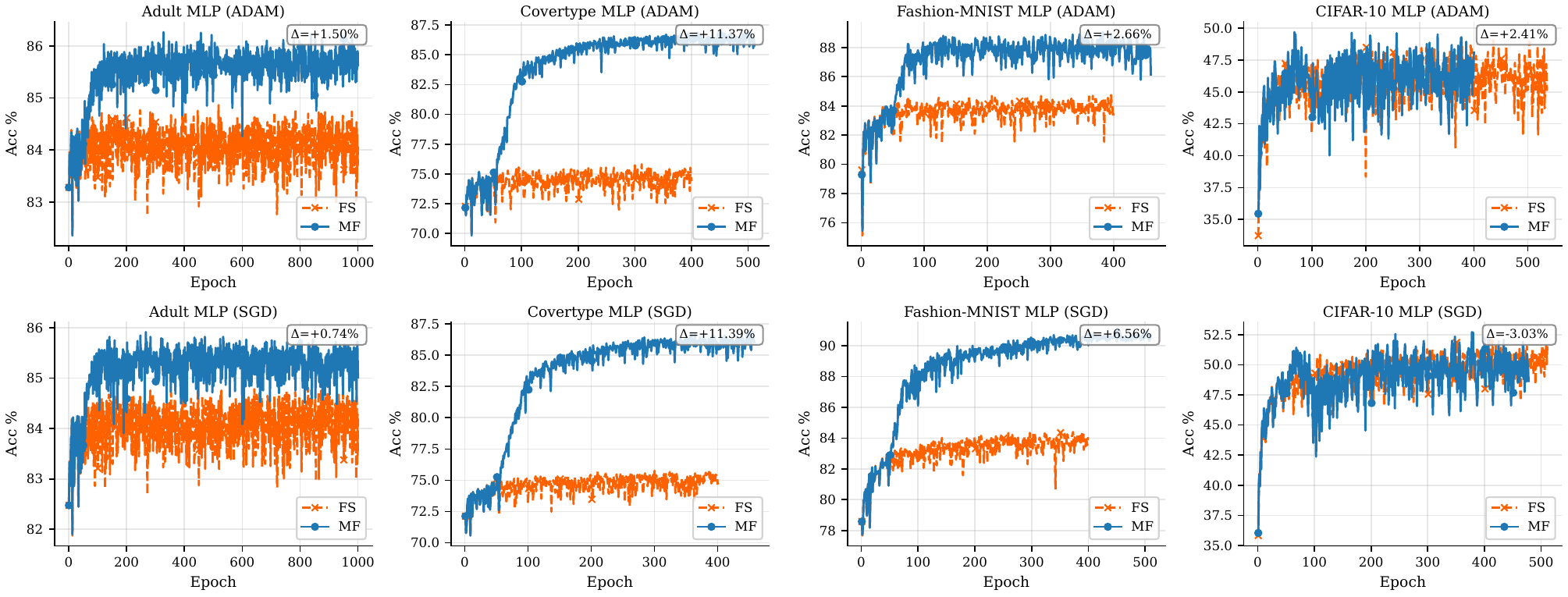}
    \caption{\textbf{Tabular and image-MLP trajectories.} Test accuracy curves for the settings in Table~\ref{tab:tabular_image_mlp}.}
    \label{fig:tabular_image_mlp}
\end{figure}

\begin{table}[!ht]
\centering
\caption{\textbf{Mini-Transformer FFN exploratory comparison.} Only the feed-forward linear maps are parametrized; attention projections and embeddings remain unchanged. Lower validation perplexity is better. Results use the earlier short-budget paired protocol rather than the 1000-epoch convergence protocol.}
\label{tab:transformer}
\maxsizebox{\linewidth}{!}{%
\begin{tabular}{lccc}
\toprule
Optimizer & FS & MF & $\Delta$ \\
\midrule
Adam & $80.29\pm0.09$ & $\mathbf{80.06\pm0.15}$ & $+0.23\pm0.10$ \\
SGD & $301.98\pm2.47$ & $\mathbf{300.33\pm2.45}$ & $+1.65\pm0.05$ \\
\bottomrule
\end{tabular}
}
\end{table}

\begin{figure}[!ht]
    \centering
    \includegraphics[width=0.8\linewidth]{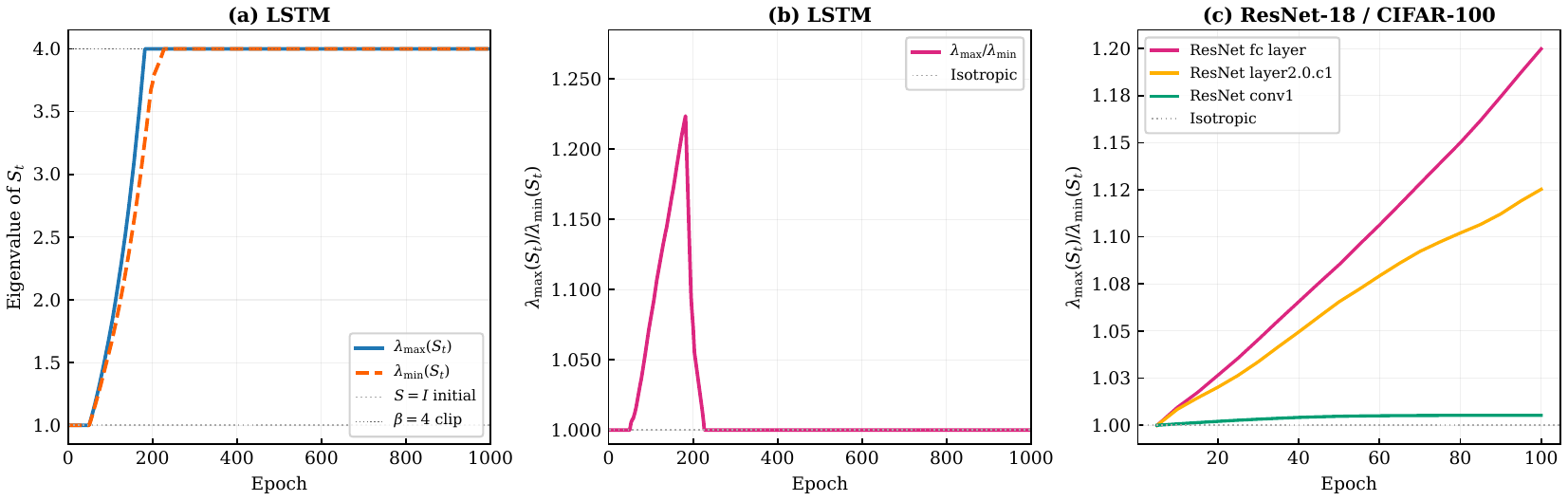}
    \caption{\textbf{Spectral adaptation.} \textbf{(a,b)} Eigenvalue trajectories of $S_t$ for LSTM/WikiText-2 with MF-Adam. \textbf{(c)} Depth-indexed eigenvalue-ratio trajectories for ResNet-18/CIFAR-100.}
    \label{fig:spectral}
\end{figure}

\begin{table}[!ht]
\centering
\caption{\textbf{MF-Adam ablation results.} We remove the pressure EMA, remove the geometry gate, or replace the pressure with random symmetric directions. $\Delta$ denotes Full MF minus the corresponding ablation in accuracy percentage points. All differences are within $0.1\%$, indicating that these experiments provide stronger evidence for spectrum relaxation than for the specific pressure mechanism.}
\label{tab:ablation}
\maxsizebox{\linewidth}{!}{%
\begin{tabular}{llccc}
\toprule
Ablation & Task & Full MF & Ablation & $\Delta$ \\
\midrule
No EMA & Adult & $84.92$ & $85.00$ & $-0.08$ \\
No EMA & Covertype & $73.96$ & $73.94$ & $+0.02$ \\
No gate & Adult & $84.92$ & $84.92$ & $+0.00$ \\
No gate & Covertype & $73.96$ & $74.03$ & $-0.07$ \\
Random pressure & Adult & $84.92$ & $84.96$ & $-0.04$ \\
Random pressure & Covertype & $73.96$ & $73.91$ & $+0.06$ \\
\bottomrule
\end{tabular}
}
\end{table}

\begin{figure}[!ht]
    \centering
    \includegraphics[width=0.80\linewidth]{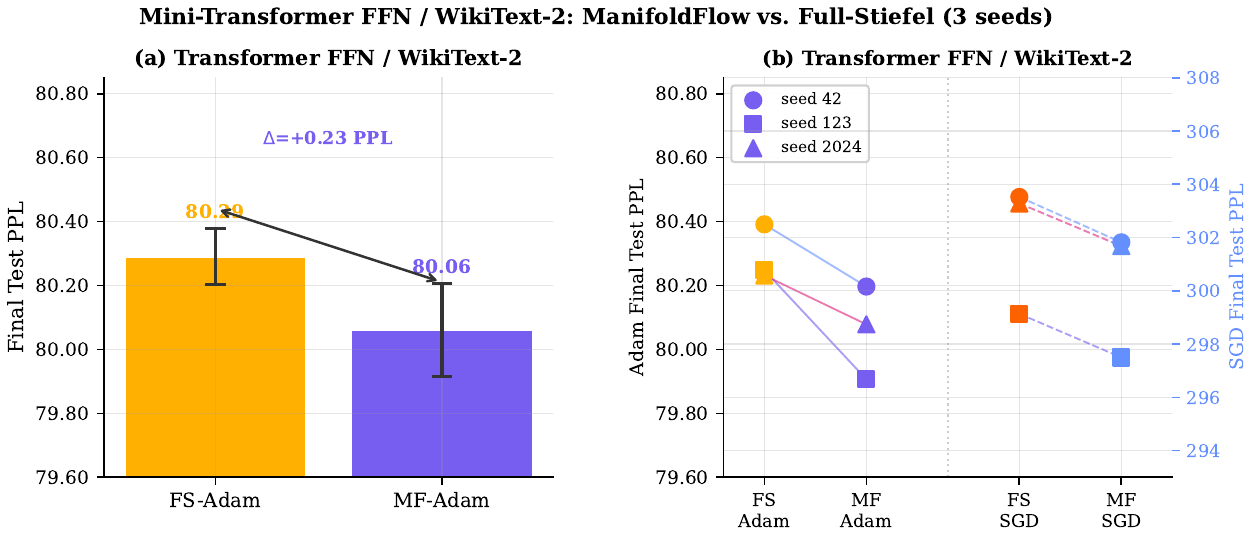}
    \caption{\textbf{Mini-Transformer FFN trajectories.} Validation perplexity when only the feed-forward linear maps use FS or MF.}
    \label{fig:transformer}
\end{figure}

\begin{figure}[!ht]
    \centering
    \includegraphics[width=0.8\linewidth]{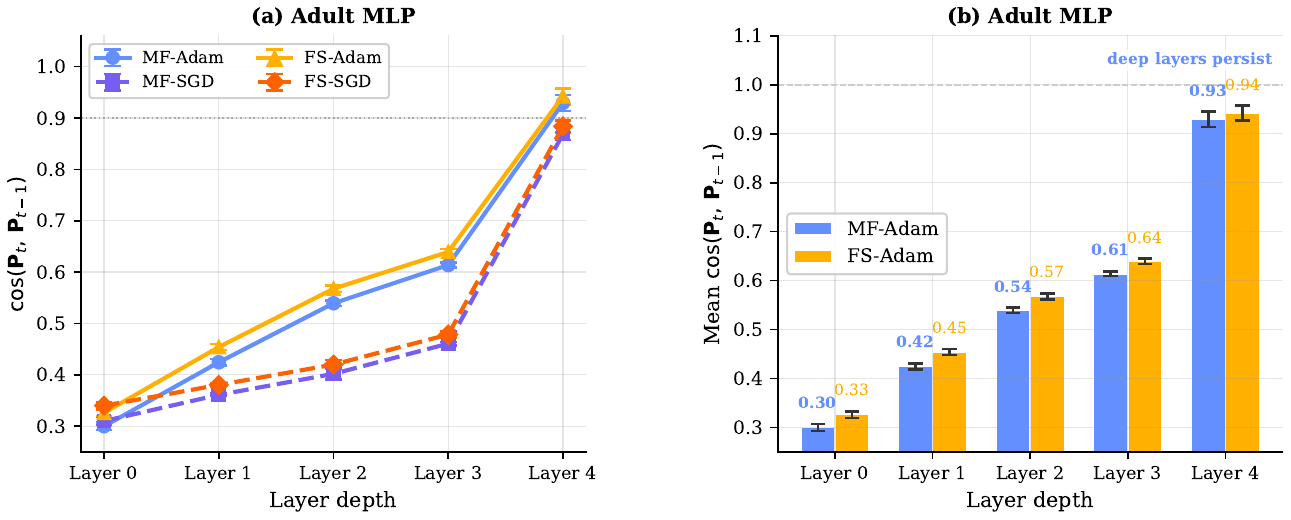}
    \caption{\textbf{Pressure persistence on Adult MLP.} Cosine similarity between consecutive pressure matrices $P_t$ and $P_{t-1}$ across module depths.}
    \label{fig:persistence}
\end{figure}

\subsection{Boundary of the paired claim}
\label{sec:claim_boundary}

The paired FS/MF results do not imply superiority over unconstrained or alternative spectral parametrizations. Table~\ref{tab:app_boundary} reports archived comparisons from the earlier protocol. On Covertype \citep{blackard1999covertype,uci1998repository}, unconstrained dense MLPs \citep{rumelhart1986learning,glorot2010understanding,nair2010relu} and diagonal-spectrum parametrizations can outperform the Stiefel-family models. Because these alternatives were not rerun under the convergence protocol, their archived values are used only to define the comparison scope. The converged Covertype result in Table~\ref{tab:tabular_image_mlp} therefore supports only the paired FS/MF comparison. The LSTM hidden-projection results similarly compare MF with Dense, $Q\operatorname{diag}(s)$, and Intrinsic Muon \citep{jordan2024muon,li2026intrinsicmuon}. These results delimit our claim to improving fixed-spectrum Stiefel parametrizations in settings where the Stiefel constraint is appropriate.

\subsection{Additional settings}
\label{sec:additional}

The remaining experiments examine boundary settings and implementation behavior. We report convolutional classifier heads separately because their results are mixed.

\paragraph{Convolutional classifier heads.}
Table~\ref{tab:additional_cnn} and Figure~\ref{fig:cnn_fchead} report three settings in which only the final classifier head is parametrized. Convolutional and residual networks already impose strong architectural structure \citep{lecun1998gradient,he2016resnet}. ResNet-50/CIFAR-100 favors MF under SGD but FS under Adam; ResNet-18/CIFAR-10 favors FS under both optimizers; and the Simple CNN result depends on the optimizer. These results show that learning $S$ does not consistently improve constrained classifier heads.

\begin{table}[!ht]
\centering
\caption{\textbf{Convolutional classifier-head convergence.} Only the final fully connected classifier head is parametrized. Higher test accuracy is better. Mixed signs indicate that MF does not consistently improve FS in this setting.}
\label{tab:additional_cnn}
\maxsizebox{\linewidth}{!}{%
\begin{tabular}{llccc}
\toprule
Setting & Optimizer & FS & MF & $\Delta$ \\
\midrule
ResNet-50 / CIFAR-100 & Adam & $\mathbf{54.43}$ & $53.44$ & $-0.99\%$ \\
ResNet-50 / CIFAR-100 & SGD & $48.79$ & $\mathbf{51.40}$ & $+2.61\%$ \\
ResNet-18 / CIFAR-10 & Adam & $\mathbf{86.31}$ & $85.79$ & $-0.52\%$ \\
ResNet-18 / CIFAR-10 & SGD & $\mathbf{83.79}$ & $83.27$ & $-0.52\%$ \\
Simple CNN / CIFAR-10 & Adam & $86.06$ & $\mathbf{86.44}$ & $+0.38\%$ \\
Simple CNN / CIFAR-10 & SGD & $\mathbf{86.54}$ & $85.79$ & $-0.75\%$ \\
\bottomrule
\end{tabular}
}
\end{table}

\begin{figure}[!ht]
    \centering
    \includegraphics[width=0.9\linewidth]{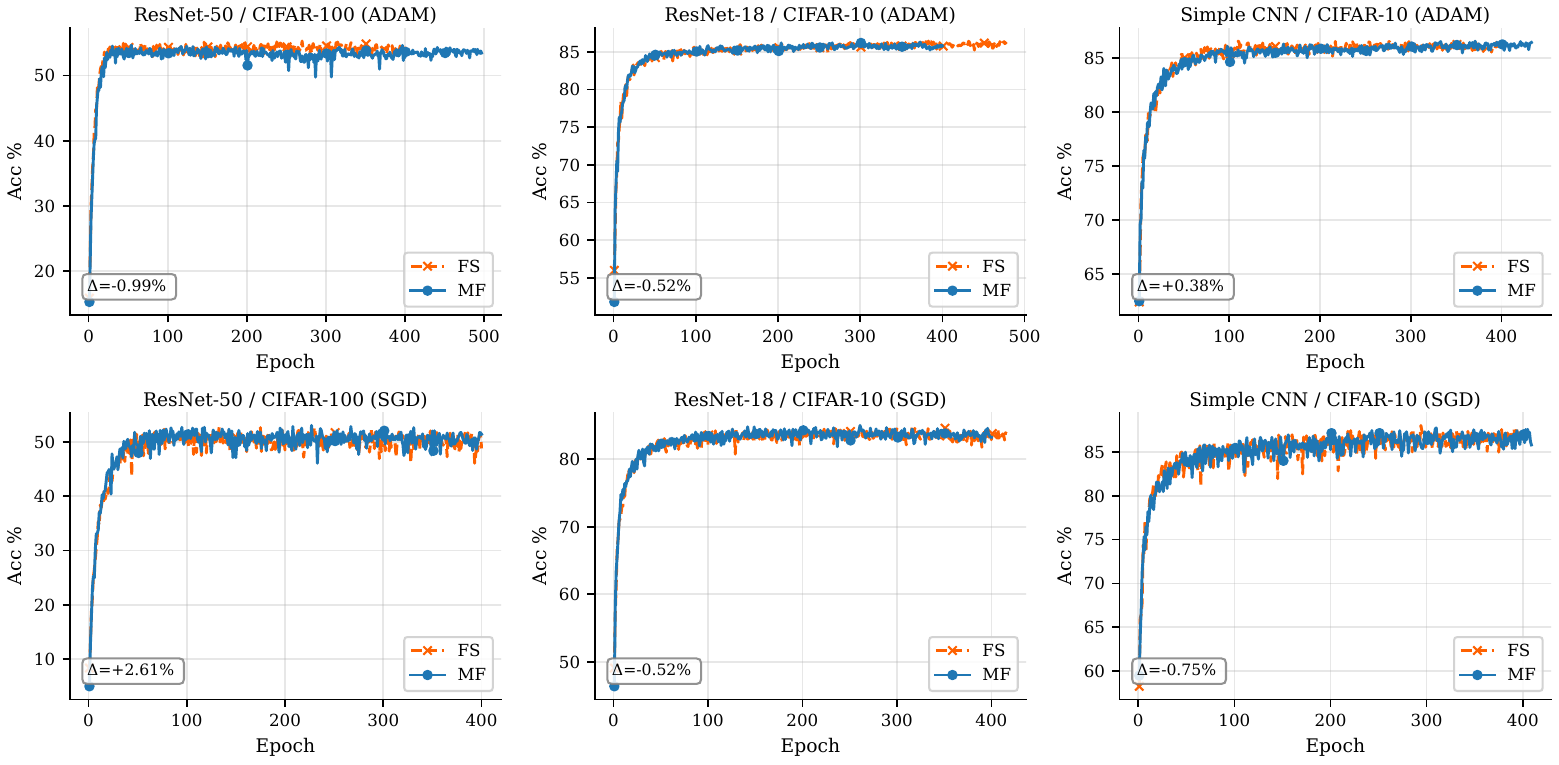}
    \caption{\textbf{Convolutional classifier-head trajectories.} Test accuracy curves for the settings in Table~\ref{tab:additional_cnn}.}
    \label{fig:cnn_fchead}
\end{figure}

\subsection{Spectral-bounded regime stability}
\label{sec:stability}

We use Adult as a stability test for bounded-spectrum parametrizations. The same four-block MLP is trained with Dense, fixed-Stiefel (FS), or \ours{} hidden maps while holding the optimizer, seed, batch size, and architecture fixed. Table~\ref{tab:stability} reports the epoch~500 accuracy, gradient-norm statistics, and loss variability. The purpose is to compare training behavior under the three parametrizations rather than to provide an additional accuracy benchmark.

\begin{table}[!ht]
    \centering
    \caption{\textbf{Stability statistics at epoch 500.} Adult MLP, single seed. Accuracy, gradient-norm peak, and loss standard deviation are reported at epoch~500; the ep.~491--500 column reports the mean gradient-norm peak over the final ten epochs. FS and MF have smaller gradient peaks than Dense in this run, with MF giving the smallest observed values.}
    \label{tab:stability}
    \maxsizebox{\linewidth}{!}{%
    \begin{tabular}{lcccc}
        \toprule
        Method & Test acc (ep.~500) & Gn-peak (ep.~500) & Gn-peak (ep.~491--500 mean) & Loss-std (ep.~500) \\
        \midrule
        Dense & 0.8209 & 18.66 & 11.47 & 0.0457 \\
        FS    & 0.8346 & 7.20  & 8.82  & 0.0509 \\
        MF    & \textbf{0.8351} & \textbf{5.97} & \textbf{6.89} & 0.0506 \\
        \bottomrule
    \end{tabular}
    }
\end{table}


The story is a stability one. Dense retains unconstrained spectral freedom, but in this long run that freedom comes with late accuracy drift and larger gradient excursions. FS removes those excursions by pinning the represented spectrum to one. \ours{} keeps the same orthonormal basis prior while allowing a clipped spectrum to move, and in this run that extra degree of freedom does not destabilize training; it preserves the FS accuracy profile and yields the smallest gradient peak. We therefore read the Adult stability cell consistently with Section~\ref{sec:claim_boundary}: it is not evidence for a robust MF accuracy advantage over FS on Adult, but it does show that learning a bounded SPD spectrum can improve the gradient-spike profile without giving up the stabilizing effect of the Stiefel prior.

\subsection{Summary}
\label{sec:summary}

Table~\ref{tab:summary} summarizes the converged paired results. MF improves most FS baselines, with the largest gains in recurrent language modeling and tabular MLPs, while image MLPs and convolutional classifier heads show more mixed results.

\begin{table}[!ht]
\centering
\caption{\textbf{Summary of converged paired results.} Positive values favor MF. Language-model results report perplexity reduction, $\mathrm{PPL}(\mathrm{FS})-\mathrm{PPL}(\mathrm{MF})$; classification results report accuracy change, $\mathrm{Acc}(\mathrm{MF})-\mathrm{Acc}(\mathrm{FS})$.}
\label{tab:summary}
\maxsizebox{\linewidth}{!}{%
\begin{tabular}{lrr@{\qquad}lrr}
\toprule
\multicolumn{3}{c}{\textbf{Language modeling: PPL decrease}} &
\multicolumn{3}{c}{\textbf{Classification: accuracy change}} \\
\cmidrule(lr){1-3}\cmidrule(lr){4-6}
Setting & Adam & SGD & Setting & Adam & SGD \\
\midrule
LSTM / WT2     & $+43.62$ & $\mathbf{+184.51}$ & Adult MLP          & $\mathbf{+1.50\%}$  & $+0.74\%$ \\
LSTM-256 / WT2 & $+2.79$  & $\mathbf{+165.39}$ & Covertype MLP      & $+11.37\%$ & $\mathbf{+11.39\%}$ \\
GRU / WT2      & $+58.45$ & $\mathbf{+218.61}$ & Fashion-MNIST MLP  & $+2.66\%$  & $\mathbf{+6.56\%}$ \\
LSTM / WT103   & $+32.78$ & $\mathbf{+132.99}$ & CIFAR-10 MLP       & $\mathbf{+2.41\%}$ & $-3.03\%$ \\
GRU-256 / WT2  & $+41.77$ & $\mathbf{+172.60}$ & ResNet-50 / C-100  & $-0.99\%$ & $\mathbf{+2.61\%}$ \\
GRU / WT103    & $+45.24$ & $\mathbf{+123.26}$ & ResNet-18 / C-10   & $\mathbf{-0.52\%}$ & $\mathbf{-0.52\%}$ \\
                &          &                    & Simple CNN / C-10  & $\mathbf{+0.38\%}$ & $-0.75\%$ \\
\bottomrule
\end{tabular}
}
\end{table}

\section{Theory and Proofs}
\label{app:mfv4_theory}
\label{app:spd_theory}

This appendix provides the proofs of the theoretical results in Section~\ref{sec:theory}. Throughout, all wrapped weights are written as
\begin{equation}
    W(Q,S)=QS^{1/2}, \qquad Q\in\St(p,r),\quad S\in\SPD(r),
    \label{eq:mfv4_weight_map}
\end{equation}
and the bounded SPD set is
\begin{equation}
    \mathcal S_{\alpha,\beta}=\{S\in\SPD(r):\alpha I_r\preceq S\preceq \beta I_r\},
    \label{eq:mfv4_clipped_spd_set}
\end{equation}
where $0<\alpha\leq\beta<\infty$. The corresponding product manifold is
\begin{equation}
    \mathcal M=\St(p,r)\times\mathcal S_{\alpha,\beta},
    \label{eq:mfv4_product_manifold}
\end{equation}
equipped with the canonical Stiefel metric and the affine-invariant SPD metric $g_S(U,V)=\tr(S^{-1}US^{-1}V)$ \citep{absil2008matrix,boumal2023smooth,pennec2006riemannian}. On the domain in Eq.~\eqref{eq:mfv4_product_manifold}, the convergence analysis considers the regularized objective
\begin{equation}
    \mathcal J(Q,S)=F(QS^{1/2})+\frac{\lambda_S}{2}\|\log S\|_F^2.
    \label{eq:mfv4_joint_objective}
\end{equation}
The proofs follow the order of the results in Section~\ref{sec:theory}. Assumptions and intermediate lemmas are restated where needed, while the main theoretical statements remain those given in Section~\ref{sec:theory}.

\theoremstyle{plain}
\renewcommand{\thelemma}{3.\arabic{lemma}}
\renewcommand{\thetheorem}{3.\arabic{theorem}}
\renewcommand{\thecorollary}{3.\arabic{corollary}}
\providecommand{\St}{\mathrm{St}}
\providecommand{\SPD}{\mathrm{SPD}}
\providecommand{\Sym}{\mathrm{Sym}}
\providecommand{\tr}{\mathrm{tr}}
\providecommand{\Exp}{\mathrm{Exp}}
\providecommand{\Log}{\mathrm{Log}}
\providecommand{\dist}{\mathrm{dist}}
\providecommand{\diag}{\mathrm{diag}}
\providecommand{\R}{\mathbb{R}}
\providecommand{\E}{\mathbb{E}}
\providecommand{\Prb}{\mathbb{P}}
\providecommand{\Unif}{\mathrm{Unif}}
\providecommand{\rank}{\mathrm{rank}}
\providecommand{\op}{\mathrm{op}}

\subsection{Pressure as Rejected Geometry Gradient}
\label{app:pressure_geometry}

\subsubsection{Proof of Proposition~\ref{prop:pressure_decomposition}}
\label{app:proof_pressure_decomposition}

\begin{proposition}[Restatement of Proposition~\ref{prop:pressure_decomposition}]
Let $\ell(W)$ be differentiable, let $W=QS^{1/2}$ as in Eq.~\eqref{eq:mfv4_weight_map}, and let $G_W=\nabla_W\ell(W)$. Define $\bar G=G_WS^{1/2}$ and
\begin{equation}
    \Pi_Q(\bar G)=\bar G-Q\sym(Q^\top\bar G), \qquad P=\sym(Q^\top\bar G).
    \label{eq:app_pressure_projection}
\end{equation}
Then $\Pi_Q(\bar G)\in T_Q\St(p,r)$, $\bar G=\Pi_Q(\bar G)+QP$, and at the fixed-spectrum point $S=I_r$,
\begin{align}
    D_S\ell(QS^{1/2})[U]
    &=\frac12\langle \sym(Q^\top G_W),U\rangle_F \notag\\
    &\qquad \forall U\in\Sym(r).
    \label{eq:app_pressure_gram_signal}
\end{align}
Thus, Eq.~\eqref{eq:app_pressure_projection} separates the ambient gradient into the Stiefel tangent component and a symmetric rejected coefficient $P$, while Eq.~\eqref{eq:app_pressure_gram_signal} identifies $P$ with the first-order signal for changing the Gram factor at $S=I_r$.
\end{proposition}

\begin{proof}
By definition, $\Delta\in T_Q\St(p,r)$ if and only if $Q^\top\Delta+\Delta^\top Q=0$. Substituting the projected gradient from Eq.~\eqref{eq:app_pressure_projection} gives
\begin{align}
    Q^\top\Pi_Q(\bar G)+\Pi_Q(\bar G)^\top Q
    &=Q^\top\bar G-\sym(Q^\top\bar G)+\bar G^\top Q-\sym(Q^\top\bar G) \notag\\
    &=Q^\top\bar G+\bar G^\top Q-2\sym(Q^\top\bar G) \notag\\
    &=0.
    \label{eq:app_pressure_tangent_check}
\end{align}
Hence $\Pi_Q(\bar G)\in T_Q\St(p,r)$. Using the definition of $P$ in Eq.~\eqref{eq:app_pressure_projection}, the complementary decomposition is
\begin{align}
    \Pi_Q(\bar G)+QP
    &=\bar G-Q\sym(Q^\top\bar G)+Q\sym(Q^\top\bar G) \notag\\
    &=\bar G.
    \label{eq:app_pressure_decomposition}
\end{align}
Equations~\eqref{eq:app_pressure_tangent_check}--\eqref{eq:app_pressure_decomposition} therefore establish that $QP$ is exactly the component of the ambient gradient excluded by the Stiefel tangent projection.

It remains to connect this rejected component to variations of the SPD factor in Eq.~\eqref{eq:mfv4_weight_map}. Fix $Q$ and define $\phi(S)=\ell(QS^{1/2})$. At $S=I_r$, the Fr\'echet derivative of the principal square-root map is $D(S^{1/2})|_{I_r}[U]=U/2$, since differentiating $X^2=S$ at $X=I_r$ gives $Y+Y=U$ and hence $Y=U/2$ \citep{higham2008functions}. Therefore, for every $U\in\Sym(r)$,
\begin{align}
    D\phi(I_r)[U]
    &=\left\langle G_W,Q\,D(S^{1/2})|_{I_r}[U]\right\rangle_F \notag\\
    &=\frac12\langle Q^\top G_W,U\rangle_F \notag\\
    &=\frac12\langle \sym(Q^\top G_W),U\rangle_F.
    \label{eq:app_pressure_directional_derivative}
\end{align}
The last equality follows because $U\in\Sym(r)$ is Frobenius-orthogonal to the skew-symmetric part of $Q^\top G_W$. Comparing Eq.~\eqref{eq:app_pressure_directional_derivative} with the definition of $P$ in Eq.~\eqref{eq:app_pressure_projection} yields Eq.~\eqref{eq:app_pressure_gram_signal}. Thus, at $S=I_r$, the symmetric coefficient of the rejected Stiefel component is precisely the first-order loss signal along symmetric perturbations of the SPD Gram factor.
\end{proof}

\subsection{Basic SPD Geometry}
\label{app:mfv4_retained_facts}

We first establish the basic SPD properties required by the update. Proposition~\ref{prop:update_stability} proves that the exponential SPD step preserves positive definiteness. Proposition~\ref{prop:alignment_clipping} establishes that Procrustes alignment preserves the spectral and norm structure of the pressure state and that eigenvalue clipping maps the SPD factor into the bounded domain in Eq.~\eqref{eq:mfv4_clipped_spd_set}. Finally, Lemma~\ref{lem:mfv4_ai_distance} gives the affine-invariant distance used in Proposition~\ref{prop:geo_equiv} to compare local SPD movements.

\renewcommand{\theproposition}{\arabic{proposition}}
\renewcommand{\theHproposition}{retained.\arabic{proposition}}
\setcounter{proposition}{2}

\subsubsection{Proof of Proposition~\ref{prop:update_stability}}
\label{app:proof_update_stability}

\begin{proposition}[Restatement of Proposition~\ref{prop:update_stability}]
Let $S_t\in\SPD(r)$, let $A_t\in\Sym(r)$, let $\gamma_t>0$, and let $a_t\in[0,1]$. Define
\begin{equation}
    S_{t+1}=S_t^{1/2}\exp\left(-\gamma_ta_tA_t\right)S_t^{1/2}.
    \label{eq:mfv4_spd_update}
\end{equation}
Then $S_{t+1}\in\SPD(r)$.
\end{proposition}

\begin{proof}
Define
\begin{equation}
    B_t=-\gamma_ta_tA_t.
    \label{eq:mfv4_update_A}
\end{equation}
Since $A_t\in\Sym(r)$, Eq.~\eqref{eq:mfv4_update_A} implies $B_t^\top=B_t$. By the spectral theorem for real symmetric matrices \citep{bhatia1997matrix}, there exist $U\in\mathrm O(r)$ and $\mu_1,\ldots,\mu_r\in\R$ such that $B_t=U\diag(\mu_1,\ldots,\mu_r)U^\top$. Therefore,
\begin{equation}
    \exp(B_t)=U\diag(e^{\mu_1},\ldots,e^{\mu_r})U^\top\succ0.
    \label{eq:mfv4_exp_spd}
\end{equation}
For any nonzero $x\in\R^r$, define $y=S_t^{1/2}x$. Since $S_t\in\SPD(r)$, its principal square root $S_t^{1/2}$ is nonsingular, and hence $y\neq0$. Combining Eqs.~\eqref{eq:mfv4_spd_update} and \eqref{eq:mfv4_exp_spd} gives
\begin{align}
    x^\top S_{t+1}x
    &=x^\top S_t^{1/2}\exp(B_t)S_t^{1/2}x \notag\\
    &=y^\top\exp(B_t)y \notag\\
    &>0.
    \label{eq:mfv4_update_pd_quadratic}
\end{align}
Since Eq.~\eqref{eq:mfv4_update_pd_quadratic} holds for every nonzero $x$, the quadratic-form characterization of positive definiteness \citep{bhatia1997matrix} yields $S_{t+1}\in\SPD(r)$.
\end{proof}

\subsubsection{Proof of Proposition~\ref{prop:alignment_clipping}}
\label{app:proof_alignment_clipping}

\begin{proposition}[Restatement of Proposition~\ref{prop:alignment_clipping}]
Let $Q,Q_-\in\St(p,r)$ and let $Q^\top Q_-=U\Sigma V^\top$. Then $O_\star=UV^\top$ solves the orthogonal Procrustes problem
\begin{equation}
    \min_{O\in\mathrm O(r)}\|QO-Q_-\|_F^2.
    \label{eq:app_procrustes_problem}
\end{equation}
For every $M\in\Sym(r)$, $O_\star MO_\star^\top$ is symmetric and has the same eigenvalues and Frobenius norm as $M$. If $S=U_S\diag(\lambda_i)U_S^\top\in\SPD(r)$, then $\mathcal C_{\alpha,\beta}(S)=U_S\diag(\clip(\lambda_i,\alpha,\beta))U_S^\top$ belongs to $\mathcal S_{\alpha,\beta}$ defined in Eq.~\eqref{eq:mfv4_clipped_spd_set}.
\end{proposition}

\begin{proof}
Since $Q^\top Q=Q_-^\top Q_-=I_r$, expanding the objective in Eq.~\eqref{eq:app_procrustes_problem} gives
\begin{align}
    \|QO-Q_-\|_F^2
    &=\tr(O^\top Q^\top QO)-2\tr(O^\top Q^\top Q_-)+\tr(Q_-^\top Q_-) \notag\\
    &=2r-2\tr(O^\top Q^\top Q_-).
    \label{eq:app_procrustes_trace}
\end{align}
Thus, minimizing Eq.~\eqref{eq:app_procrustes_problem} is equivalent to maximizing the trace term in Eq.~\eqref{eq:app_procrustes_trace}. Using $Q^\top Q_-=U\Sigma V^\top$,
\begin{align}
    \tr(O^\top Q^\top Q_-)
    &=\tr(O^\top U\Sigma V^\top) \notag\\
    &=\tr((U^\top OV)^\top\Sigma) \notag\\
    &\leq\tr(\Sigma),
    \label{eq:app_procrustes_von_neumann}
\end{align}
where the last inequality follows from von Neumann's trace inequality. Equality is attained by $O=UV^\top=O_\star$, because $U^\top O_\star V=I_r$. Hence $O_\star$ solves Eq.~\eqref{eq:app_procrustes_problem}.

Next, let $M=M^\top$. Since $O_\star^\top O_\star=I_r$,
\begin{equation}
    (O_\star MO_\star^\top)^\top=O_\star M^\top O_\star^\top=O_\star MO_\star^\top,
\end{equation}
so the aligned state remains symmetric. Moreover, $O_\star MO_\star^\top$ is orthogonally similar to $M$ and therefore has exactly the same eigenvalues. Its Frobenius norm is also invariant:
\begin{align}
    \|O_\star MO_\star^\top\|_F^2
    &=\tr(O_\star MO_\star^\top O_\star MO_\star^\top) \notag\\
    &=\tr(O_\star M^2O_\star^\top) \notag\\
    &=\tr(M^2) \notag\\
    &=\|M\|_F^2.
    \label{eq:app_alignment_norm}
\end{align}
Thus, the Procrustes transport changes only the coordinate representation of $M$ and preserves its symmetry, spectrum, and Frobenius norm.

Finally, because $S\in\SPD(r)$, every eigenvalue satisfies $\lambda_i>0$. By definition, $\clip(\lambda_i,\alpha,\beta)\in[\alpha,\beta]$ for every $i$. Therefore,
\begin{equation}
    \alpha I_r\preceq\mathcal C_{\alpha,\beta}(S)\preceq\beta I_r,
    \label{eq:app_clip_bounds}
\end{equation}
and Eq.~\eqref{eq:app_clip_bounds}, together with Eq.~\eqref{eq:mfv4_clipped_spd_set}, yields $\mathcal C_{\alpha,\beta}(S)\in\mathcal S_{\alpha,\beta}$.
\end{proof}

\begin{lemma}[Affine-invariant SPD distance]
\label{lem:mfv4_ai_distance}
Let $S,T\in\SPD(r)$. Under the affine-invariant metric $g_S(U,V)=\tr(S^{-1}US^{-1}V)$, the geodesic distance is
\begin{equation}
    d_{\mathrm{AI}}(S,T)=\left\|\log\left(S^{-1/2}TS^{-1/2}\right)\right\|_F.
    \label{eq:mfv4_ai_distance}
\end{equation}
\end{lemma}

\begin{proof}
Under the affine-invariant metric, the exponential and logarithm maps on $\SPD(r)$ are \citep{pennec2006riemannian,moakher2005differential,boumal2023smooth}
\begin{align}
    \Exp_S(X)&=S^{1/2}\exp\left(S^{-1/2}XS^{-1/2}\right)S^{1/2},
    \label{eq:mfv4_ai_exp}\\
    \Log_S(T)&=S^{1/2}\log\left(S^{-1/2}TS^{-1/2}\right)S^{1/2}.
    \label{eq:mfv4_ai_log}
\end{align}
The affine-invariant distance is the Riemannian norm of the logarithm vector, $d_{\mathrm{AI}}(S,T)=\|\Log_S(T)\|_S$. Let $B=S^{-1/2}TS^{-1/2}$. Substituting Eq.~\eqref{eq:mfv4_ai_log} into the metric gives
\begin{align}
    d_{\mathrm{AI}}(S,T)^2
    &=g_S(\Log_S(T),\Log_S(T)) \notag\\
    &=\tr\left(S^{-1}\Log_S(T)S^{-1}\Log_S(T)\right) \notag\\
    &=\tr\left(\log(B)^2\right) \notag\\
    &=\|\log(B)\|_F^2 \notag\\
    &=\left\|\log\left(S^{-1/2}TS^{-1/2}\right)\right\|_F^2.
    \label{eq:mfv4_ai_distance_proof}
\end{align}
Taking the nonnegative square root of Eq.~\eqref{eq:mfv4_ai_distance_proof} gives Eq.~\eqref{eq:mfv4_ai_distance}.
\end{proof}

\subsubsection{Proof of Proposition~\ref{prop:geo_equiv}}
\label{app:proof_geo_equiv}

\begin{proposition}[Restatement of Proposition~\ref{prop:geo_equiv}]
Let $S_t=\lambda I_r$ with $\lambda>0$, let $M_t,M_t'\in\Sym(r)$ satisfy $\|M_t\|_F=\|M_t'\|_F$, and define
\begin{equation}
    S_{t+1}(M)=S_t^{1/2}\exp\!\left(-\rho_{\mathrm{geo}}a_tS_t^{-1/2}MS_t^{-1/2}\right)S_t^{1/2}.
    \label{eq:mfv4_pressure_spd_update}
\end{equation}
Then
\begin{equation}
    d_{\mathrm{AI}}\left(S_t,S_{t+1}(M_t)\right)=d_{\mathrm{AI}}\left(S_t,S_{t+1}(M_t')\right).
    \label{eq:mfv4_geo_equiv}
\end{equation}
Moreover,
\begin{equation}
    d_{\mathrm{AI}}\left(S_t,S_{t+1}(M_t)\right)=\rho_{\mathrm{geo}}a_t\lambda^{-1}\|M_t\|_F.
    \label{eq:mfv4_geo_equiv_explicit}
\end{equation}
\end{proposition}

\begin{proof}
Since $S_t=\lambda I_r$, we have $S_t^{-1/2}=\lambda^{-1/2}I_r$. Substituting this identity into Eq.~\eqref{eq:mfv4_pressure_spd_update} yields
\begin{equation}
    S_t^{-1/2}S_{t+1}(M_t)S_t^{-1/2}=\exp\left(-\rho_{\mathrm{geo}}a_t\lambda^{-1}M_t\right).
    \label{eq:mfv4_geo_log_argument}
\end{equation}
Because $M_t\in\Sym(r)$, the matrix $-\rho_{\mathrm{geo}}a_t\lambda^{-1}M_t$ is symmetric. Hence its exponential is SPD and its principal logarithm satisfies $\log(\exp(B))=B$ \citep{higham2008functions}. Applying Lemma~\ref{lem:mfv4_ai_distance} to Eq.~\eqref{eq:mfv4_geo_log_argument} gives
\begin{align}
    d_{\mathrm{AI}}\left(S_t,S_{t+1}(M_t)\right)
    &=\left\|\log\left(S_t^{-1/2}S_{t+1}(M_t)S_t^{-1/2}\right)\right\|_F \notag\\
    &=\left\|-\rho_{\mathrm{geo}}a_t\lambda^{-1}M_t\right\|_F \notag\\
    &=\rho_{\mathrm{geo}}a_t\lambda^{-1}\|M_t\|_F.
    \label{eq:mfv4_geo_equiv_proof}
\end{align}
Equation~\eqref{eq:mfv4_geo_equiv_proof} proves Eq.~\eqref{eq:mfv4_geo_equiv_explicit}. Repeating the same calculation for $M_t'$ gives $d_{\mathrm{AI}}(S_t,S_{t+1}(M_t'))=\rho_{\mathrm{geo}}a_t\lambda^{-1}\|M_t'\|_F$. Since $\|M_t\|_F=\|M_t'\|_F$, Eq.~\eqref{eq:mfv4_geo_equiv} follows.
\end{proof}

\subsection{Convergence}
\label{app:mfv4_convergence}

We next establish convergence on the product space in Eq.~\eqref{eq:mfv4_product_manifold}. The argument follows the standard stochastic Riemannian gradient analysis: smoothness controls the second-order change of the objective, retraction regularity transfers this control to the implemented update, unbiasedness controls the stochastic first-order term, and summation yields a nonconvex stationarity bound \citep{absil2008matrix,boumal2023smooth,bonnabel2013stochastic}.

\begin{assumption}[Geodesic smoothness]
\label{ass:mfv4_smoothness}
The objective $\mathcal J$ in Eq.~\eqref{eq:mfv4_joint_objective} is bounded below by $\mathcal J_\star>-\infty$ on $\mathcal M$ and is geodesically $L_{\mathcal M}$-smooth: for every $x=(Q,S)\in\mathcal M$, every $\xi\in T_x\mathcal M$, and every geodesic $\gamma$ satisfying $\gamma(0)=x$ and $\dot\gamma(0)=\xi$,
\begin{equation}
    \mathcal J(\gamma(1))\leq \mathcal J(x)+\langle \operatorname{grad}\mathcal J(x),\xi\rangle_x+\frac{L_{\mathcal M}}{2}\|\xi\|_x^2.
    \label{eq:mfv4_geodesic_smoothness}
\end{equation}
The retraction $R$ used by ManifoldFlow satisfies
\begin{equation}
    \mathcal J(R_x(\xi))\leq \mathcal J(x)+\langle \operatorname{grad}\mathcal J(x),\xi\rangle_x+\frac{L_R}{2}\|\xi\|_x^2
    \label{eq:mfv4_retraction_smoothness}
\end{equation}
for all iterates and tangent vectors used by the algorithm, where $L_R\geq L_{\mathcal M}$ depends only on $L_{\mathcal M}$ and the second-order retraction constants over the bounded domain in Eq.~\eqref{eq:mfv4_clipped_spd_set}.
\end{assumption}

\begin{assumption}[Bounded SPD spectrum]
\label{ass:mfv4_bounded_spectrum}
For every iterate $t$, the spectral clipping and gate damping maintain
\begin{equation}
    \alpha I_r\preceq S_t\preceq\beta I_r.
    \label{eq:mfv4_spectrum_bound}
\end{equation}
Equivalently, $S_t\in\mathcal S_{\alpha,\beta}$ as defined in Eq.~\eqref{eq:mfv4_clipped_spd_set}.
\end{assumption}

\begin{lemma}[Norm equivalence on the clipped SPD cone]
\label{lem:mfv4_norm_equivalence}
If $S\in\mathcal S_{\alpha,\beta}$ and $U\in T_S\SPD(r)=\Sym(r)$, then
\begin{equation}
    \beta^{-2}\|U\|_F^2\leq\|U\|_{S}^2\leq\alpha^{-2}\|U\|_F^2.
    \label{eq:mfv4_spd_norm_equivalence}
\end{equation}
\end{lemma}

\begin{proof}
Let $S=V\Lambda V^\top$ with $V\in\mathrm O(r)$ and $\Lambda=\diag(\lambda_1,\ldots,\lambda_r)$. Since $S\in\mathcal S_{\alpha,\beta}$, Eq.~\eqref{eq:mfv4_clipped_spd_set} gives $\alpha\leq\lambda_i\leq\beta$ for every $i$. Define $\widetilde U=V^\top UV$. Using the affine-invariant metric and orthogonal invariance,
\begin{align}
    \|U\|_{S}^{2}
    &=\tr(S^{-1}US^{-1}U) \notag\\
    &=\tr(\Lambda^{-1}\widetilde U\Lambda^{-1}\widetilde U) \notag\\
    &=\sum_{i=1}^{r}\sum_{j=1}^{r}\lambda_i^{-1}\lambda_j^{-1}\widetilde U_{ij}^{2}.
    \label{eq:mfv4_norm_entrywise}
\end{align}
Since $\alpha\leq\lambda_i,\lambda_j\leq\beta$, we have $\beta^{-2}\leq\lambda_i^{-1}\lambda_j^{-1}\leq\alpha^{-2}$. Applying these bounds termwise to Eq.~\eqref{eq:mfv4_norm_entrywise} gives
\begin{equation}
    \beta^{-2}\sum_{i,j}\widetilde U_{ij}^{2}\leq\|U\|_{S}^{2}\leq\alpha^{-2}\sum_{i,j}\widetilde U_{ij}^{2}.
    \label{eq:mfv4_norm_bound_entrywise}
\end{equation}
Finally, $\sum_{i,j}\widetilde U_{ij}^{2}=\|\widetilde U\|_F^2=\|U\|_F^2$ by orthogonal invariance of the Frobenius norm. Substituting this identity into Eq.~\eqref{eq:mfv4_norm_bound_entrywise} proves Eq.~\eqref{eq:mfv4_spd_norm_equivalence}.
\end{proof}

\begin{assumption}[Bounded pressure and stochastic gradient estimator]
\label{ass:mfv4_stochastic}
Let $x_t=(Q_t,S_t)$, let $g_t=\operatorname{grad}\mathcal J(x_t)$, and let $\widehat g_t\in T_{x_t}\mathcal M$ denote the stochastic tangent estimator used by ManifoldFlow. With respect to the filtration $\mathcal F_t$ generated before iteration $t$, assume
\begin{equation}
    \E[\widehat g_t\mid\mathcal F_t]=g_t,
    \label{eq:mfv4_unbiased_gradient}
\end{equation}
and let $\nu_t^2$ satisfy
\begin{equation}
    \E[\|\widehat g_t-g_t\|_{x_t}^2\mid\mathcal F_t]\leq\nu_t^2.
    \label{eq:mfv4_variance_bound}
\end{equation}
Assume further that the pressure direction is symmetric and satisfies $\E[\|M_t\|_F^2\mid\mathcal F_t]\leq P^2$. Under Assumption~\ref{ass:mfv4_bounded_spectrum}, Lemma~\ref{lem:mfv4_norm_equivalence} bounds the induced SPD tangent norm by a constant $V_{\alpha,\beta}$ depending only on $P$, $\alpha$, and $\beta$. For a horizon $T$, define the accumulated variance term
\begin{equation}
    \overline\nu_T^2=\sum_{t=1}^{T}\eta_t^2\nu_t^2.
    \label{eq:mfv4_accumulated_variance}
\end{equation}
\end{assumption}

\begin{lemma}[Riemannian Descent Lemma]
\label{lem:mfv4_descent}
Under Assumption~\ref{ass:mfv4_smoothness}, let $x^+=R_x(-\eta h)$ for $x\in\mathcal M$, $h\in T_x\mathcal M$, and $\eta>0$. Then
\begin{equation}
    \mathcal J(x^+)\leq\mathcal J(x)-\eta\langle\operatorname{grad}\mathcal J(x),h\rangle_x+\frac{L_R\eta^2}{2}\|h\|_x^2.
    \label{eq:mfv4_descent_basic}
\end{equation}
If $\langle\operatorname{grad}\mathcal J(x),h\rangle_x\geq\frac12\|\operatorname{grad}\mathcal J(x)\|_x^2$, then
\begin{equation}
    \mathcal J(x^+)\leq\mathcal J(x)-\frac{\eta}{2}\|\operatorname{grad}\mathcal J(x)\|_x^2+\frac{L_R\eta^2}{2}\|h\|_x^2.
    \label{eq:mfv4_descent_half}
\end{equation}
\end{lemma}

\begin{proof}
Set $\xi=-\eta h$. Applying the retraction smoothness inequality in Eq.~\eqref{eq:mfv4_retraction_smoothness} gives
\begin{align}
    \mathcal J(x^+)
    &=\mathcal J(R_x(-\eta h)) \notag\\
    &\leq\mathcal J(x)+\langle\operatorname{grad}\mathcal J(x),-\eta h\rangle_x+\frac{L_R}{2}\|-\eta h\|_x^2 \notag\\
    &=\mathcal J(x)-\eta\langle\operatorname{grad}\mathcal J(x),h\rangle_x+\frac{L_R\eta^2}{2}\|h\|_x^2.
    \label{eq:mfv4_descent_proof_1}
\end{align}
Thus Eq.~\eqref{eq:mfv4_descent_proof_1} is exactly Eq.~\eqref{eq:mfv4_descent_basic}. Under the additional condition $\langle\operatorname{grad}\mathcal J(x),h\rangle_x\geq\frac12\|\operatorname{grad}\mathcal J(x)\|_x^2$, substituting this lower bound into the negative first-order term of Eq.~\eqref{eq:mfv4_descent_basic} immediately yields Eq.~\eqref{eq:mfv4_descent_half}. The retraction inequality in Eq.~\eqref{eq:mfv4_retraction_smoothness} follows from the second-order agreement of the retraction with the exponential map together with the smoothness assumption in Eq.~\eqref{eq:mfv4_geodesic_smoothness} \citep{absil2008matrix,boumal2023smooth}.
\end{proof}

\subsubsection{Proof of Theorem~\ref{thm:mf_convergence}}
\label{app:proof_convergence}

\begin{theorem}[Restatement of Theorem~\ref{thm:mf_convergence}]
Under Assumptions~\ref{ass:mfv4_smoothness}--\ref{ass:mfv4_stochastic}, let $x_{t+1}=R_{x_t}(-\eta_t\widehat g_t)$. If $0<\eta_t\leq L_R^{-1}$ and $R$ is sampled from $\{1,\ldots,T\}$ with probability proportional to $\eta_t$, then
\begin{equation}
    \E\|\operatorname{grad}\mathcal J(x_R)\|_{x_R}^2
    \leq
    \frac{2(\mathcal J(x_1)-\mathcal J_\star)+L_R\overline\nu_T^2}{\sum_{t=1}^{T}\eta_t}.
    \label{eq:mfv4_general_rate}
\end{equation}
For the schedule $\eta_t=c/\sqrt t$ with $0<c\leq L_R^{-1}$ and $\overline\nu_T^2\leq C_\nu$, the iterates satisfy
\begin{equation}
    \begin{aligned}
    \min_{1\leq t\leq T}
    \E\|\operatorname{grad}\mathcal J(x_t)\|_{x_t}^2
    &\leq
    \frac{2(\mathcal J(x_1)-\mathcal J_\star)+L_RC_\nu}
    {2c(\sqrt{T+1}-1)}
    \\
    &=\mathcal O(T^{-1/2}).
    \end{aligned}
    \label{eq:mfv4_sqrt_rate}
\end{equation}
With uniformly bounded stochastic variance $\nu_t^2\leq\nu^2$, the same derivation gives the explicit bound obtained by replacing $C_\nu$ in \eqref{eq:mfv4_sqrt_rate} with $c^2\nu^2(1+\log T)$.
\end{theorem}

\begin{proof}
Apply Lemma~\ref{lem:mfv4_descent} with $h=\widehat g_t$ and condition on $\mathcal F_t$. Equation \eqref{eq:mfv4_descent_basic} gives
\begin{align}
    \E[\mathcal J(x_{t+1})\mid\mathcal F_t]
    &\leq \mathcal J(x_t)
    -\eta_t\left\langle g_t,\E[\widehat g_t\mid\mathcal F_t]\right\rangle_{x_t}+\frac{L_R\eta_t^2}{2}\E[\|\widehat g_t\|_{x_t}^2\mid\mathcal F_t] \notag \\
    &=\mathcal J(x_t)-\eta_t\|g_t\|_{x_t}^2+\frac{L_R\eta_t^2}{2}
    \E[\|\widehat g_t\|_{x_t}^2\mid\mathcal F_t].
    \label{eq:mfv4_conditional_descent_1}
\end{align}
The variance identity in a Hilbert space gives
\begin{align}
    \E[\|\widehat g_t\|_{x_t}^2\mid\mathcal F_t]
    &=\|g_t\|_{x_t}^2+\E[\|\widehat g_t-g_t\|_{x_t}^2\mid\mathcal F_t] \notag \\
    &\leq \|g_t\|_{x_t}^2+\nu_t^2.
    \label{eq:mfv4_second_moment_bound}
\end{align}
Substituting \eqref{eq:mfv4_second_moment_bound} into \eqref{eq:mfv4_conditional_descent_1} yields
\begin{align}
    \E[\mathcal J(x_{t+1})\mid\mathcal F_t]
    &\leq \mathcal J(x_t)
    -\eta_t\left(1-\frac{L_R\eta_t}{2}\right)
    \|g_t\|_{x_t}^2+\frac{L_R\eta_t^2}{2}\nu_t^2 \notag \\
    &\leq \mathcal J(x_t)-\frac{\eta_t}{2}\|g_t\|_{x_t}^2
    +\frac{L_R\eta_t^2}{2}\nu_t^2.
    \label{eq:mfv4_conditional_descent_2}
\end{align}
Taking total expectation and rearranging \eqref{eq:mfv4_conditional_descent_2} gives
\begin{equation}
    \eta_t\E\|g_t\|_{x_t}^2\leq 2\E\mathcal J(x_t)-2\E\mathcal J(x_{t+1})+L_R\eta_t^2\nu_t^2.
    \label{eq:mfv4_rearranged_descent}
\end{equation}
Summing \eqref{eq:mfv4_rearranged_descent} over $t=1,\ldots,T$ gives
\begin{align}
    \sum_{t=1}^{T}\eta_t\E\|g_t\|_{x_t}^2
    &\leq
    2\sum_{t=1}^{T}
    \left(\E\mathcal J(x_t)-\E\mathcal J(x_{t+1})\right)+L_R\sum_{t=1}^{T}\eta_t^2\nu_t^2 \notag \\
    &=2\E\mathcal J(x_1)-2\E\mathcal J(x_{T+1})
    +L_R\overline\nu_T^2 \notag \\
    &\leq2(\mathcal J(x_1)-\mathcal J_\star)+L_R\overline\nu_T^2.
    \label{eq:mfv4_telescoping}
\end{align}
Since $\Prb(R=t)=\eta_t/\sum_{s=1}^{T}\eta_s$,
\begin{equation}
    \E\|g_R\|_{x_R}^2=\frac{\sum_{t=1}^{T}\eta_t\E\|g_t\|_{x_t}^2}{\sum_{t=1}^{T}\eta_t}.
    \label{eq:mfv4_random_iterate_identity}
\end{equation}
Combining \eqref{eq:mfv4_telescoping} and \eqref{eq:mfv4_random_iterate_identity} proves \eqref{eq:mfv4_general_rate}. For $\eta_t=c/\sqrt t$,
\begin{equation}
    \sum_{t=1}^{T}\eta_t=c\sum_{t=1}^{T}t^{-1/2}\geq 2c(\sqrt{T+1}-1),
    \label{eq:mfv4_stepsize_sum}
\end{equation}
where the inequality follows from the integral lower bound for a decreasing positive function. Substitution of \eqref{eq:mfv4_stepsize_sum} and $\overline\nu_T^2\leq C_\nu$ into \eqref{eq:mfv4_general_rate} gives \eqref{eq:mfv4_sqrt_rate}. If $\nu_t^2\leq\nu^2$, then
\begin{equation}
    \overline\nu_T^2=\sum_{t=1}^{T}\frac{c^2\nu_t^2}{t}\leq c^2\nu^2\left(1+\log T\right),
    \label{eq:mfv4_log_variance}
\end{equation}
and substituting \eqref{eq:mfv4_log_variance} into \eqref{eq:mfv4_general_rate} gives the stated logarithmic variant.
\end{proof}

\subsubsection{Proof of Theorem~\ref{thm:two_block_convergence}}
\label{app:proof_two_block_convergence}

\begin{theorem}[Restatement of Theorem~\ref{thm:two_block_convergence}]
Let $\mathcal J$ be retraction-smooth and bounded below on $\St(p,r)\times\mathcal S_{\alpha,\beta}$. Suppose the $Q$-block uses an unbiased tangent estimator $\widehat g_{Q,t}$ with conditional variance at most $\sigma_Q^2$. Suppose the SPD block uses a symmetric estimator $\widehat g_{S,t}$ with conditional variance at most $\sigma_S^2$ and bias satisfying Eq.~\eqref{eq:main_geometry_bias}. If the step sizes $\eta_t$ and $\gamma_t$ are sufficiently small for the descent inequality to hold, and $R$ is sampled with probability proportional to $\eta_t+\gamma_t$, then Eq.~\eqref{eq:main_two_block_rate} holds.
\end{theorem}

\begin{proof}
Define
\begin{equation}
    g_{Q,t}=\grad_Q\mathcal J(Q_t,S_t), \qquad g_{S,t}=\grad_S\mathcal J(Q_t,S_t).
\end{equation}
Apply the retraction-smooth descent inequality in Eq.~\eqref{eq:mfv4_retraction_smoothness} to the two-block update
\begin{equation}
    Q_{t+1}=\Retr_{Q_t}(-\eta_t\widehat g_{Q,t}), \qquad S_{t+1}=\Retr_{S_t}(-\gamma_t\widehat g_{S,t}),
\end{equation}
where the SPD update uses the affine-invariant exponential map followed by spectral clipping. On the bounded domain in Eq.~\eqref{eq:mfv4_clipped_spd_set}, the effect of clipping can be absorbed into fixed local smoothness constants $L_Q$ and $L_S$. Conditioning on $\mathcal F_t$ therefore gives
\begin{align}
    \E_t[\mathcal J_{t+1}]
    &\leq\mathcal J_t-\eta_t\langle g_{Q,t},\E_t\widehat g_{Q,t}\rangle-\gamma_t\langle g_{S,t},\E_t\widehat g_{S,t}\rangle_{S_t} +\frac{L_Q\eta_t^2}{2}\E_t\|\widehat g_{Q,t}\|^2+\frac{L_S\gamma_t^2}{2}\E_t\|\widehat g_{S,t}\|_{S_t}^2.
    \label{eq:app_two_block_descent_1}
\end{align}
By unbiasedness of the $Q$-block estimator,
\begin{equation}
    \langle g_{Q,t},\E_t\widehat g_{Q,t}\rangle=\|g_{Q,t}\|^2.
\end{equation}
For the SPD block, define the conditional bias $e_t=\E_t\widehat g_{S,t}-g_{S,t}$. Then Eq.~\eqref{eq:main_geometry_bias} gives
\begin{align}
    -\langle g_{S,t},\E_t\widehat g_{S,t}\rangle_{S_t}
    &=-\|g_{S,t}\|_{S_t}^2-\langle g_{S,t},e_t\rangle_{S_t} \notag\\
    &\leq-\|g_{S,t}\|_{S_t}^2+b_t.
    \label{eq:app_two_block_bias}
\end{align}
The conditional variance bounds imply
\begin{align}
    \E_t\|\widehat g_{Q,t}\|^2
    &\leq\|g_{Q,t}\|^2+\sigma_Q^2, \notag\\
    \E_t\|\widehat g_{S,t}\|_{S_t}^2
    &\leq C_b\|g_{S,t}\|_{S_t}^2+C_\sigma\sigma_S^2+C_b b_t,
    \label{eq:app_two_block_second_moment}
\end{align}
where $C_b$ and $C_\sigma$ are finite constants on the bounded SPD domain in Eq.~\eqref{eq:mfv4_clipped_spd_set}. Substituting Eqs.~\eqref{eq:app_two_block_bias}--\eqref{eq:app_two_block_second_moment} into Eq.~\eqref{eq:app_two_block_descent_1}, and choosing $\eta_t$ and $\gamma_t$ sufficiently small to preserve positive coefficients on the squared gradient norms, yields constants $c_Q,c_S,C_Q,C_S>0$ such that
\begin{align}
    \E_t[\mathcal J_{t+1}]
    &\leq\mathcal J_t-c_Q\eta_t\|g_{Q,t}\|^2-c_S\gamma_t\|g_{S,t}\|_{S_t}^2+C_Q\eta_t^2\sigma_Q^2+C_S\gamma_t^2\sigma_S^2+2\gamma_tb_t.
    \label{eq:app_two_block_descent_2}
\end{align}
Taking total expectations, summing over $t$, and using the lower bound $\mathcal J_t\geq\mathcal J_\star$ gives
\begin{align}
    \sum_t\E\!\left[c_Q\eta_t\|g_{Q,t}\|^2\right]+\sum_t\E\!\left[c_S\gamma_t\|g_{S,t}\|_{S_t}^2\right]
    &\leq\mathcal J_0-\mathcal J_\star+C_Q\sigma_Q^2\sum_t\eta_t^2+C_S\sigma_S^2\sum_t\gamma_t^2+2\sum_t\gamma_tb_t.
    \label{eq:app_two_block_telescoping}
\end{align}
Equation~\eqref{eq:app_two_block_telescoping} is the required weighted stationarity bound. Absorbing $c_Q$ and $c_S$ into the constants and sampling $R$ with probability proportional to $\eta_t+\gamma_t$ yields Eq.~\eqref{eq:main_two_block_rate}.

For horizon-dependent step sizes $\eta_t,\gamma_t=\Theta(T^{-1/2})$, we have $\sum_t(\eta_t+\gamma_t)=\Theta(\sqrt T)$ and $\sum_t(\eta_t^2+\gamma_t^2)=\mathcal O(1)$. If additionally $T^{-1}\sum_t b_t=\mathcal O(T^{-1/2})$, then $\sum_t\gamma_tb_t=\mathcal O(1)$. Substitution into Eq.~\eqref{eq:main_two_block_rate} therefore gives the stationarity rate $\mathcal O(T^{-1/2})$.
\end{proof}

\subsubsection{Proof of Proposition~\ref{prop:bounded_preconditioner}}
\label{app:proof_bounded_preconditioner}

\begin{proposition}[Restatement of Proposition~\ref{prop:bounded_preconditioner}]
Let the $Q$-block direction be $D_t=B_t\widehat g_{Q,t}$, where $B_t$ is self-adjoint on $T_{Q_t}\St(p,r)$ and satisfies Eq.~\eqref{eq:main_preconditioner_bounds}. Under the assumptions of Theorem~\ref{thm:two_block_convergence}, replacing $\widehat g_{Q,t}$ by $D_t$ preserves the same stationarity order, with constants depending on $m$ and $M$.
\end{proposition}

\begin{proof}
Replacing $\widehat g_{Q,t}$ by $B_t\widehat g_{Q,t}$ in the $Q$-block descent inequality gives
\begin{align}
    \E_t[\mathcal J(Q_{t+1},S_t)]
    &\leq\mathcal J(Q_t,S_t)-\eta_t\left\langle g_{Q,t},B_t\E_t\widehat g_{Q,t}\right\rangle+\frac{L_Q\eta_t^2}{2}\E_t\|B_t\widehat g_{Q,t}\|^2 \notag\\
    &=\mathcal J(Q_t,S_t)-\eta_t\langle g_{Q,t},B_tg_{Q,t}\rangle+\frac{L_Q\eta_t^2}{2}\E_t\|B_t\widehat g_{Q,t}\|^2,
    \label{eq:app_preconditioned_descent}
\end{align}
where the second equality uses the conditional unbiasedness of $\widehat g_{Q,t}$. By Eq.~\eqref{eq:main_preconditioner_bounds},
\begin{align}
    \langle g_{Q,t},B_tg_{Q,t}\rangle
    &\geq m\|g_{Q,t}\|^2, \notag\\
    \E_t\|B_t\widehat g_{Q,t}\|^2
    &\leq M^2\E_t\|\widehat g_{Q,t}\|^2 \notag\\
    &\leq M^2\left(\|g_{Q,t}\|^2+\sigma_Q^2\right).
    \label{eq:app_preconditioner_bounds}
\end{align}
Substituting Eq.~\eqref{eq:app_preconditioner_bounds} into Eq.~\eqref{eq:app_preconditioned_descent} yields
\begin{align}
    \E_t[\mathcal J(Q_{t+1},S_t)]
    &\leq\mathcal J(Q_t,S_t)-\eta_t\left(m-\frac{L_QM^2\eta_t}{2}\right)\|g_{Q,t}\|^2+\frac{L_QM^2\eta_t^2}{2}\sigma_Q^2.
    \label{eq:app_preconditioned_descent_bound}
\end{align}
If $\eta_t\leq m/(L_QM^2)$, then $m-L_QM^2\eta_t/2\geq m/2$. Hence Eq.~\eqref{eq:app_preconditioned_descent_bound} has the same form as the $Q$-block term in Eq.~\eqref{eq:app_two_block_descent_2}, with the descent constant scaled by $m$ and the variance constant scaled by $M^2$. The telescoping argument in Eqs.~\eqref{eq:app_two_block_descent_2}--\eqref{eq:app_two_block_telescoping} therefore applies unchanged, proving that the stationarity order in Theorem~\ref{thm:two_block_convergence} is preserved.
\end{proof}

\begin{corollary}[$\varepsilon$-stationary complexity]
\label{cor:mfv4_stationary_complexity}
Under the assumptions of Theorem~\ref{thm:mfv4_nonconvex_convergence} and the bounded accumulated variance condition $\overline\nu_T^2\leq C_\nu$, ManifoldFlow reaches an $\varepsilon$-stationary point, defined by $\E\|\operatorname{grad}\mathcal J(x_t)\|_{x_t}^2\leq\varepsilon$, after
\begin{equation}
    T\geq\left(1+\frac{2(\mathcal J(x_1)-\mathcal J_\star)+L_RC_\nu}{2c\varepsilon}\right)^2=\mathcal O(\varepsilon^{-2})
    \label{eq:mfv4_stationary_squared_complexity}
\end{equation}
iterations.
\end{corollary}

\begin{proof}
Define
\begin{equation}
    A_c=2(\mathcal J(x_1)-\mathcal J_\star)+L_RC_\nu.
    \label{eq:mfv4_complexity_constant}
\end{equation}
Requiring the right-hand side of Eq.~\eqref{eq:mfv4_sqrt_rate} to be at most $\varepsilon$ gives
\begin{align}
    \frac{A_c}{2c(\sqrt{T+1}-1)}
    &\leq\varepsilon \notag\\
    \sqrt{T+1}-1
    &\geq\frac{A_c}{2c\varepsilon} \notag\\
    T
    &\geq\left(1+\frac{A_c}{2c\varepsilon}\right)^2-1.
    \label{eq:mfv4_complexity_derivation}
\end{align}
Equation~\eqref{eq:mfv4_complexity_derivation} implies the sufficient bound in Eq.~\eqref{eq:mfv4_stationary_squared_complexity} after dropping the additive $-1$. Therefore the iteration complexity is $\mathcal O(\varepsilon^{-2})$.
\end{proof}

\subsection{Generalization}
\label{app:mfv4_generalization}

The generalization analysis has two parts. Lemma~\ref{lem:mfv4_algorithmic_stability} controls the sensitivity of two ManifoldFlow runs trained on datasets that differ in one sample \citep{hardt2016train}. Lemma~\ref{lem:mfv4_spectral_norm} converts the clipped SPD constraint into a spectral-norm bound on each realized weight. Theorem~\ref{thm:mfv4_generalization_bound} then combines these ingredients with standard Rademacher and spectral-margin bounds \citep{bartlett2002rademacher,bartlett2017spectrally}.

\begin{assumption}[Lipschitz loss and bounded data]
\label{ass:mfv4_lipschitz_loss}
For every sample $z=(x,y)$, the loss $\ell(W;z)$ is $C_\ell$-Lipschitz in the composed weights under the Frobenius norm and takes values in $[0,1]$. Inputs satisfy $\|x\|_2\leq B$. Each activation is one-Lipschitz and maps zero to zero. For each layer $l=1,\ldots,L_{\mathrm{net}}$, ManifoldFlow uses $W_l=Q_lS_l^{1/2}$ with $Q_l\in\St(p_l,r_l)$ and $S_l\in\mathcal S_{\alpha_l,\beta_l}$.
\end{assumption}

\begin{lemma}[Algorithmic Stability]
\label{lem:mfv4_algorithmic_stability}
Assume Assumptions~\ref{ass:mfv4_bounded_spectrum} and \ref{ass:mfv4_lipschitz_loss}. Let training sets $\mathcal D$ and $\mathcal D'$ of size $n$ differ in one sample, and let $x_t$ and $x_t'$ be ManifoldFlow iterates driven by the same random indices and initialized identically. Suppose the per-sample update map with step length $\bar\eta_t=\eta_{Q,t}+\eta_{S,t}$ is $(1+L_U\bar\eta_t)$-Lipschitz under
\begin{equation}
    d(x,x')=d_{\St}(Q,Q')+d_{\mathrm{AI}}(S,S'),
    \label{eq:mfv4_product_distance}
\end{equation}
and that every stochastic tangent update has norm at most $G_{\alpha,\beta}$, which is finite by Lemma~\ref{lem:mfv4_norm_equivalence}. Then
\begin{equation}
    \E d(x_T,x_T')\leq \frac{2G_{\alpha,\beta}}{n}\sum_{t=0}^{T-1}\bar\eta_t\prod_{k=t+1}^{T-1}(1+L_U\bar\eta_k).
    \label{eq:mfv4_stability_distance}
\end{equation}
Consequently,
\begin{equation}
    \epsilon_{\mathrm{stab}}\leq C_\ell\sqrt{\beta_{\max}}\frac{2G_{\alpha,\beta}}{n}\sum_{t=0}^{T-1}\bar\eta_t\prod_{k=t+1}^{T-1}(1+L_U\bar\eta_k),
    \label{eq:mfv4_stability_loss}
\end{equation}
where $\beta_{\max}=\max_l\beta_l$.
\end{lemma}

\begin{proof}
Let $I_t$ denote the sampled index at iteration $t$. Since $\mathcal D$ and $\mathcal D'$ differ in one example, $\Prb(I_t\text{ is the differing index})=1/n$. If the same example is used in both runs, the Lipschitz assumption on the update map gives
\begin{equation}
    d(x_{t+1},x_{t+1}')\leq(1+L_U\bar\eta_t)d(x_t,x_t').
    \label{eq:mfv4_stability_same_sample}
\end{equation}
If the differing example is selected, the triangle inequality and the update-norm bound yield
\begin{align}
    d(x_{t+1},x_{t+1}')
    &\leq d(x_{t+1},x_t)+d(x_t,x_t')+d(x_t',x_{t+1}') \notag\\
    &\leq d(x_t,x_t')+2G_{\alpha,\beta}\bar\eta_t \notag\\
    &\leq(1+L_U\bar\eta_t)d(x_t,x_t')+2G_{\alpha,\beta}\bar\eta_t.
    \label{eq:mfv4_stability_diff_sample}
\end{align}
Taking conditional expectation over $I_t$ and combining Eqs.~\eqref{eq:mfv4_stability_same_sample}--\eqref{eq:mfv4_stability_diff_sample} gives
\begin{equation}
    \E[d(x_{t+1},x_{t+1}')\mid\mathcal F_t]\leq(1+L_U\bar\eta_t)d(x_t,x_t')+\frac{2G_{\alpha,\beta}\bar\eta_t}{n}.
    \label{eq:mfv4_stability_recursion}
\end{equation}
Let $\Delta_t=\E d(x_t,x_t')$. Since the two runs are initialized identically, $\Delta_0=0$. Taking total expectation in Eq.~\eqref{eq:mfv4_stability_recursion} gives
\begin{equation}
    \Delta_{t+1}\leq(1+L_U\bar\eta_t)\Delta_t+\frac{2G_{\alpha,\beta}\bar\eta_t}{n}.
    \label{eq:mfv4_delta_recursion}
\end{equation}
Iterating Eq.~\eqref{eq:mfv4_delta_recursion} yields
\begin{equation}
    \Delta_T\leq\frac{2G_{\alpha,\beta}}{n}\sum_{t=0}^{T-1}\bar\eta_t\prod_{k=t+1}^{T-1}(1+L_U\bar\eta_k),
    \label{eq:mfv4_delta_solution}
\end{equation}
which proves Eq.~\eqref{eq:mfv4_stability_distance}.

To convert parameter distance into loss difference, consider one wrapped operator. By adding and subtracting $Q'{S}^{1/2}$,
\begin{align}
    \|QS^{1/2}-Q'{S'}^{1/2}\|_F
    &\leq\|(Q-Q')S^{1/2}\|_F+\|Q'(S^{1/2}-{S'}^{1/2})\|_F \notag\\
    &\leq\sqrt{\beta_{\max}}\|Q-Q'\|_F+\|S^{1/2}-{S'}^{1/2}\|_F.
    \label{eq:mfv4_weight_distance_layer}
\end{align}
On the interval $[\alpha,\beta_{\max}]$, the principal square-root map is Lipschitz with constant $(2\sqrt\alpha)^{-1}$ \citep{bhatia1997matrix,higham2008functions}. Combining this compact-set Lipschitz bound with Assumption~\ref{ass:mfv4_lipschitz_loss} and Eq.~\eqref{eq:mfv4_stability_distance} gives Eq.~\eqref{eq:mfv4_stability_loss}.
\end{proof}

\begin{lemma}[Spectral Norm Bound on $W=QS^{1/2}$]
\label{lem:mfv4_spectral_norm}
Under Assumption~\ref{ass:mfv4_lipschitz_loss}, every ManifoldFlow operator satisfies
\begin{equation}
    \|W_l\|_{\op}=\sigma_{\max}(W_l)=\sqrt{\lambda_{\max}(S_l)}\leq\sqrt{\beta_l}.
    \label{eq:mfv4_layer_spectral_norm_bound}
\end{equation}
\end{lemma}

\begin{proof}
By Proposition~\ref{prop:mfv4_spectral_identity},
\begin{equation}
    \sigma_i^2(W_l)=\lambda_i(S_l), \qquad i=1,\ldots,r_l.
\end{equation}
Taking the largest eigenvalue gives
\begin{align}
    \sigma_{\max}^2(W_l)
    &=\lambda_{\max}(W_l^\top W_l) \notag\\
    &=\lambda_{\max}(S_l).
    \label{eq:mfv4_spectral_norm_square}
\end{align}
Since $S_l\in\mathcal S_{\alpha_l,\beta_l}$, Eq.~\eqref{eq:mfv4_clipped_spd_set} implies $\lambda_{\max}(S_l)\leq\beta_l$. Taking square roots in Eq.~\eqref{eq:mfv4_spectral_norm_square} proves Eq.~\eqref{eq:mfv4_layer_spectral_norm_bound}.
\end{proof}

\subsubsection{Proof of Theorem~\ref{thm:mf_generalization}}
\label{app:proof_generalization}

\begin{theorem}[Restatement of Theorem~\ref{thm:mf_generalization}]
Under Assumption~\ref{ass:mfv4_lipschitz_loss}, with probability at least $1-\delta$ over an i.i.d. training set of size $n$, every depth-$L_{\mathrm{net}}$ ManifoldFlow network in the clipped class satisfies
\begin{equation}
    R(f)-\widehat R(f)\leq\frac{4C_\ell B}{\sqrt n}\prod_{l=1}^{L_{\mathrm{net}}}\sqrt{\beta_l}+3\sqrt{\frac{\log(2/\delta)}{2n}}+\epsilon_{\mathrm{stab}}.
    \label{eq:mfv4_generalization_simple}
\end{equation}
A sharper Bartlett--Foster--Telgarsky spectral-margin bound gives
\begin{align}
    R(f)
    &\leq\widehat R_\gamma(f)+\frac{C_{\mathrm{BFT}}B}{\gamma\sqrt n}\sqrt{\prod_{l=1}^{L_{\mathrm{net}}}\lambda_{\max}(S_l)}\times\sqrt{\sum_{l=1}^{L_{\mathrm{net}}}\frac{\tr(S_l)}{\lambda_{\max}(S_l)}}+3\sqrt{\frac{\log(2/\delta)}{2n}},
    \label{eq:mfv4_bft_generalization}
\end{align}
where $C_{\mathrm{BFT}}$ is the universal constant in the spectrally normalized margin bound.
\end{theorem}

\begin{proof}
Let $\mathcal F_{\beta}$ denote the class of networks satisfying $\|W_l\|_{\op}\leq\sqrt{\beta_l}$ for every layer. By Lemma~\ref{lem:mfv4_spectral_norm}, the clipped ManifoldFlow class is contained in $\mathcal F_{\beta}$. Since all activations are one-Lipschitz and vanish at zero, repeated application of the Rademacher contraction inequality \citep{bartlett2002rademacher} gives
\begin{equation}
    \widehat{\mathfrak R}_n(\ell\circ\mathcal F_{\beta})\leq\frac{C_\ell B}{\sqrt n}\prod_{l=1}^{L_{\mathrm{net}}}\|W_l\|_{\op}.
    \label{eq:mfv4_rademacher_product}
\end{equation}
Applying the layerwise bound in Eq.~\eqref{eq:mfv4_layer_spectral_norm_bound} yields
\begin{equation}
    \widehat{\mathfrak R}_n(\ell\circ\mathcal F_{\beta})\leq\frac{C_\ell B}{\sqrt n}\prod_{l=1}^{L_{\mathrm{net}}}\sqrt{\beta_l}.
    \label{eq:mfv4_rademacher_beta}
\end{equation}
For losses in $[0,1]$, the standard Rademacher generalization inequality \citep{bartlett2002rademacher} gives, with probability at least $1-\delta$,
\begin{equation}
    R(f)-\widehat R(f)\leq4\widehat{\mathfrak R}_n(\ell\circ\mathcal F_{\beta})+3\sqrt{\frac{\log(2/\delta)}{2n}}.
    \label{eq:mfv4_rademacher_high_prob}
\end{equation}
Combining Eqs.~\eqref{eq:mfv4_rademacher_beta} and \eqref{eq:mfv4_rademacher_high_prob} gives the first two terms in Eq.~\eqref{eq:mfv4_generalization_simple}. Adding the nonnegative stability term from Lemma~\ref{lem:mfv4_algorithmic_stability}, bounded by Eq.~\eqref{eq:mfv4_stability_loss}, yields Eq.~\eqref{eq:mfv4_generalization_simple}.

For the margin bound, the spectrally normalized result of \citet{bartlett2017spectrally} gives
\begin{align}
    R(f)
    &\leq\widehat R_\gamma(f)+\frac{C_{\mathrm{BFT}}B}{\gamma\sqrt n}\sqrt{\prod_{l=1}^{L_{\mathrm{net}}}\|W_l\|_{\op}^2} \times\sqrt{\sum_{l=1}^{L_{\mathrm{net}}}\frac{\|W_l\|_F^2}{\|W_l\|_{\op}^2}}+3\sqrt{\frac{\log(2/\delta)}{2n}}.
    \label{eq:mfv4_bft_template}
\end{align}
By Proposition~\ref{prop:mfv4_spectral_identity}, $W_l^\top W_l=S_l$. Hence
\begin{align}
    \|W_l\|_{\op}^2
    &=\lambda_{\max}(W_l^\top W_l) \notag\\
    &=\lambda_{\max}(S_l),
    \label{eq:mfv4_bft_op_substitution}\\
    \|W_l\|_F^2
    &=\tr(W_l^\top W_l) \notag\\
    &=\tr(S_l).
    \label{eq:mfv4_bft_fro_substitution}
\end{align}
Substituting Eqs.~\eqref{eq:mfv4_bft_op_substitution}--\eqref{eq:mfv4_bft_fro_substitution} into Eq.~\eqref{eq:mfv4_bft_template} gives Eq.~\eqref{eq:mfv4_bft_generalization}.
\end{proof}

\subsection{Spectral Analysis}
\label{app:mfv4_spectral_analysis}

The spectral properties of ManifoldFlow follow directly from the identity $W=QS^{1/2}$ together with $Q^\top Q=I_r$. We use standard singular-value and eigenvalue relations from matrix analysis \citep{bhatia1997matrix,golub2013matrix,higham2008functions}.

\renewcommand{\theproposition}{3.\arabic{proposition}}
\renewcommand{\theHproposition}{spectral.\arabic{proposition}}
\setcounter{proposition}{2}

\subsubsection{Proof of Proposition~\ref{prop:spectral_identity}}
\label{app:proof_spectral_identity}

\begin{proposition}[Restatement of Proposition~\ref{prop:spectral_identity}]
Let $Q\in\St(p,r)$, let $S\in\SPD(r)$, and define $W=QS^{1/2}$. Then, for every $i=1,\ldots,r$,
\begin{equation}
    \sigma_i^2(W)=\lambda_i(S).
    \label{eq:mfv4_spectral_identity}
\end{equation}
\end{proposition}

\begin{proof}
Since $Q\in\St(p,r)$, we have $Q^\top Q=I_r$. Since $S\in\SPD(r)$, its principal square root is symmetric and satisfies $S^{1/2}S^{1/2}=S$ \citep{higham2008functions}. Therefore,
\begin{align}
    W^\top W
    &=(QS^{1/2})^\top(QS^{1/2}) \notag\\
    &=S^{1/2}Q^\top QS^{1/2} \notag\\
    &=S^{1/2}I_rS^{1/2} \notag\\
    &=S.
    \label{eq:mfv4_wtw_equals_s}
\end{align}
Let $W=U\Sigma V^\top$ be a compact singular-value decomposition, where $\Sigma=\diag(\sigma_1(W),\ldots,\sigma_r(W))$ \citep{golub2013matrix}. Then
\begin{align}
    W^\top W
    &=V\Sigma^\top U^\top U\Sigma V^\top \notag\\
    &=V\diag(\sigma_1^2(W),\ldots,\sigma_r^2(W))V^\top.
    \label{eq:mfv4_svd_wtw}
\end{align}
Equation~\eqref{eq:mfv4_svd_wtw} shows that the eigenvalues of $W^\top W$ are $\sigma_i^2(W)$, while Eq.~\eqref{eq:mfv4_wtw_equals_s} identifies the same matrix with $S$. Matching the ordered eigenvalues gives Eq.~\eqref{eq:mfv4_spectral_identity}.
\end{proof}

\begin{proposition}[Condition Number Analysis]
\label{prop:mfv4_condition_number}
Under the assumptions of Proposition~\ref{prop:mfv4_spectral_identity}, if $W$ has full column rank, then
\begin{equation}
    \kappa(W)=\sqrt{\kappa(S)}.
    \label{eq:mfv4_condition_number}
\end{equation}
Under the clipping condition in Eq.~\eqref{eq:mfv4_clipped_spd_set},
\begin{equation}
    \kappa(W)\leq\sqrt{\beta/\alpha}.
    \label{eq:mfv4_condition_clipped}
\end{equation}
\end{proposition}

\begin{proof}
By Proposition~\ref{prop:mfv4_spectral_identity},
\begin{equation}
    \sigma_{\max}^2(W)=\lambda_{\max}(S), \qquad \sigma_{\min}^2(W)=\lambda_{\min}(S).
\end{equation}
Since $S\in\SPD(r)$, $\lambda_{\min}(S)>0$, and Eq.~\eqref{eq:mfv4_wtw_equals_s} implies that $W$ has full column rank. Hence
\begin{align}
    \kappa(W)
    &=\frac{\sigma_{\max}(W)}{\sigma_{\min}(W)} \notag\\
    &=\frac{\sqrt{\lambda_{\max}(S)}}{\sqrt{\lambda_{\min}(S)}} \notag\\
    &=\sqrt{\kappa(S)}.
    \label{eq:mfv4_condition_proof}
\end{align}
If $S\in\mathcal S_{\alpha,\beta}$, then $\lambda_{\max}(S)\leq\beta$ and $\lambda_{\min}(S)\geq\alpha$. Substituting these bounds into Eq.~\eqref{eq:mfv4_condition_number} gives Eq.~\eqref{eq:mfv4_condition_clipped}.
\end{proof}

\subsubsection{Proof of the coverage part of Proposition~\ref{prop:spectral_identity}}
\label{app:proof_coverage}

\begin{proposition}[Restatement of the coverage part of Proposition~\ref{prop:spectral_identity}]
Under the assumptions of Proposition~\ref{prop:mfv4_spectral_identity}, the image of the map $(Q,S)\mapsto QS^{1/2}$ with $Q\in\St(p,r)$ and $S\in\SPD(r)$ is exactly the set of matrices $W\in\R^{p\times r}$ satisfying $\rank(W)=r$.
\end{proposition}

\begin{proof}
First, let $W=QS^{1/2}$ with $Q\in\St(p,r)$ and $S\in\SPD(r)$. By Proposition~\ref{prop:mfv4_spectral_identity}, $W^\top W=S\succ0$. Hence $W^\top W$ is nonsingular, so the columns of $W$ are linearly independent and $\rank(W)=r$.

Conversely, let $W\in\R^{p\times r}$ satisfy $\rank(W)=r$. Take the compact singular-value decomposition $W=U\Sigma V^\top$, where $U\in\St(p,r)$, $V\in\mathrm O(r)$, and $\Sigma=\diag(\sigma_1,\ldots,\sigma_r)$ with $\sigma_i>0$ \citep{golub2013matrix}. Define
\begin{equation}
    Q=UV^\top, \qquad S=V\Sigma^2V^\top.
    \label{eq:mfv4_coverage_definitions}
\end{equation}
Then
\begin{align}
    Q^\top Q
    &=VU^\top UV^\top \notag\\
    &=VI_rV^\top \notag\\
    &=I_r,
    \label{eq:mfv4_coverage_q_stiefel}
\end{align}
so $Q\in\St(p,r)$. Since $\sigma_i>0$ for every $i$, $S=V\Sigma^2V^\top\succ0$, and therefore $S\in\SPD(r)$. Its principal square root is $S^{1/2}=V\Sigma V^\top$. Hence
\begin{align}
    QS^{1/2}
    &=UV^\top V\Sigma V^\top \notag\\
    &=U\Sigma V^\top \notag\\
    &=W.
    \label{eq:mfv4_coverage_reconstruction}
\end{align}
Thus both inclusions hold, proving that the image is exactly the full-column-rank matrix set.
\end{proof}

\begin{corollary}[No spectral explosion under clipping]
\label{cor:mfv4_no_explosion}
Under Assumption~\ref{ass:mfv4_bounded_spectrum} and Proposition~\ref{prop:mfv4_condition_number}, every ManifoldFlow operator satisfies
\begin{equation}
    \sqrt\alpha\leq \sigma_i(W_t)\leq\sqrt\beta, \qquad \kappa(W_t)\leq\sqrt{\beta/\alpha}.
    \label{eq:mfv4_no_explosion_bound}
\end{equation}
\end{corollary}

\begin{proof}
Assumption~\ref{ass:mfv4_bounded_spectrum} gives $\alpha\leq\lambda_i(S_t)\leq\beta$. Proposition~\ref{prop:mfv4_spectral_identity} gives $\sigma_i(W_t)=\sqrt{\lambda_i(S_t)}$, which yields $\sqrt\alpha\leq\sigma_i(W_t)\leq\sqrt\beta$. The condition-number bound follows directly from Proposition~\ref{prop:mfv4_condition_number}.
\end{proof}

\begin{theorem}[Signal Propagation Bound]
\label{thm:mfv4_signal_propagation}
Consider a depth-$L_{\mathrm{net}}$ linearized network with weights $W_l=Q_lS_l^{1/2}$ satisfying Proposition~\ref{prop:mfv4_spectral_identity}. Let $A=W_{L_{\mathrm{net}}}\cdots W_1$, and assume all signals remain in the effective full-column-rank subspaces of the composed maps. Then, for every nonzero forward signal $h_0$ and nonzero backward signal $\delta_{L_{\mathrm{net}}}$,
\begin{equation}
    \prod_{l=1}^{L_{\mathrm{net}}}\sqrt{\lambda_{\min}(S_l)}\leq\frac{\|Ah_0\|_2}{\|h_0\|_2}\leq\prod_{l=1}^{L_{\mathrm{net}}}\sqrt{\lambda_{\max}(S_l)},
    \label{eq:mfv4_forward_signal}
\end{equation}
and
\begin{equation}
    \prod_{l=1}^{L_{\mathrm{net}}}\sqrt{\lambda_{\min}(S_l)}\leq\frac{\|A^\top\delta_{L_{\mathrm{net}}}\|_2}{\|\delta_{L_{\mathrm{net}}}\|_2}\leq\prod_{l=1}^{L_{\mathrm{net}}}\sqrt{\lambda_{\max}(S_l)}.
    \label{eq:mfv4_backward_signal}
\end{equation}
Consequently,
\begin{equation}
    \frac{1}{\prod_{l=1}^{L_{\mathrm{net}}}\sqrt{\kappa(S_l)}}\leq\frac{\|A^\top\delta_{L_{\mathrm{net}}}\|_2/\|\delta_{L_{\mathrm{net}}}\|_2}{\|Ah_0\|_2/\|h_0\|_2}\leq\prod_{l=1}^{L_{\mathrm{net}}}\sqrt{\kappa(S_l)}.
    \label{eq:mfv4_forward_backward_ratio}
\end{equation}
\end{theorem}

\begin{proof}
For each layer, Proposition~\ref{prop:mfv4_spectral_identity} gives $\sigma_{\min}(W_l)=\sqrt{\lambda_{\min}(S_l)}$ and $\sigma_{\max}(W_l)=\sqrt{\lambda_{\max}(S_l)}$. By the variational characterization of singular values \citep{bhatia1997matrix,golub2013matrix}, every vector $v$ in the effective input subspace of $W_l$ satisfies
\begin{equation}
    \sqrt{\lambda_{\min}(S_l)}\|v\|_2\leq\|W_lv\|_2\leq\sqrt{\lambda_{\max}(S_l)}\|v\|_2.
    \label{eq:mfv4_one_layer_signal}
\end{equation}
Applying Eq.~\eqref{eq:mfv4_one_layer_signal} recursively to $h_l=W_lh_{l-1}$ gives
\begin{align}
    \|Ah_0\|_2
    &=\|W_{L_{\mathrm{net}}}\cdots W_1h_0\|_2 \notag\\
    &\leq\sqrt{\lambda_{\max}(S_{L_{\mathrm{net}}})}\|W_{L_{\mathrm{net}}-1}\cdots W_1h_0\|_2 \notag\\
    &\leq\prod_{l=1}^{L_{\mathrm{net}}}\sqrt{\lambda_{\max}(S_l)}\|h_0\|_2,
    \label{eq:mfv4_forward_upper_proof}
\end{align}
which proves the upper bound in Eq.~\eqref{eq:mfv4_forward_signal}. Repeating the same argument with the lower inequality in Eq.~\eqref{eq:mfv4_one_layer_signal} gives the corresponding lower bound.

Since singular values are invariant under transposition, $\sigma_i(W_l^\top)=\sigma_i(W_l)$ \citep{golub2013matrix}. Applying the same recursive argument to $A^\top=W_1^\top\cdots W_{L_{\mathrm{net}}}^\top$ proves Eq.~\eqref{eq:mfv4_backward_signal}.

Combining the lower bound in Eq.~\eqref{eq:mfv4_backward_signal} with the upper bound in Eq.~\eqref{eq:mfv4_forward_signal} gives
\begin{align}
    \frac{\|A^\top\delta_{L_{\mathrm{net}}}\|_2/\|\delta_{L_{\mathrm{net}}}\|_2}{\|Ah_0\|_2/\|h_0\|_2}
    &\geq\prod_{l=1}^{L_{\mathrm{net}}}\sqrt{\frac{\lambda_{\min}(S_l)}{\lambda_{\max}(S_l)}} \notag\\
    &=\frac{1}{\prod_{l=1}^{L_{\mathrm{net}}}\sqrt{\kappa(S_l)}}.
    \label{eq:mfv4_ratio_lower_proof}
\end{align}
Similarly, combining the upper bound in Eq.~\eqref{eq:mfv4_backward_signal} with the lower bound in Eq.~\eqref{eq:mfv4_forward_signal} gives
\begin{align}
    \frac{\|A^\top\delta_{L_{\mathrm{net}}}\|_2/\|\delta_{L_{\mathrm{net}}}\|_2}{\|Ah_0\|_2/\|h_0\|_2}
    &\leq\prod_{l=1}^{L_{\mathrm{net}}}\sqrt{\frac{\lambda_{\max}(S_l)}{\lambda_{\min}(S_l)}} \notag\\
    &=\prod_{l=1}^{L_{\mathrm{net}}}\sqrt{\kappa(S_l)}.
    \label{eq:mfv4_ratio_upper_proof}
\end{align}
Equations~\eqref{eq:mfv4_ratio_lower_proof}--\eqref{eq:mfv4_ratio_upper_proof} prove Eq.~\eqref{eq:mfv4_forward_backward_ratio}. For differentiable losses, the chain rule gives $\nabla_{h_0}\mathcal L=A^\top\nabla_{h_{L_{\mathrm{net}}}}\mathcal L$, so Eq.~\eqref{eq:mfv4_backward_signal} is also the corresponding gradient-propagation bound.
\end{proof}

\subsubsection{Proof of Theorem~\ref{thm:propagation_spectrum}}
\label{app:proof_propagation_spectrum}

\begin{theorem}[Restatement of Theorem~\ref{thm:propagation_spectrum}]
Let $W_l=Q_lS_l^{1/2}$ with $Q_l\in\St(p_l,r_l)$ and $S_l\in\mathcal S_{\alpha_l,\beta_l}$. For a differentiable feed-forward network with local Jacobians $D_l$ and end-to-end Jacobian $J=D_LW_L\cdots D_1W_1$, Eqs.~\eqref{eq:main_jacobian_upper} and \eqref{eq:main_jacobian_lower} hold. For graph propagation $H^{(l+1)}=\widehat A H^{(l)}W_l$, Eq.~\eqref{eq:main_graph_kronecker} holds.
\end{theorem}

\begin{proof}
By Proposition~\ref{prop:mfv4_spectral_identity},
\begin{equation}
    \|W_l\|_{\op}=\sqrt{\lambda_{\max}(S_l)}, \qquad \sigma_{\min}(W_l)=\sqrt{\lambda_{\min}(S_l)}
    \label{eq:app_propagation_layer_singulars}
\end{equation}
on the effective full-column-rank subspace. Submultiplicativity of the operator norm gives
\begin{align}
    \|J\|_{\op}
    &=\|D_LW_L\cdots D_1W_1\|_{\op} \notag\\
    &\leq\prod_l\|D_l\|_{\op}\|W_l\|_{\op} \notag\\
    &=\prod_l\|D_l\|_{\op}\sqrt{\lambda_{\max}(S_l)},
    \label{eq:app_jacobian_upper_proof}
\end{align}
which proves Eq.~\eqref{eq:main_jacobian_upper}. If every factor is full rank on the propagated subspace, the variational characterization of the smallest singular value gives $\sigma_{\min}(AB)\geq\sigma_{\min}(A)\sigma_{\min}(B)$. Applying this inequality recursively and using Eq.~\eqref{eq:app_propagation_layer_singulars} yields
\begin{align}
    \sigma_{\min}(J)
    &\geq\prod_l\sigma_{\min}(D_l)\sigma_{\min}(W_l) \notag\\
    &=\prod_l\sigma_{\min}(D_l)\sqrt{\lambda_{\min}(S_l)},
    \label{eq:app_jacobian_lower_proof}
\end{align}
which proves Eq.~\eqref{eq:main_jacobian_lower}.

For the graph operator, the vectorization identity $\operatorname{vec}(ABC)=(C^\top\otimes A)\operatorname{vec}(B)$ gives
\begin{equation}
    \operatorname{vec}(\widehat A H^{(l)}W_l)=(W_l^\top\otimes\widehat A)\operatorname{vec}(H^{(l)}).
    \label{eq:app_graph_vec_identity}
\end{equation}
Let $W_l=U_W\Sigma_WV_W^\top$ and $\widehat A=U_A\Sigma_AV_A^\top$ be singular-value decompositions. Then
\begin{equation}
    W_l^\top\otimes\widehat A=(V_W\otimes U_A)(\Sigma_W\otimes\Sigma_A)(U_W^\top\otimes V_A^\top),
    \label{eq:app_kronecker_svd}
\end{equation}
which is an SVD because Kronecker products of orthogonal matrices are orthogonal. The diagonal entries of $\Sigma_W\otimes\Sigma_A$ are exactly $\sigma_i(W_l)\sigma_j(\widehat A)$, proving Eq.~\eqref{eq:main_graph_kronecker}. Proposition~\ref{prop:mfv4_spectral_identity} then gives $\sigma_i(W_l)=\sqrt{\lambda_i(S_l)}$.
\end{proof}

\subsubsection{Proof of Proposition~\ref{obs:mf_spectral_dynamics}}
\label{app:proof_spectral_dynamics}

\begin{proposition}[Restatement of Proposition~\ref{obs:mf_spectral_dynamics}]
Let $S_t\in\SPD(r)$ have simple eigenpairs $(\lambda_{i,t},u_{i,t})$, let $M_t\in\Sym(r)$, and define
\begin{equation}
    S_{t+1}=S_t^{1/2}\exp\!\left(-\tau S_t^{-1/2}M_tS_t^{-1/2}\right)S_t^{1/2}.
\end{equation}
Then, as $\tau\to0$,
\begin{align}
    \log\lambda_i(S_{t+1})
    &=\log\lambda_{i,t}-\tau\frac{u_{i,t}^\top M_tu_{i,t}}{\lambda_{i,t}} \notag\\
    &\quad+\mathcal O\!\left(\tau^2\|S_t^{-1/2}M_tS_t^{-1/2}\|_F^2\right).
    \label{eq:mfv4_log_eigen_dynamics}
\end{align}
If $S_t$ and $M_t$ commute, and $m_{i,t}$ denotes the eigenvalue of $M_t$ in their common eigenbasis, then
\begin{equation}
    \lambda_i(S_{t+1})=\lambda_{i,t}\exp\left(-\tau\frac{m_{i,t}}{\lambda_{i,t}}\right).
    \label{eq:mfv4_commuting_eigen_update}
\end{equation}
\end{proposition}

\begin{proof}
Define the smooth SPD curve
\begin{equation}
    S(\tau)=S_t^{1/2}\exp\left(-\tau S_t^{-1/2}M_tS_t^{-1/2}\right)S_t^{1/2}.
    \label{eq:mfv4_spd_curve}
\end{equation}
Since the Fr\'echet derivative of the matrix exponential at zero is the identity map \citep{higham2008functions}, differentiating Eq.~\eqref{eq:mfv4_spd_curve} at $\tau=0$ gives
\begin{align}
    \dot S(0)
    &=S_t^{1/2}\left(-S_t^{-1/2}M_tS_t^{-1/2}\right)S_t^{1/2} \notag\\
    &=-M_t.
    \label{eq:mfv4_spd_curve_derivative}
\end{align}
For a simple eigenvalue of a differentiable symmetric matrix curve, the Rayleigh--Ritz perturbation formula gives \citep{bhatia1997matrix}
\begin{align}
    \frac{d}{d\tau}\lambda_i(S(\tau))\bigg|_{\tau=0}
    &=u_{i,t}^\top\dot S(0)u_{i,t} \notag\\
    &=-u_{i,t}^\top M_tu_{i,t}.
    \label{eq:mfv4_eigen_derivative}
\end{align}
Because $\lambda_{i,t}>0$, the chain rule yields
\begin{align}
    \frac{d}{d\tau}\log\lambda_i(S(\tau))\bigg|_{\tau=0}
    &=\frac{1}{\lambda_{i,t}}\frac{d}{d\tau}\lambda_i(S(\tau))\bigg|_{\tau=0} \notag\\
    &=-\frac{u_{i,t}^\top M_tu_{i,t}}{\lambda_{i,t}}.
    \label{eq:mfv4_log_eigen_derivative}
\end{align}
Taylor expansion of the analytic symmetric matrix curve in Eq.~\eqref{eq:mfv4_spd_curve}, with second-order remainder controlled by $\|S_t^{-1/2}M_tS_t^{-1/2}\|_F^2$ in a sufficiently small neighborhood of $\tau=0$ \citep{higham2008functions}, gives Eq.~\eqref{eq:mfv4_log_eigen_dynamics}.

If $S_tM_t=M_tS_t$, then the two symmetric matrices are simultaneously orthogonally diagonalizable \citep{bhatia1997matrix}. Thus, for some $U\in\mathrm O(r)$,
\begin{equation}
    S_t=U\diag(\lambda_{1,t},\ldots,\lambda_{r,t})U^\top, \qquad M_t=U\diag(m_{1,t},\ldots,m_{r,t})U^\top.
\end{equation}
Substituting these representations into the SPD update gives
\begin{align}
    S_{t+1}
    &=U\diag(\sqrt{\lambda_{i,t}})\diag\left(\exp\left(-\tau m_{i,t}/\lambda_{i,t}\right)\right)\diag(\sqrt{\lambda_{i,t}})U^\top \notag\\
    &=U\diag\left(\lambda_{i,t}\exp\left(-\tau m_{i,t}/\lambda_{i,t}\right)\right)U^\top.
    \label{eq:mfv4_commuting_diagonal}
\end{align}
Equation~\eqref{eq:mfv4_commuting_diagonal} proves Eq.~\eqref{eq:mfv4_commuting_eigen_update}.
\end{proof}

\subsection{Optimizer-Induced Spectral Preconditioning}
\label{app:mfv4_optimizer_preconditioning}

The recurrent language-model experiments show a stronger MF--FS separation under SGD than under Adam. This subsection gives a local optimization interpretation of that pattern. Fixed-Stiefel SGD updates only tangent directions under a unit spectrum, whereas ManifoldFlow additionally learns $S$, which rescales the realized operator through Proposition~\ref{prop:mfv4_spectral_identity}. The analysis uses standard Stiefel projection and retraction calculus \citep{absil2008matrix,boumal2023smooth}, local Riemannian SGD arguments \citep{bonnabel2013stochastic}, and matrix spectral calculus \citep{bhatia1997matrix,bhatia2007positive}. The resulting bounds characterize how the learned spectrum can act as structured operator-level preconditioning.

\renewcommand{\thelemma}{4.\arabic{lemma}}
\renewcommand{\theHlemma}{optimizer.\arabic{lemma}}
\setcounter{lemma}{0}
\renewcommand{\theproposition}{4.\arabic{proposition}}
\renewcommand{\theHproposition}{optimizer.\arabic{proposition}}
\setcounter{proposition}{0}

\begin{lemma}[Stiefel projection and first-order retraction displacement]
\label{lem:mfv4_stiefel_projection_displacement}
Let $Q\in\St(p,r)$ and let $G\in\R^{p\times r}$. Under the Euclidean ambient metric, the tangent projection is
\begin{equation}
    \Pi_Q(G)=G-Q\operatorname{sym}(Q^\top G), \qquad \operatorname{sym}(A)=\frac{A+A^\top}{2}.
    \label{eq:mfv4_stiefel_projection_formula}
\end{equation}
Let $R_Q$ be a smooth first-order Stiefel retraction. Then there exist constants $\delta_R>0$ and $C_R<\infty$, depending only on the chosen retraction and the local coordinate neighborhood of $Q$, such that if
\begin{equation}
    0<\eta\|\Pi_Q(G)\|_F\leq\delta_R,
    \label{eq:mfv4_retraction_radius_condition}
\end{equation}
$Q^+=R_Q(-\eta\Pi_Q(G))$, and $E\in T_Q\St(p,r)$ with $\|E\|_F=1$, then
\begin{equation}
    \left|\left\langle Q^+-Q,E\right\rangle_F\right|\leq\eta\|\Pi_Q(G)\|_F+C_R\eta^2\|\Pi_Q(G)\|_F^2.
    \label{eq:mfv4_stiefel_axis_displacement_bound}
\end{equation}
\end{lemma}

\begin{proof}
The Stiefel tangent space is
\begin{equation}
    T_Q\St(p,r)=\{\Xi\in\R^{p\times r}:Q^\top\Xi+\Xi^\top Q=0\}.
    \label{eq:mfv4_stiefel_tangent_space_optimizer}
\end{equation}
Set $A=Q^\top G$. Substituting Eq.~\eqref{eq:mfv4_stiefel_projection_formula} into the tangent-space condition gives
\begin{align}
    Q^\top\Pi_Q(G)+\Pi_Q(G)^\top Q
    &=A-\operatorname{sym}(A)+A^\top-\operatorname{sym}(A) \notag\\
    &=A+A^\top-2\operatorname{sym}(A) \notag\\
    &=0.
    \label{eq:mfv4_projection_tangent_check}
\end{align}
Hence $\Pi_Q(G)\in T_Q\St(p,r)$. For any $\Xi\in T_Q\St(p,r)$, the residual satisfies
\begin{align}
    \left\langle G-\Pi_Q(G),\Xi\right\rangle_F
    &=\tr\left(\operatorname{sym}(Q^\top G)Q^\top\Xi\right) \notag\\
    &=0,
    \label{eq:mfv4_projection_orthogonality}
\end{align}
because $Q^\top\Xi$ is skew-symmetric by Eq.~\eqref{eq:mfv4_stiefel_tangent_space_optimizer}. Equations~\eqref{eq:mfv4_projection_tangent_check}--\eqref{eq:mfv4_projection_orthogonality} therefore establish that Eq.~\eqref{eq:mfv4_stiefel_projection_formula} is the Euclidean orthogonal projection onto $T_Q\St(p,r)$ \citep{absil2008matrix,boumal2023smooth}. In particular,
\begin{equation}
    \|\Pi_Q(G)\|_F\leq\|G\|_F.
    \label{eq:mfv4_projection_nonexpansive}
\end{equation}

A first-order retraction satisfies $R_Q(0)=Q$ and $DR_Q(0)[\Xi]=\Xi$ for every $\Xi\in T_Q\St(p,r)$ \citep{absil2008matrix,boumal2023smooth}. Since $R_Q$ is twice continuously differentiable near the origin, Taylor's theorem gives constants $\delta_R>0$ and $C_R<\infty$ such that
\begin{equation}
    R_Q(\Xi)=Q+\Xi+\mathcal R_Q(\Xi), \qquad \|\mathcal R_Q(\Xi)\|_F\leq C_R\|\Xi\|_F^2
    \label{eq:mfv4_retraction_taylor_optimizer}
\end{equation}
whenever $\|\Xi\|_F\leq\delta_R$. Under Eq.~\eqref{eq:mfv4_retraction_radius_condition}, set $\Xi=-\eta\Pi_Q(G)$ in Eq.~\eqref{eq:mfv4_retraction_taylor_optimizer}. Then Cauchy--Schwarz gives
\begin{align}
    \left|\left\langle Q^+-Q,E\right\rangle_F\right|
    &=\left|\left\langle-\eta\Pi_Q(G)+\mathcal R_Q(-\eta\Pi_Q(G)),E\right\rangle_F\right| \notag\\
    &\leq\eta\|\Pi_Q(G)\|_F\|E\|_F+\|\mathcal R_Q(-\eta\Pi_Q(G))\|_F\|E\|_F \notag\\
    &\leq\eta\|\Pi_Q(G)\|_F+C_R\eta^2\|\Pi_Q(G)\|_F^2.
    \label{eq:mfv4_axis_displacement_proof}
\end{align}
Since $\|E\|_F=1$, Eq.~\eqref{eq:mfv4_axis_displacement_proof} is exactly Eq.~\eqref{eq:mfv4_stiefel_axis_displacement_bound}.
\end{proof}

\begin{lemma}[Local constant-stepsize Riemannian SGD floor]
\label{lem:mfv4_constant_stepsize_floor}
Let $Q_\star$ be a nondegenerate local minimizer of the fixed-Stiefel objective $\mathcal J_{\mathrm{FS}}$, and let $H_{\mathrm{FS}}=\operatorname{Hess}\mathcal J_{\mathrm{FS}}(Q_\star)$. Let $\mathcal U\subset T_{Q_\star}\St(p,r)$ be an $r$-dimensional invariant subspace on which the Hessian eigenvalues satisfy
\begin{equation}
0<\mu_{\mathrm{FS}}\leq\mu_i\leq c_\mu\mu_{\mathrm{FS}},\qquad i=1,\ldots,r,\qquad \mu_{\mathrm{FS}}=\lambda_{\min}(H_{\mathrm{FS}}|_{\mathcal U}),\quad c_\mu\geq1.
\label{eq:mfv4_weak_curvature_interval}
\end{equation}
Assume $\widehat g_t=g_t+\xi_t$ with
\begin{equation}
\E[\xi_t\mid\mathcal F_t]=0,\qquad \E[P_{\mathcal U}\xi_t\xi_t^\top P_{\mathcal U}\mid\mathcal F_t]\succeq\nu^2P_{\mathcal U}.
\label{eq:mfv4_noise_covariance_floor}
\end{equation}
For the local linearized recursion
\begin{equation}
z_{t+1}=z_t-\eta(H_{\mathrm{FS}}z_t+\xi_t),\qquad 0<\eta\leq(c_\mu\mu_{\mathrm{FS}})^{-1},
\label{eq:mfv4_linearized_rsgd_chain}
\end{equation}
we have
\begin{equation}
\liminf_{t\to\infty}\left(\E\|P_{\mathcal U}z_t\|_2^2\right)^{1/2}\geq\nu\sqrt{\frac{\eta r}{2c_\mu\mu_{\mathrm{FS}}}},
\label{eq:mfv4_rsgd_rms_floor}
\end{equation}
and, for
\begin{equation}
\mathcal E_{\mathrm{FS}}(\eta)=\liminf_{t\to\infty}\left(\E\left\|\eta H_{\mathrm{FS}}^{-1}\widehat g_t\right\|_F^2\right)^{1/2},
\label{eq:mfv4_curvature_normalized_error_def}
\end{equation}
the curvature-normalized stochastic floor satisfies
\begin{equation}
\mathcal E_{\mathrm{FS}}(\eta)\geq\frac{\eta\nu\sqrt r}{c_\mu\mu_{\mathrm{FS}}}.
\label{eq:mfv4_curvature_normalized_error_floor}
\end{equation}
\end{lemma}

\begin{proof}
Choose an orthonormal eigenbasis $u_1,\ldots,u_r$ of $H_{\mathrm{FS}}|_{\mathcal U}$ and define $y_{i,t}=\langle z_t,u_i\rangle$ and $\zeta_{i,t}=\langle\xi_t,u_i\rangle$. Projecting Eq.~\eqref{eq:mfv4_linearized_rsgd_chain} onto $u_i$ gives $y_{i,t+1}=(1-\eta\mu_i)y_{i,t}-\eta\zeta_{i,t}$. By Eq.~\eqref{eq:mfv4_noise_covariance_floor}, $\E[\zeta_{i,t}\mid\mathcal F_t]=0$ and $\E[\zeta_{i,t}^2\mid\mathcal F_t]\geq\nu^2$, hence
\begin{align}
\E[y_{i,t+1}^2\mid\mathcal F_t]
&=(1-\eta\mu_i)^2y_{i,t}^2+\eta^2\E[\zeta_{i,t}^2\mid\mathcal F_t] \notag\\
&\geq(1-\eta\mu_i)^2y_{i,t}^2+\eta^2\nu^2.
\label{eq:mfv4_scalar_second_moment_recursion}
\end{align}
Let $v_i=\liminf_{t\to\infty}\E y_{i,t}^2$. Taking total expectations and lower limits in Eq.~\eqref{eq:mfv4_scalar_second_moment_recursion} gives
\begin{align}
v_i
&\geq(1-\eta\mu_i)^2v_i+\eta^2\nu^2 \notag\\
&\geq\frac{\eta^2\nu^2}{2\eta\mu_i-\eta^2\mu_i^2} \notag\\
&\geq\frac{\eta\nu^2}{2c_\mu\mu_{\mathrm{FS}}},
\label{eq:mfv4_scalar_variance_floor}
\end{align}
where the last step uses Eq.~\eqref{eq:mfv4_weak_curvature_interval} and $\eta\mu_i\leq1$. Summing over $i$ yields
\begin{equation}
\liminf_{t\to\infty}\E\|P_{\mathcal U}z_t\|_2^2\geq\frac{\eta\nu^2r}{2c_\mu\mu_{\mathrm{FS}}},
\label{eq:mfv4_linearized_vector_floor}
\end{equation}
and Eq.~\eqref{eq:mfv4_rsgd_rms_floor} follows by taking square roots.

For the curvature-normalized bound, Eq.~\eqref{eq:mfv4_weak_curvature_interval} implies
\begin{equation}
\|H_{\mathrm{FS}}^{-1}v\|_F\geq(c_\mu\mu_{\mathrm{FS}})^{-1}\|v\|_F,\qquad v\in\mathcal U.
\label{eq:mfv4_inverse_hessian_lower_bound}
\end{equation}
Moreover, Eq.~\eqref{eq:mfv4_noise_covariance_floor} gives $\E\|P_{\mathcal U}\widehat g_t\|_F^2\geq r\nu^2$. Therefore
\begin{equation}
\left(\E\left\|\eta H_{\mathrm{FS}}^{-1}\widehat g_t\right\|_F^2\right)^{1/2}\geq\eta(c_\mu\mu_{\mathrm{FS}})^{-1}\left(\E\|P_{\mathcal U}\widehat g_t\|_F^2\right)^{1/2}\geq\frac{\eta\nu\sqrt r}{c_\mu\mu_{\mathrm{FS}}},
\label{eq:mfv4_curvature_floor_proof}
\end{equation}
and taking $\liminf_{t\to\infty}$ proves Eq.~\eqref{eq:mfv4_curvature_normalized_error_floor}.
\end{proof}

\begin{proposition}[Fixed-Stiefel SGD has a unit-spectrum error floor]
\label{prop:mfv4_fs_sgd_unit_spectrum_floor}
Consider the fixed-Stiefel operator $W_t=Q_t$ with $Q_t\in\St(p,r)$ and update
\begin{equation}
Q_{t+1}=R_{Q_t}(-\eta\widehat g_t),\qquad \widehat g_t=\Pi_{Q_t}(G_t),
\label{eq:mfv4_fs_rsgd_update}
\end{equation}
where $\Pi_{Q_t}$ is given by Eq.~\eqref{eq:mfv4_stiefel_projection_formula}. Then $\sigma_i(W_t)=1$ for every $i=1,\ldots,r$. If $0<\eta\|\widehat g_t\|_F\leq\delta_R$, then for every unit tangent direction $E\in T_{Q_t}\St(p,r)$,
\begin{equation}
\left|\left\langle Q_{t+1}-Q_t,E\right\rangle_F\right|\leq\eta\|\widehat g_t\|_F+C_R\eta^2\|\widehat g_t\|_F^2.
\label{eq:mfv4_fs_axis_bound}
\end{equation}
Under Lemma~\ref{lem:mfv4_constant_stepsize_floor},
\begin{equation}
\mathcal E_{\mathrm{FS}}(\eta)\geq\frac{\eta\nu\sqrt r}{c_\mu\lambda_{\min}(\operatorname{Hess}\mathcal J_{\mathrm{FS}}(Q_\star)|_{\mathcal U})}.
\label{eq:mfv4_fs_sgd_error_floor}
\end{equation}
\end{proposition}

\begin{proof}
Setting $S=I_r$ in Eq.~\eqref{eq:mfv4_weight_map} gives $W_t=Q_t$. By Proposition~\ref{prop:mfv4_spectral_identity},
\begin{equation}
\sigma_i^2(W_t)=\lambda_i(I_r)=1,\qquad i=1,\ldots,r,
\label{eq:mfv4_fs_unit_spectrum_proof}
\end{equation}
hence every represented singular value is fixed at one.

Applying Lemma~\ref{lem:mfv4_stiefel_projection_displacement} to Eq.~\eqref{eq:mfv4_fs_rsgd_update} gives Eq.~\eqref{eq:mfv4_fs_axis_bound}. Under the local model of Lemma~\ref{lem:mfv4_constant_stepsize_floor}, we further have
\begin{equation}
\mathcal E_{\mathrm{FS}}(\eta)\geq\frac{\eta\nu\sqrt r}{c_\mu\mu_{\mathrm{FS}}},\qquad \mu_{\mathrm{FS}}=\lambda_{\min}(\operatorname{Hess}\mathcal J_{\mathrm{FS}}(Q_\star)|_{\mathcal U}),
\label{eq:mfv4_fs_floor_substitution}
\end{equation}
and substituting the definition of $\mu_{\mathrm{FS}}$ proves Eq.~\eqref{eq:mfv4_fs_sgd_error_floor}.
\end{proof}

Proposition~\ref{prop:mfv4_fs_sgd_unit_spectrum_floor} isolates the fixed-Stiefel limitation relevant to SGD. Equation~\eqref{eq:mfv4_fs_unit_spectrum_proof} keeps the realized spectrum fixed at one, while Eq.~\eqref{eq:mfv4_fs_sgd_error_floor} shows that, in the local constant-stepsize model, the stochastic error floor scales as $\eta\nu\sqrt r/\mu_{\mathrm{FS}}$ up to the curvature-spread factor $c_\mu$. Thus convergence of the Stiefel iterate does not provide direction-dependent spectral adaptation.
\begin{lemma}[Differential spectral scaling of $W(Q,S)=QS^{1/2}$]
\label{lem:mfv4_differential_spectral_scaling}
Let $S=V\operatorname{diag}(s_1,\ldots,s_r)V^\top$ with $s_i>0$ and $V=[v_1,\ldots,v_r]\in\mathrm O(r)$. Let $\Phi(Q,S)=QS^{1/2}$. For every tangent direction $\Xi\in T_Q\St(p,r)$, the $Q$-direction differential satisfies
\begin{equation}
    D_Q\Phi(Q,S)[\Xi]=\Xi S^{1/2}.
    \label{eq:mfv4_q_differential}
\end{equation}
If a unit tangent direction $\Xi_i\in T_Q\St(p,r)$ is aligned with the $i$th SPD eigenvector in the sense that
\begin{equation}
    \Xi_i v_i=a_i,
    \qquad
    \Xi_i v_j=0\quad(j\neq i),
    \qquad
    \|a_i\|_2=1,
    \label{eq:mfv4_aligned_tangent_direction_def}
\end{equation}
then
\begin{equation}
    \|D_Q\Phi(Q,S)[\Xi_i]\|_F=\sqrt{s_i}.
    \label{eq:mfv4_axis_sqrt_scaling}
\end{equation}
\end{lemma}

\begin{proof}
For fixed $S$, the map $Q\mapsto QS^{1/2}$ is linear, so its Fr\'echet derivative with respect to $Q$ is given by Eq.~\eqref{eq:mfv4_q_differential}. By the spectral theorem for SPD matrices \citep{bhatia1997matrix,bhatia2007positive},
\begin{equation}
    S^{1/2}=V\operatorname{diag}(\sqrt{s_1},\ldots,\sqrt{s_r})V^\top=\sum_{j=1}^{r}\sqrt{s_j}v_jv_j^\top.
    \label{eq:mfv4_sqrt_spectral_decomposition}
\end{equation}
Combining Eqs.~\eqref{eq:mfv4_q_differential}, \eqref{eq:mfv4_aligned_tangent_direction_def}, and \eqref{eq:mfv4_sqrt_spectral_decomposition} gives
\begin{align}
    D_Q\Phi(Q,S)[\Xi_i]
    &=\Xi_iS^{1/2} \notag\\
    &=\sum_{j=1}^{r}\sqrt{s_j}\Xi_iv_jv_j^\top \notag\\
    &=\sqrt{s_i}a_iv_i^\top.
    \label{eq:mfv4_axis_differential_rank_one}
\end{align}
Using the Frobenius norm of a rank-one matrix and $\|a_i\|_2=\|v_i\|_2=1$,
\begin{align}
    \|D_Q\Phi(Q,S)[\Xi_i]\|_F
    &=\sqrt{s_i}\|a_iv_i^\top\|_F \notag\\
    &=\sqrt{s_i}\|a_i\|_2\|v_i\|_2 \notag\\
    &=\sqrt{s_i}.
    \label{eq:mfv4_axis_sqrt_scaling_proof}
\end{align}
Thus Eq.~\eqref{eq:mfv4_axis_sqrt_scaling} follows.
\end{proof}

\subsubsection{Proof of Proposition~\ref{prop:mfv4_mf_structured_preconditioner}}
\label{app:proof_structured_preconditioner}

\begin{proposition}[Restatement of Proposition~\ref{prop:mfv4_mf_structured_preconditioner}]
Let $W=QS^{1/2}$ and $S=V\operatorname{diag}(s_1,\ldots,s_r)V^\top$ with $s_i>0$. Fix tangent directions $\Xi_1,\ldots,\Xi_r\in T_Q\St(p,r)$ satisfying Eq.~\eqref{eq:mfv4_aligned_tangent_direction_def}, and define
\begin{equation}
    Y_i=s_i^{-1/2}D_Q\Phi(Q,S)[\Xi_i], \qquad \|Y_i\|_F=1.
    \label{eq:mfv4_normalized_realized_directions}
\end{equation}
Assume that, at a nondegenerate local minimizer $Q_\star$ for fixed $S$, the realized-weight Hessian is diagonal in the directions $\{Y_i\}_{i=1}^r$:
\begin{equation}
    \left\langle Y_i,\nabla_W^2F(W_\star)[Y_j]\right\rangle_F=\mu_i\delta_{ij}, \qquad \mu_1\geq\cdots\geq\mu_r>0,
    \label{eq:mfv4_realized_hessian_alignment}
\end{equation}
where $W_\star=Q_\star S^{1/2}$. Define
\begin{equation}
    m_S=\min_i s_i\mu_i, \qquad L_S=\max_i s_i\mu_i, \qquad \kappa_{\mathrm{eff}}(S)=\frac{L_S}{m_S}, \qquad \kappa=\frac{\mu_1}{\mu_r}.
    \label{eq:mfv4_effective_condition_number}
\end{equation}
Assume that the fixed-$S$ pullback objective is $m_S$-strongly convex and $L_S$-smooth in the local tangent chart and that gradient iterates with $\eta=L_S^{-1}$ remain in this neighborhood. Then
\begin{equation}
    \mathcal J(Q_T,S)-\mathcal J(Q_\star,S)\leq\left(1-\kappa_{\mathrm{eff}}(S)^{-1}\right)^T\left[\mathcal J(Q_0,S)-\mathcal J(Q_\star,S)\right].
    \label{eq:mfv4_mf_local_linear_rate}
\end{equation}
If the learned spectrum satisfies the inverse-curvature alignment condition
\begin{equation}
    \frac{a}{\mu_i}\leq s_i\leq\frac{b}{\mu_i}, \qquad i=1,\ldots,r,\qquad 0<a\leq b<\infty,
    \label{eq:mfv4_inverse_hessian_alignment}
\end{equation}
then
\begin{equation}
    \kappa_{\mathrm{eff}}(S)\leq\frac{b}{a}.
    \label{eq:mfv4_kappa_eff_alignment_bound}
\end{equation}
In particular, $\kappa_{\mathrm{eff}}(S)\leq\kappa$ whenever $b/a\leq\kappa$.
\end{proposition}

\begin{proof}
By Proposition~\ref{prop:mfv4_spectral_identity},
\begin{equation}
    s_i=\lambda_i(S)=\sigma_i^2(W).
    \label{eq:mfv4_preconditioner_spectral_identity_use}
\end{equation}
Thus the eigenvalues of $S$ directly determine the squared singular scales of the realized operator.

Consider the aligned tangent coordinates from Lemma~\ref{lem:mfv4_differential_spectral_scaling}. Let $z_i$ denote the coordinate along $\Xi_i$ and $y_i$ the corresponding realized-weight coordinate along $Y_i$. By Eqs.~\eqref{eq:mfv4_axis_sqrt_scaling} and \eqref{eq:mfv4_normalized_realized_directions},
\begin{equation}
    y_i=\sqrt{s_i}z_i.
    \label{eq:mfv4_y_z_coordinate_relation}
\end{equation}
Using the Hessian alignment in Eq.~\eqref{eq:mfv4_realized_hessian_alignment}, the realized-weight quadratic form is
\begin{equation}
    \left\langle\sum_{i=1}^{r}y_iY_i,\nabla_W^2F(W_\star)\left[\sum_{j=1}^{r}y_jY_j\right]\right\rangle_F=\sum_{i=1}^{r}\mu_i y_i^2.
    \label{eq:mfv4_weight_coordinate_quadratic}
\end{equation}
For fixed $S$, the regularizer $\lambda_S\|\log S\|_F^2/2$ is constant with respect to $Q$. Substituting Eq.~\eqref{eq:mfv4_y_z_coordinate_relation} into Eq.~\eqref{eq:mfv4_weight_coordinate_quadratic} therefore gives
\begin{equation}
    \left\langle z,\nabla_z^2\mathcal J(Q_\star,S)z\right\rangle=\sum_{i=1}^{r}s_i\mu_i z_i^2.
    \label{eq:mfv4_pullback_quadratic}
\end{equation}
Hence the aligned pullback curvature eigenvalues are $s_i\mu_i$, with extremal values $m_S$ and $L_S$ defined in Eq.~\eqref{eq:mfv4_effective_condition_number}.

By the assumed local $m_S$-strong convexity and $L_S$-smoothness, gradient descent with $\eta=L_S^{-1}$ satisfies the standard local contraction \citep{boumal2023smooth}
\begin{align}
    \mathcal J(Q_{t+1},S)-\mathcal J(Q_\star,S)
    &\leq\left(1-\frac{m_S}{L_S}\right)\left[\mathcal J(Q_t,S)-\mathcal J(Q_\star,S)\right] \notag\\
    &=\left(1-\kappa_{\mathrm{eff}}(S)^{-1}\right)\left[\mathcal J(Q_t,S)-\mathcal J(Q_\star,S)\right].
    \label{eq:mfv4_one_step_preconditioned_rate}
\end{align}
Iterating Eq.~\eqref{eq:mfv4_one_step_preconditioned_rate} for $t=0,\ldots,T-1$ proves Eq.~\eqref{eq:mfv4_mf_local_linear_rate}.

Under Eq.~\eqref{eq:mfv4_inverse_hessian_alignment},
\begin{equation}
    a\leq s_i\mu_i\leq b, \qquad i=1,\ldots,r.
    \label{eq:mfv4_aligned_eigenvalue_interval}
\end{equation}
Therefore,
\begin{align}
    \kappa_{\mathrm{eff}}(S)
    &=\frac{\max_i s_i\mu_i}{\min_i s_i\mu_i} \notag\\
    &\leq\frac{b}{a}.
    \label{eq:mfv4_kappa_eff_bound_proof}
\end{align}
This proves Eq.~\eqref{eq:mfv4_kappa_eff_alignment_bound}; if $b/a\leq\kappa$, then $\kappa_{\mathrm{eff}}(S)\leq\kappa$ follows immediately.

Finally, Theorem~\ref{thm:mfv4_nonconvex_convergence} governs the full stochastic product-manifold update with $S_t\in\mathcal S_{\alpha,\beta}$, whereas the present proposition characterizes the local fixed-$S$ conditioning mechanism. Equation~\eqref{eq:mfv4_preconditioner_spectral_identity_use} shows that this conditioning is induced by the realized operator spectrum itself rather than by an external optimizer state such as Adam's coordinatewise second-moment estimate \citep{kingma2014adam}.
\end{proof}

Proposition~\ref{prop:mfv4_mf_structured_preconditioner} makes the conditioning effect precise under the alignment assumption in Eq.~\eqref{eq:mfv4_realized_hessian_alignment}. Fixed-Stiefel training corresponds to $s_i=1$ and therefore $\kappa_{\mathrm{eff}}=\kappa$. Under the inverse-curvature alignment in Eq.~\eqref{eq:mfv4_inverse_hessian_alignment}, ManifoldFlow instead yields $\kappa_{\mathrm{eff}}\leq b/a$ by Eq.~\eqref{eq:mfv4_kappa_eff_alignment_bound}. Whenever $b/a<\kappa$, the local contraction factor improves from $1-\kappa^{-1}$ to $1-\kappa_{\mathrm{eff}}^{-1}$.

\section{Experimental Details and Paired Replicate Check}
\label{app:number_audit}

The main text reports the fixed-protocol FS/MF comparisons from the 1000-epoch convergence runs. Table~\ref{tab:app_main_audit} separately reports archived shorter-budget paired replicates as a check on the direction of the observed effects. These runs are not pooled with, or used to replace, the main convergence endpoints.

For LSTM/WikiText-2, we use the standard recurrent language-modeling setup \citep{hochreiter1997lstm,merity2016pointer} and parametrize only the hidden-to-vocabulary projection as either fixed-spectrum Stiefel (FS) or ManifoldFlow (MF). For Adult, we use a four-block MLP on the UCI Adult Income benchmark \citep{kohavi1996adult,uci1998repository} and parametrize all hidden linear maps. For the Mini-Transformer experiment, we follow the Transformer feed-forward architecture \citep{vaswani2017attention} on WikiText-2, parametrize only the feed-forward linear maps, and leave the attention and embedding parameters unconstrained.

Across paired FS/MF comparisons, FS is obtained by setting the ManifoldFlow geometry learning rate to zero, which fixes $S_t=I$ throughout training. MF uses the realized weight $W=QS^{1/2}$ together with the pressure EMA, geometry gate, affine-invariant SPD update, and eigenvalue clipping defined in Section~\ref{sec:method}. Adam, AMSGrad-style variants, AdamW-style decoupled weight decay, Shampoo-style preconditioning, Muon-style orthogonalized updates, and SGD follow their standard optimizer definitions \citep{kingma2014adam,reddi2018adam,loshchilov2019decoupled,gupta2018shampoo,jordan2024muon,robbins1951sgd}. The archived result files are retained with the released code for reproducibility.

\begin{table}[!ht]
\centering
\caption{\textbf{Short-budget paired replicate check.} Archived shorter-budget FS/MF runs reported separately from the 1000-epoch convergence experiments. Positive $\Delta$ favors MF; for perplexity, $\Delta=\mathrm{FS}-\mathrm{MF}$, whereas for accuracy, $\Delta=\mathrm{MF}-\mathrm{FS}$.}
\label{tab:app_main_audit}
\maxsizebox{\columnwidth}{!}{%
\begin{tabular}{llccc}
\toprule
Task & Optimizer & FS & MF & $\Delta$ \\
\midrule
LSTM/WT2 & Adam & $287.86\pm2.64$ & $277.20\pm2.85$ & $+10.66\pm0.76$ PPL \\
LSTM/WT2 & SGD & $613.66\pm0.06$ & $608.56\pm0.26$ & $+5.10\pm0.09$ PPL \\
Adult MLP & Adam & $84.26\pm0.16$ & $85.04\pm0.19$ & $+0.78\pm0.25\%$ \\
Adult MLP & SGD & $83.66\pm0.63$ & $84.69\pm0.29$ & $+1.03\pm0.69\%$ \\
Transformer FFN & Adam & $80.29\pm0.09$ & $80.06\pm0.15$ & $+0.23\pm0.10$ PPL \\
Transformer FFN & SGD & $301.98\pm2.47$ & $300.33\pm2.45$ & $+1.65\pm0.05$ PPL \\
\bottomrule
\end{tabular}
}
\end{table}

\section{Boundary Analysis}
\label{app:boundary}

The empirical claim in the main text is deliberately restricted to the paired FS/MF comparison: ManifoldFlow relaxes the fixed-spectrum constraint and can improve performance when a Stiefel-structured basis is an appropriate modeling prior. We do not claim that Stiefel-family parametrizations dominate unconstrained dense maps. The archived comparisons in Table~\ref{tab:app_boundary} illustrate this boundary.

On Covertype \citep{blackard1999covertype,uci1998repository}, MF-Adam obtains $74.41\%$ under the archived boundary protocol, whereas the unconstrained Dense and $Q\operatorname{diag}(s)$ baselines obtain $87.77\%$ and $89.84\%$, respectively. Thus, under this protocol, relaxing the fixed Stiefel spectrum does not close the gap to less restrictive parametrizations, indicating that the Stiefel-family structural prior is not advantageous for this setting.

The LSTM output-projection comparison shows a similar distinction between improvement within the Stiefel family and performance relative to alternative parametrizations. In the archived comparison, Dense obtains $104.66$ PPL, $Q\operatorname{diag}(s)$ obtains $186.55$ PPL, and Intrinsic Muon obtains $250.67$ PPL \citep{li2026intrinsicmuon}. These results are reported only to delimit the scope of the paired FS/MF claim. Because these baselines were not rerun under the 1000-epoch convergence protocol, they are not directly combined with the main paired results.

\begin{table}[!ht]
\centering
\caption{\textbf{Boundary baselines under the archived protocol.} Covertype reports accuracy, whereas the LSTM hidden-projection experiment reports perplexity. These results delimit the scope of the paired FS/MF comparison and were not rerun under the 1000-epoch convergence protocol.}
\label{tab:app_boundary}
\maxsizebox{\columnwidth}{!}{%
\begin{tabular}{llc}
\toprule
Task & Method & Recorded value \\
\midrule
Covertype & MF-Adam & $74.41$ \\
Covertype & Dense & $87.77$ \\
Covertype & $Q\operatorname{diag}(s)$ & $89.84$ \\
LSTM hidden projection & Dense & $104.66$ \\
LSTM hidden projection & $Q\operatorname{diag}(s)$ & $186.55$ \\
LSTM hidden projection & Intrinsic Muon & $250.67$ \\
\bottomrule
\end{tabular}
}
\end{table}

\section{Ablation Details}
\label{app:ablation_full}

Table~\ref{tab:app_ablation} isolates the pressure-specific components of ManifoldFlow by removing the pressure EMA, removing the geometry gate, or replacing the pressure signal with random symmetric directions. Adult and Covertype use MF-Adam \citep{kohavi1996adult,blackard1999covertype,kingma2014adam}. We report $\Delta=\mathrm{Full\ MF}-\mathrm{Ablation}$ in percentage points, so positive values favor the complete method. The absolute change is at most $0.08$ percentage points across all six comparisons. These small differences indicate that performance in this ablation protocol is insensitive to the individual pressure-specific components; accordingly, the results provide stronger evidence for the spectrum-relaxation mechanism than for any single pressure-specific design choice.

\begin{table}[!ht]
\centering
\caption{\textbf{Ablation of pressure-specific components.} We remove the pressure EMA, remove the geometry gate, or replace the pressure signal with random symmetric directions. $\Delta=\mathrm{Full\ MF}-\mathrm{Ablation}$ in percentage points.}
\label{tab:app_ablation}
\maxsizebox{\columnwidth}{!}{%
\begin{tabular}{llccc}
\toprule
Ablation & Task & Full MF & Ablation & $\Delta$ \\
\midrule
no EMA & Adult & $84.92$ & $85.00$ & $-0.08$ \\
no EMA & Covertype & $73.96$ & $73.94$ & $+0.02$ \\
no gate & Adult & $84.92$ & $84.92$ & $+0.00$ \\
no gate & Covertype & $73.96$ & $74.03$ & $-0.07$ \\
random pressure & Adult & $84.92$ & $84.96$ & $-0.04$ \\
random pressure & Covertype & $73.96$ & $73.91$ & $+0.06$ \\
\bottomrule
\end{tabular}
}
\end{table}

\end{document}